\documentclass{article}

\usepackage[preprint]{nips_2018}
\usepackage{bold-extra}
\usepackage{amsmath}
\usepackage{amssymb}
\usepackage{mathrsfs}
\usepackage{bbm}
\usepackage{bm}
\usepackage{multirow}
\usepackage{color}
\usepackage{enumitem}
\usepackage{marginnote}
\usepackage{graphicx}
\usepackage{mdframed}
\usepackage{mathcomp}
\usepackage{stmaryrd}
\usepackage{mathtools}
\usepackage{xr}
\usepackage{nicefrac}       

\DeclareMathOperator{\T}{\mathsf{T}}

\DeclareMathOperator{\ext}{ext}

\DeclareMathOperator{\Mix}{Mix}

\DeclareMathOperator{\Dim}{dim}

\DeclareMathOperator{\co}{co}

\DeclareMathOperator{\tr}{tr}
\DeclareMathOperator{\Int}{int}
\DeclareMathOperator{\Ri}{ri}
\DeclareMathOperator{\Diag}{diag}

\DeclareMathOperator{\Dom}{dom}
\DeclareMathOperator{\Bd}{bd}
\DeclareMathOperator{\Rbd}{rbd}

\DeclareMathOperator*{\Argmin}{Argmin}

\DeclareMathOperator*{\argmin}{argmin}
\DeclareMathOperator*{\argmax}{argmax}

\DeclareMathOperator*{\Se}{S}

\DeclareMathOperator*{\sgn}{sgn}

\newcommand*{\cI}{\mathfrak{l}}

\newcommand*{\hbm}[1]{\hat{\bm{#1}}}
\newcommand*{\smat}[1]{\begin{bsmallmatrix}#1\end{bsmallmatrix}}
\newcommand*{\ibid}{ibid.}
\newcommand*{\I}{\iota}
\newcommand*{\Hi}{Hiriart-Urruty}
\newcommand*{\m}{\mathsf{m}}
\newcommand*{\M}{\operatorname{Mix}}
\newcommand*{\op}[1]{\operatorname{#1}}

\newcommand{\tbr}{\underline{\tilde{L}}}
\newcommand{\br}{\underline{L}}
\newcommand{\Inner}[2]{\left\langle#1,#2\right\rangle}
\newcommand{\inner}[2]{\langle#1,#2\rangle}

\newcommand{\tbm}[1]{\tilde{\bm{#1}}}

\newcommand{\norm}[1]{\left\lVert#1\right\rVert}   
\newcommand{\Norm}[1]{\lVert#1\rVert}   
\newcommand{\diag}[1]{\Diag\left(#1\right)}

\newcommand{\sps}{\mathscr{S}_{\ell}}
\newcommand{\cco}{\overline{\co}}
\newcommand{\sell}{\underline{\ell}}
\newcommand{\Der}[1]{\left. \frac{d}{dt} #1\right|_{t=0}}
\newcommand{\der}[1]{ \frac{d}{dt} #1|_{t=0}}

\newmdtheoremenv{condition}{Condition}

\newcommand{\vertiii}[1]{{\left\vert\kern-0.25ex\left\vert\kern-0.25ex\left\vert #1 
    \right\vert\kern-0.25ex\right\vert\kern-0.25ex\right\vert}}

\makeatletter
\newcommand*\bcdot{\mathpalette\bigcdot@{.5}}
\newcommand*\bigcdot@[2]{\mathbin{\vcenter{\hbox{\scalebox{#2}{$\m@th#1\bullet$}}}}}

\usepackage[utf8]{inputenc} 
\usepackage[T1]{fontenc}    
\usepackage{hyperref}       
\usepackage{cleveref}

\usepackage{url}            
\usepackage[ruled,commentsnumbered]{algorithm2e}
\usepackage{booktabs}       
\usepackage{array}
\usepackage{amsfonts}       
\usepackage{nicefrac}       
\usepackage{microtype}      
\usepackage{amsthm}
\usepackage{footnote}
\makesavenoteenv{tabular}

\crefformat{section}{\S#2#1#3} 
\crefformat{subsection}{\S#2#1#3}

\newtheorem{theorem}{Theorem}
\newtheorem{proposition}[theorem]{Proposition}
\newtheorem{lemma}[theorem]{Lemma}
\newtheorem{remark}[theorem]{Remark}
\newtheorem{corollary}[theorem]{Corollary}
\newtheorem{assumption}[theorem]{Assumption}
\newtheorem{definition}[theorem]{Definition}
\newtheorem{example}{Example}
\newtheorem{claim}{Claim}

\title{Constant Regret, Generalized Mixability, and Mirror Descent}

\author{
  Zakaria Mhammedi 
 \\
  Research School of Computer Science\\
  Australian National University and DATA61 \\
  \texttt{zak.mhammedi@anu.edu.au} \\
  \And
   Robert C. Williamson \\
   Research School of Computer Science\\
  Australian National University and DATA61 \\
   \texttt{bob.williamson@anu.edu.au} 
}

\begin{document}

\maketitle

\begin{abstract}
 We consider the setting of prediction with expert advice; a learner makes predictions by aggregating those of a group of experts. Under this setting, and for the right choice of loss function and ``mixing'' algorithm, it is possible for the learner to achieve a constant regret regardless of the number of prediction rounds. For example, a constant regret can be achieved for \emph{mixable} losses using the \emph{aggregating algorithm}. The \emph{Generalized Aggregating Algorithm} (GAA) is a name for a family of algorithms parameterized by convex functions on simplices (entropies), which reduce to the aggregating algorithm when using the \emph{Shannon entropy} $\op{S}$. For a given entropy $\Phi$, losses for which a constant regret is possible using the \textsc{GAA} are called $\Phi$-mixable. Which losses are $\Phi$-mixable was previously left as an open question. We fully characterize $\Phi$-mixability and answer other open questions posed by \cite{Reid2015}. We show that the Shannon entropy $\op{S}$ is fundamental in nature when it comes to mixability; any $\Phi$-mixable loss is necessarily $\op{S}$-mixable, and the lowest worst-case regret of the \textsc{GAA} is achieved using the Shannon entropy. Finally, by leveraging the connection between the \emph{mirror descent algorithm} and the update step of the GAA, we suggest a new \emph{adaptive} generalized aggregating algorithm and analyze its performance in terms of the regret bound. 
\end{abstract}

 \section{Introduction}
    Two fundamental problems in learning are how to aggregate information and under what circumstances can one learn fast. In this paper, we consider the problems jointly, extending the understanding and characterization of exponential mixing due to \cite{Vovk1998}, who showed that not only does the ``\emph{aggregating algorithm}'' learn quickly when the loss is suitably chosen, but that it is in fact a generalization of classical Bayesian updating, to which it reduces when the loss is log-loss \cite{Vovk2001}. We consider a general class of aggregating schemes, going beyond Vovk's exponential mixing, and provide a complete characterization of the mixing behavior for general losses and general mixing schemes parameterized by an arbitrary entropy function.

        In the \emph{game of prediction with expert advice} a \emph{learner} predicts the outcome of a random variable (outcome of the \emph{environment}) by aggregating the predictions of a pool of experts. At the end of each prediction round, the outcome of the environment is announced and the learner and experts suffer losses based on their predictions. We are interested in algorithms that the learner can use to ``aggregate'' the experts' predictions and minimize the \emph{regret} at the end of the game. In this case, the regret is defined as the difference between the cumulative loss of the learner and that of the best expert in hindsight after $T$ rounds.
        
The \emph{Aggregating Algorithm} (AA) \citep{Vovk1998} achieves a constant regret --- a precise notion of fast learning --- for \emph{mixable} losses; that is, the regret is bounded from above by a constant $R_{\ell}$ which depends only on the loss function $\ell$ and not on the number of rounds $T$. It is worth mentioning that mixability is a weaker condition than exp-concavity, and contrary to the latter, mixability is an intrinsic, parametrization-independent notion \citep{Kamalaruban2015}.

 Reid et al. \cite{Reid2015} introduced the \emph{Generalized Aggregating Algorithm} (GAA), going beyond the AA. The \textsc{GAA} is parameterized by the choice of a convex function $\Phi$ on the simplex (entropy) and reduces to the AA when $\Phi$ is the Shannon entropy. The \textsc{GAA} can achieve a constant regret for losses satisfying a certain condition called $\Phi$-\emph{mixability} (characterizing when losses are $\Phi$-mixable was left as an open problem). This regret depends jointly on the \emph{generalized mixability constant} $\eta_{\ell}^{\Phi}$ --- essentially the largest $\eta$ such that $\ell$ is $(\frac{1}{\eta}\Phi)$-mixable --- and the divergence $D_{\Phi}(\bm{e}_{\theta}, \bm{q})$, where $\bm{q} \in \Delta_k$ is a prior distribution over $k$ experts and $\bm{e}_{\theta}$ is the $\theta$th standard basis element of $\mathbb{R}^k$ \citep{Reid2015}. At each prediction round, the \textsc{GAA} can be divided into two steps; a \emph{substitution step} where the learner picks a prediction from a set specified by the $\Phi$-mixability condition; and an \emph{update step} where a new distribution $\bm{q}$ over experts is computed depending on their performance. 
Interestingly, this update step is exactly the \emph{mirror descent algorithm} \citep{Steinhardt2014, DBLP:journals/ml/OrabonaCC15} which minimizes the weighted loss of experts.

        \paragraph{Contributions.} We introduce the notion of a \emph{support loss}; given a loss $\ell$ defined on any action space, there exists a proper loss $\sell$ which shares the same Bayes risk as $\ell$. When a loss is mixable, one can essentially work with a proper (support) loss instead --- this will be the first stepping stone towards a characterization of (generalized) mixability. 
        
The notion of $\Phi$-mixable and the \textsc{GAA} were previously restricted to finite losses. We extend these to allow for the use of losses which can take infinite values (such as the $\log$-loss), and we show in this case that under the $\Phi$-mixability condition a constant regret is achievable using the GAA.

   For an entropy $\Phi$ and a loss $\ell$, we derive a necessary and sufficient condition (Theorems \ref{17:} and \ref{18:}) for $\ell$ to be $\Phi$-mixable. In particular, if $\ell$ and $\Phi$ satisfy some regularity conditions, then $\ell$ is $\Phi$-mixable if and only if $\eta_{\ell} \Phi - \Se$ is convex on the simplex, where $\Se$ is the Shannon entropy and $\eta_{\ell}$ is essentially the largest $\eta$ such that $\ell$ is $\eta$-mixable \citep{Vovk1998, DBLP:journals/jmlr/ErvenRW12}. This implies that a loss $\ell$ is $\Phi$-mixable only if it is $\eta$-mixable for some $\eta>0$. This, combined with the fact that $\eta$-mixability is equivalently $(\frac{1}{\eta}\Se)$-mixability (Theorem \ref{14:}), reflects one fundamental aspect of the Shannon entropy.
   
Then, we derive an explicit expression for the generalized mixability constant $\eta_{\ell}^{\Phi}$ (Corollary \ref{20:}), and thus for the regret bound of the GAA. This allows us to compare the regret bound $R^{\Phi}_{\ell}$ of any entropy $\Phi$ with that of the Shannon entropy $\Se$. In this case, we show (Theorem \ref{21:}) that $R^{\Se}_{\ell} \leq R^{\Phi}_{\ell}$; that is, the \textsc{GAA} achieves the lowest worst-case regret when using the Shannon entropy --- another result which reflects the fundamental nature of the Shannon entropy.
        
Finally, by leveraging the connection between the \textsc{GAA} and the mirror descent algorithm, we present a new algorithm --- the \emph{Adaptive Generalized Aggregating Algorithm} (AGAA). This algorithm consists of changing the entropy function at each prediction round similar to the \emph{adaptive mirror descent algorithm} \citep{Steinhardt2014}. We analyze the performance of this algorithm in terms of its regret bound. 
 \paragraph{Layout.}  In \S \ref{Preliminaries}, we give some background on loss functions and present new results (Theorem \ref{4:} and \ref{5:}) based on the new notion of a \emph{proper support loss}; we show that, as far as mixability is concerned, one can always work with a proper (support) loss instead of the original loss (which can be defined on an arbitrary action space). In \S \ref{Mixability}, we introduce the notions of classical and generalized mixability and derive a characterization of $\Phi$-mixability (Theorems \ref{17:} and \ref{18:}). We then introduce our new algorithm --- the \textsc{AGAA} --- and analyze its performance. We conclude the paper by a general discussion and direction for future work. All proofs, except for that of Theorem \ref{19:}, are deferred to Appendix \ref{proofsmain}.

\paragraph{Notation.}
	\label{notation}
Let $m\in \mathbb{N}$. We denote $[m] \coloneqq \{1, \dots, m\}$ and $\tilde{m}\coloneqq m-1$. We write $\inner{\cdot}{\cdot}$ for the standard inner product in Euclidean space. Let $\Delta_m \coloneqq \{ \bm{p} \in [0,+\infty[^m\colon \inner{\bm{p}}{\bm{1}_m}=1  \}$ be the \emph{probability simplex} in $\mathbb{R}^m$, and let $\tilde{\Delta}_m \coloneqq \{\tbm{p} \in [0,+\infty[^{\tilde{m}}\colon \inner{\tbm{p}}{\bm{1}_{\tilde{m}}} \leq 1 \}$. We will extensively make use of the affine map $\amalg_m\colon \mathbb{R}^{\tilde{m}} \to \mathbb{R}^m$ defined by \begin{align}\label{amalg}
\amalg_m(\bm{u}) \coloneqq [u_1, \dots, u_{\tilde{m}}, 1 - \inner{\bm{u}}{\bm{1}_{\tilde{m}}}]^{\mathsf{T}}.\end{align} 

We denote $\Int \mathcal{C}$, $\Ri \mathcal{C}$, and $\Rbd \mathcal{C}$ the \emph{interior}, \emph{relative interior}, and \emph{relative boundary} of a set $\mathcal{C}\in \mathbb{R}^m $, respectively \citep{Hiriart-Urruty}. The \emph{sub-differential} of a function $f\colon \mathbb{R}^m \rightarrow \mathbb{R}\cup\{+\infty\}$ at $\bm{u}\in \mathbb{R}^m$ such that $ f(\bm{u})<+\infty$ is defined by (\cite{Hiriart-Urruty}) 
	\begin{align}
	\partial f(\bm{u}) \coloneqq \{\bm{s}^* \in \mathbb{R}^m\colon f(\bm{v}) \geq f(\bm{u}) +  \Inner{\bm{s}^*}{\bm{v} - \bm{u}}, \forall \bm{v} \in \mathbb{R}^m  \}. \label{sub}
	\end{align}
Table \ref{notation} on page~\pageref{notation} provides a list of the main symbols used in this paper. 


	\section{Loss Functions}
	\label{Preliminaries}
	In general, a loss function is a map $\ell \colon \mathcal{X} \times \mathcal{A} \rightarrow [0, +\infty]$ where $\mathcal{X}$ is an outcome set and $\mathcal{A}$ is an action set. In this paper, we only consider the case $\mathcal{X} = [n]$, i.e. finite outcome space. Overloading notation slightly, we define the mapping $\ell \colon \mathcal{A} \rightarrow [0, +\infty]^n$ by $[\ell(\bm{a})]_x = \ell(x, \bm{a}), \forall x \in [n]$ and denote $\ell_x(\cdot) \coloneqq [\ell(\cdot)]_x$. We further extend the new definition of $\ell$ to the set $\bigcup_{k\geq 1} \mathcal{A}^k$ such that for $x\in [n]$ and $A\coloneqq [\bm{a}_{\theta}]^{\mathsf{T}}_{1 \leq \theta \leq k} \in \mathcal{A}^k$, $\ell_x(A) \coloneqq [\ell_x(\bm{a}_{\theta})]^{\mathsf{T}}_{1 \leq \theta \leq k} \in [0, +\infty]^k$. We define the \emph{effective domain} of $\ell$ by $\Dom \ell \coloneqq \{\bm{a}\in \mathcal{A}\colon \ell(\bm{a}) \in [0,+\infty[^n \}$, and the \emph{loss surface} by $\mathcal{S}_{\ell} \coloneqq \{\ell(\bm{a})\colon \bm{a}\in \Dom \ell \}$. We say that $\ell$ is \emph{closed} if $\mathcal{S}_{\ell}$ is closed in $\mathbb{R}^n$. The \emph{superprediction} set of $\ell$ is defined by $\sps^{\infty} \coloneqq \{\ell(\bm{a})+\bm{d}: (\bm{a}, \bm{d}) \in \mathcal{A} \times [0,+\infty[^n\} $. Let $\sps \coloneqq \sps^{\infty} \cap [0,+\infty[^n$ be its \emph{finite} part. 
	
 Let $\bm{a}_0 , \bm{a}_1 \in \mathcal{A}$. The prediction $\bm{a}_0$ is said to be \emph{better} than $\bm{a}_1$ if the component-wise inequality $\ell(\bm{a}_0) \leq \ell(\bm{a}_1)$ holds and there exists some $x\in[n]$ such that $\ell_x(\bm{a}_0) < \ell_x(\bm{a}_1)$ \citep{DBLP:journals/jmlr/WilliamsonVR16}. A loss $\ell$ is \emph{admissible} if for any $\bm{a} \in \mathcal{A}$ there are no better predictions.

For the rest of this paper (except for Theorem \ref{4:}), we make the following assumption on losses; 
	\begin{assumption}
		\label{B:}
		$\ell$ is a closed, admissible loss such that $\Dom \ell \neq \varnothing$.
	\end{assumption}
It is clear that there is no loss of generality in considering only admissible losses. The condition that $\ell$ is closed is a weaker version of the more common assumption that $\mathcal{A}$ is compact and that $\bm{a}\mapsto\ell(x, \bm{a})$ is continuous with respect to the extended topology of $[0,+\infty]$ for all $x\in[n]$ \cite{Kalnishkan2004,Chernov2010}. In fact, we do not make any explicit topological assumptions on the set $\mathcal{A}$ ($\mathcal{A}$ is allowed to be open in our case). Our condition simply says that if a sequence of points on the loss surface converges in $[0,+\infty[^n$, then there exists an action in $\mathcal{A}$ whose image through the loss is equal to the limit. 
For example the 0-1 loss $\ell_{0\text{-}1}$ is closed, yet the map $\bm{p}\mapsto \ell_{0\text{-}1}(x, \bm{p})$ is not continuous on $\Delta_2$, for $x\in\{0,1\}$.

In this paragraph let $\mathcal{A}$ be the $n$-\emph{simplex}, i.e. $\mathcal{A} = \Delta_n$. We define the \emph{conditional risk} $L_{\ell}\colon \Delta_n \times \Delta_n \rightarrow \mathbb{R}$ by $L_{\ell}(\bm{p},\bm{q}) = \mathbb{E}_{x \sim \bm{p}} [\ell_x(\bm{q})]=  \inner{\bm{p}}{\ell(\bm{q})}$ and the \emph{Bayes risk} by $\br_{\ell}(\bm{p})\coloneqq \inf_{\bm{q} \in \Delta_n} L_{\ell}(\bm{p},\bm{q})$. In this case, the loss $\ell$ is \emph{proper} if $\br_{\ell}(\bm{p})= \inner{\bm{p}}{\ell(\bm{p})} \leq \inner{\bm{p}}{\ell(\bm{q})}$ for all $\bm{p} \neq \bm{q}$ in $\Delta_n$ (and \emph{strictly proper} if the inequality is strict). For example, the $log$-loss $\ell_{\log}\colon \Delta_n \to [0, +\infty]^n$ is defined by $\ell_{\log}(\bm{p})=  -\log \bm{p} $, where the `$\log$' of a vector applies component-wise. One can easily check that $\ell_{\log}$ is strictly proper. We denote $\br_{\log}$ its Bayes risk. 
		
	The above definition of the Bayes risk is restricted to losses defined on the simplex. For a general loss $\ell \colon \mathcal{A} \rightarrow [0, +\infty]^n$, we use the following definition; 
	\begin{definition}[Bayes Risk]
		\label{3:}
		Let $\ell \colon \mathcal{A} \rightarrow [0, +\infty]^n$ be a loss such that $\Dom \ell \neq \varnothing$. The \emph{Bayes risk} $\br_{\ell}\colon \mathbb{R}^n \rightarrow \mathbb{R} \cup \{-\infty\}$ is defined by
		\begin{align}
		\label{bayes}
		\forall \bm{u} \in \mathbb{R}^n, \quad 	\br_{\ell}(\bm{u})  \coloneqq \inf_{\bm{z} \in \sps } \Inner{\bm{u}}{\bm{z}}.
		\end{align}
	\end{definition}
The support function of a set $\mathcal{C}\subseteq \mathbb{R}^n$ is defined by $\sigma_{\mathcal{C}}(\bm{u}) \coloneqq \sup_{\bm{z}\in \mathcal{C}} \inner{\bm{u}}{\bm{z}}$, $\bm{u}\in \mathbb{R}^n$, and thus it is easy to see that one can express the Bayes risk as $\br_{\ell}(\bm{u}) = -\sigma_{\sps}(-\bm{u})$.
 Our definition of the Bayes risk is slightly different from previous ones (\cite{Kalnishkan2004,DBLP:journals/jmlr/ErvenRW12,Chernov2010}) in two ways; 1) the Bayes risk is defined on all $\mathbb{R}^n$ instead of $[0,+\infty[^n$; and 2) the infimum is taken over the finite part of the superprediction set $\sps^{\infty}$. The first point is a mere mathematical convenience and makes no practical difference since $\br_{\ell}(\bm{p}) =-\infty$ for all $\bm{p}\notin [0,+\infty[^n$. For the second point, swapping $\sps$ for $\sps^{\infty}$ in \eqref{bayes} does not change the value of $\br_{\ell}$ for mixable losses (see Appendix \ref{bayessup}). However, we chose to work with $\sps$ --- a subset of $\mathbb{R}^n$ --- as it allows us to directly apply techniques from convex analysis.
\begin{definition}[Support Loss]
\label{supportloss}
We call a map $\underline{\ell}: \Delta_n \rightarrow [0, +\infty]^n$ a \emph{support loss} of $\ell$ if 
\begin{gather*}
\forall \bm{p} \in \Ri \Delta_n,\; \underline{\ell}(\bm{p}) \in \partial \sigma_{\sps} (-\bm{p});\\ 
\forall \bm{p}\in \Rbd \Delta_n, \exists (\bm{p}_m) \subset \Ri \Delta_n,\;\bm{p}_m\stackrel{m \to \infty}{\to} \bm{p}\; \mbox{ and }\; \sell(\bm{p}_m)\stackrel{m \to \infty}{\to} \sell(\bm{p})\; \mbox{component-wise,} 
\end{gather*}
where $\partial \sigma_{\sps}$ (see \eqref{sub}) is the sub-differential of the support function --- $\sigma_{\sps}$ --- of the set $\sps$.
\end{definition}

\begin{theorem}
\label{4:}
Any loss $\ell \colon \mathcal{A} \rightarrow [0, +\infty]^n$ such that $\Dom \ell \neq \varnothing$, has a proper support loss $\underline{\ell}$ with the same Bayes risk, $\br_{\ell}$, as $\ell$. 
\end{theorem}	
Theorem \ref{4:} shows that regardless of the action space on which the
loss is defined, there always exists a proper loss whose Bayes risk
coincides with that of the original loss. This fact is useful in situations
where the Bayes risk contains all the information one needs --- such is the
case for mixability. The next Theorem shows a stronger relationship between
a loss and its corresponding support loss. 
\begin{theorem}
\label{5:}
Let $\ell \colon \mathcal{A} \rightarrow [0, +\infty]^n$ be a loss and $\underline{\ell}$ be a proper support loss of $\ell$. If the Bayes risk $\br_{\ell}$ is differentiable on $]0, +\infty[^n$, then $\sell$ is uniquely defined on $\Ri \Delta_n$ and
\begin{align*}
\begin{array}{lll}
\forall \bm{p}\in \Dom \sell,& \exists \bm{a}_*\in \Dom \ell,&  \ell(\bm{a}_*)= \sell(\bm{p}), \\
\forall \bm{a}\in \Dom \ell, &\exists (\bm{p}_m) \subset \Ri \Delta_n, & \sell(\bm{p}_m) \stackrel{m\to \infty}{\to} \ell(\bm{a})\; \mbox{ component-wise.}
\end{array}
\end{align*}
\end{theorem}
Theorem \ref{5:} shows that when the Bayes risk is differentiable (a necessary condition for mixability --- Theorem \ref{14:}), the support loss is almost a reparametrization of the original loss, and in practice, it is enough to work with support losses instead. This will be crucial for characterizing $\Phi$-mixability. 
%

	

	\section{Mixability in the Game of Prediction with Expert Advice}
	\label{Mixability}
	
	We consider the setting of prediction with expert advice \citep{Vovk1998}; there a is pool of $k$ experts, parameterized by $\theta \in [k]$, which make predictions $\bm{a}_{\theta}^t \in \mathcal{A}$ at each round $t$. In the same round, the learner predicts $\bm{a}^t_{\mathfrak{M}} \coloneqq \mathfrak{M}\left(\bm{a}_{1:k}^t, (x^s, \bm{a}_{1:k}^s)_{1 \leq s < t }\right)\in \mathcal{A}$, where $\bm{a}^{\cdot}_{1:k}\coloneqq [\bm{a}^{\cdot}_{\theta}]_{1\leq \theta \leq k}$, $(x^s)\subset [n]$ are outcomes of the environment, and $\mathfrak{M}: \mathcal{A}^k \times ([n]\times \mathcal{A}^k)^* \rightarrow \mathcal{A}$ is a \emph{merging strategy} \cite{DBLP:journals/jmlr/ErvenRW12}. At the end of round $t$, $x^t$ is announced and each expert $\theta$ [resp. learner] suffers a loss $\ell_{x^t}(\bm{a}_{\theta})$ [resp. $\ell_{x^t}(\bm{a}^t_{\mathfrak{M}})$], where $\ell \colon \mathcal{A} \rightarrow [0,+\infty]^n$. After $T>0$ rounds, the cumulative loss of each expert $\theta$ [resp. learner] is given by $\op{Loss}^{\ell}_{\theta}(T)\coloneqq \sum_{t=1}^T \ell_{x^t}(\bm{a}^t_{\theta})$ [resp. $\op{Loss}^{\ell}_{\mathfrak{M}}(T)\coloneqq \sum_{t=1}^T \ell_{x^t}(\bm{a}^t_{\mathfrak{M}})$]. We say that $\mathfrak{M}$ achieves a \emph{constant regret} if $\exists R>0, \forall T>0,\forall \theta\in[k], \op{Loss}^{\ell}_{\mathfrak{M}}(T) \leq \op{Loss}^{\ell}_{\theta}(T)+R$. In what follows, this game setting will be referred to by $\mathfrak{G}^n_{\ell}(\mathcal{A}, k)$ and we only consider the case where $k\geq 2$.
	
\subsection{The Aggregating Algorithm and $\eta$-mixability}
\label{aa}
\begin{definition}[$\eta$-mixability]
\label{etamix}
For $\eta>0$, a loss $\ell\colon \mathcal{A} \rightarrow [0, +\infty]^{n}$ is said to be $\eta$-mixable, if $\forall \bm{q} \in \Delta_k$, 
		\begin{align}
	\forall  \bm{a}_{1:k} \in \mathcal{A}^k,	\exists \bm{a}_* \in  \mathcal{A},\forall x\in [n],\quad  \ell_x(\bm{a}_*)  \leq - \eta^{-1} \log \Inner{\bm{q}}{\exp(- \eta \ell_{x}(\bm{a}_{1:k}))},\label{4:e}
		\end{align} 
where the $\exp$ applies component-wise. Letting $\mathfrak{H}_{\ell} \coloneqq \{\eta> 0\colon  \ell \text{ is } \eta\text{-mixable}\}$, we define the \emph{mixability constant} of $\ell$ by $\eta_{\ell}\coloneqq \sup \mathfrak{H}_{\ell}$ if $\mathfrak{H}_{\ell} \neq \varnothing$; and $0$ otherwise. $\ell$ is said to be \emph{mixable} if $\eta_{\ell}>0$. 
\end{definition}
If a loss $\ell$ is $\eta$-mixable for $\eta>0$, the AA (Algorithm \ref{AA2}) achieves a constant regret in the $\mathfrak{G}^n_{\ell}(\mathcal{A}, k)$ game\cite{Vovk1998}. In Algorithm \ref{AA2}, the map $\mathfrak{S}_{\ell}\colon \sps^{\infty} \rightarrow \mathcal{A}$ is a \emph{substitution function} of the loss $\ell$ \cite{Vovk1998,Kamalaruban2015}; that is, $\mathfrak{S}_{\ell}$ satisfies the component-wise inequality $\ell(\mathfrak{G}_{\ell}(\bm{s}))\leq \bm{s}$, for all $\bm{s}\in \sps^{\infty}$.

It was shown by Chernov et al. \cite{Chernov2010} that the $\eta$-mixability condition \eqref{4:e} is equivalent to the convexity of the $\eta$-\emph{exponentiated} superprediction set of $\ell$ defined by $\exp(-\eta \sps^{\infty})\coloneqq \{\exp(-\eta \bm{s})\colon \bm{s} \in \sps^{\infty} \}$. Using this fact, van Erven et al. \cite{DBLP:journals/jmlr/ErvenRW12} showed that the mixability constant $\eta_{\ell}$ of a strictly proper loss $\ell \colon \Delta_n \rightarrow [0,+\infty[^n$, whose Bayes risk is twice continuously differentiable on $]0,+\infty[^n$, is equal to
	\begin{align}
\underline{\eta_{\ell}} & \coloneqq \inf_{\tbm{p} \in \Int \tilde{\Delta}_n} (\lambda_{\max} ([\mathsf{H} \tbr_{\log}(\tbm{p})]^{-1} \mathsf{H} \tbr_{\ell} (\tbm{p})))^{-1}, \label{5:e}
	\end{align}
where $\mathsf{H}$ is the Hessian operator and $\tbr_{\cdot}\coloneqq \br_{\cdot} \circ \amalg_n$ ($\amalg_n$ was defined in \eqref{amalg}). The next theorem extends this result by showing that the mixability constant $\eta_{\ell}$ of any loss $\ell$ is lower bounded by $\underline{\eta_{\ell}}$ in \eqref{5:e}, as long as $\ell$ satisfies Assumption \ref{B:} and its Bayes risk is twice differentiable.
\begin{theorem}
\label{8:}
Let $\eta>0$ and $\ell \colon \mathcal{A} \rightarrow [0,+\infty]^n$ be a loss. Suppose that $\Dom \ell = \mathcal{A}$ and that $\br_{\ell}$ is twice differentiable on $]0, +\infty[^n$. If $
\underline{\eta_{\ell}}>0$ then $\ell$ is $\underline{\eta_{\ell}}$-mixable. In particular, $\eta_{\ell} \geq \underline{\eta_{\ell}}$.
\end{theorem}

We later show that, under the same conditions as Theorem \ref{8:}, we actually have $\eta_{\ell} = \underline{\eta_{\ell}}$ (Theorem \ref{19:}) which indicates that the Bayes risk contains all the information necessary to characterize mixability. 
\begin{remark}
\label{inf-loss}
In practice, the requirement `$\Dom \ell = \mathcal{A}$' is not necessarily a strict restriction to finite losses; it is often the case that a loss $\bar{\ell}: \bar{\mathcal{A}} \rightarrow [0,+\infty]^n$ only takes infinite values on the relative boundary of $\bar{\mathcal{A}}$ (such is the case for the $\log$-loss  defined on the simplex), and thus the restriction $\ell \coloneqq \bar{\ell}|_{\mathcal{A}}$, where $\mathcal{A}=\Ri \bar{\mathcal{A}}$, satisfies $\Dom \ell =\mathcal{A}$. It follows trivially from the definition of mixability \eqref{4:e} that if $\ell$ is $\eta$-mixable and $\bar{\ell}$ is continuous with respect to the extended topology of $[0,+\infty]^n$ --- a condition often satisfied --- then $\bar{\ell}$ is also $\eta$-mixable. 
\end{remark}
	\subsection{The Generalized Aggregating Algorithm and $(\eta, \Phi)-$mixability} 
		A function $\Phi \colon \mathbb{R}^k\rightarrow \mathbb{R} \cup \{+\infty\}$ is an \emph{entropy} if it is convex, its epigraph $\op{epi} \Phi \coloneqq \{ (\bm{u}, h)\colon  \Phi(\bm{u}) \leq h \}$ is closed in $\mathbb{R}^k \times \mathbb{R}$, and $\Delta_k \subseteq \Dom \Phi \coloneqq \{\bm{u}\in \mathbb{R}^k\colon \Phi(\bm{u}) < +\infty \}$. For example, the \emph{Shannon entropy} is defined by $\op{S}(\bm{q})=+\infty$ if $\bm{q}\notin [0,+\infty[^k$, and  \begin{align}\forall   \bm{q}\in [0, +\infty[^k, \quad  \Se(\bm{q}) = \sum_{\substack{i\in[k]\colon q_i \neq 0 }} q_i \log q_i,  \label{shannon}\end{align}
The \emph{divergence} generated by an entropy $\Phi$ is the map $D_{\Phi}\colon \mathbb{R}^n \times \Dom \Phi \rightarrow [0,+\infty]$ defined by 
			\begin{align}
			D_{\Phi}(\bm{v}, \bm{u})\coloneqq \left\{ \begin{array}{ll} \Phi(\bm{v}) - \Phi(\bm{u}) - \Phi'(\bm{u}; \bm{v}-\bm{u}), &\mbox{ if } \bm{v} \in \Dom \Phi; \\ +\infty, & \mbox{ otherwise.}  \end{array} \right. \label{div}
			\end{align}
	where $\Phi'(\bm{u}; \bm{v}-\bm{u}) \coloneqq \lim_{\lambda
	\downarrow 0 } [\Phi(\bm{u} + \lambda(\bm{v} - \bm{u})) -
	\Phi(\bm{u})]/\lambda$ (the limit exists since $\Phi$ is convex
	\citep{Rockafellar1997a}). 

	\begin{definition}[$\Phi$-mixability] \label{9:} Let $\Phi\colon \mathbb{R}^k \rightarrow \mathbb{R} \cup \{+ \infty \}$ be an entropy. A loss $\ell \colon \mathcal{A} \rightarrow [0,+\infty]^n$ is $(\eta,\Phi)$-mixable for $\eta>0$ if $\forall \bm{q} \in \Delta_k$, $\forall \bm{a}_{1:k} \in  \mathcal{A}^k$, $\exists \bm{a}_* \in \mathcal{A},$ such that 
		\begin{align}
		\label{6:e}
		 \forall x\in [n], \; \ell_x(\bm{a}_{*}) \leq \M^{\eta}_{\Phi}(\ell_x(\bm{a}_{1:k}), \bm{q})\coloneqq \inf_{\hat{\bm{q}} \in \Delta_k} \Inner{\hat{\bm{q}}}{\ell_x(\bm{a}_{1:k})} + \eta^{-1}D_{\Phi}(\hat{\bm{q}}, \bm{q}).
		\end{align}
When $\eta=1$, we simply say that $\ell$ is $\Phi$-mixable and we denote $\M_{\Phi}\coloneqq \M^1_{\Phi}$. Letting $\mathfrak{H}^{\Phi}_{\ell}\coloneqq \{\eta>0 \colon \ell \text{ is } (\eta, \Phi)\text{-mixable}\}$, we define the \emph{generalized mixability constant} of $(\ell, \Phi)$ by $\eta_{\ell}^{\Phi}\coloneqq \sup\mathfrak{H}^{\Phi}_{\ell}$, if $\mathfrak{H}^{\Phi}_{\ell} \neq \varnothing$; and $0$ otherwise. 
		\end{definition}

Reid et al. \citep{Reid2015} introduced the \textsc{GAA} (see Algorithm \ref{GAA2}) which uses an entropy function $\Phi \colon \mathbb{R}^k \rightarrow \mathbb{R} \cup \{+ \infty \}$ and a substitution function $\mathfrak{S}_{\ell}$ (see previous section) to specify the learner's merging strategy $\mathfrak{M}$. It was shown that the \textsc{GAA} reduces to the AA when $\Phi$ is the Shannon entropy $\Se$. It was also shown that under some regularity conditions on $\Phi$, the \textsc{GAA} achieves a constant regret  in the $\mathfrak{G}^n_{\ell}(\mathcal{A},k)$ game for any finite, $(\eta, \Phi)$-mixable loss. 
	
Our definition of $\Phi$-mixability differs slightly from that of Reid et al. \cite{Reid2015}  --- we use directional derivatives to define the divergence $D_{\Phi}$. This distinction makes it possible to extend the \textsc{GAA} to losses which can take infinite values (such as the $\log$-loss defined on the simplex). We show, in this case, that a constant regret is still achievable under the $(\eta, \Phi)$-mixability condition. Before presenting this result, we define the notion of $\Delta$-\emph{differentiability}; for $\cI \subseteq [k]$, let $\Delta_{\cI}\coloneqq \{\bm{q}\in \Delta_k\colon q_{\theta} = 0, \forall \theta \notin \cI\}$. We say that an entropy $\Phi$ is $\Delta$-differentiable if $\forall \cI \subseteq [k]$, $\forall \bm{u},\bm{u}_0 \in \Ri \Delta_{\cI}$, the map $\bm{z} \mapsto \Phi'(\bm{u};\bm{z})$ is linear on $\mathcal{L}^0_{\cI}\coloneqq \{\lambda (\bm{v}- \bm{u}_0)\colon (\lambda,\bm{v})\in \mathbb{R}\times \Delta_{\cI}\}$.
\begin{theorem}
\label{constant regret}
Let $\Phi\colon \mathbb{R}^k \rightarrow \mathbb{R} \cup\{+\infty\}$ be a $\Delta$-differentiable entropy. Let $\ell\colon \mathcal{A}\rightarrow [0,+\infty]^n$ be a loss (not  necessarily finite) such that $\br_{\ell}$ is twice differentiable on $]0,+\infty[^n$. If $\ell$ is $(\eta, \Phi)$-mixable then the \textsc{GAA} achieves a constant regret in the $\mathfrak{G}^n_{\ell}(\mathcal{A},k)$ game; for any sequence $(x^t, \bm{a}^t_{1:k})_{t=1}^T$,
	\begin{align}
	\op{Loss}^{\ell}_{\textsc{GAA}}(T) - \min_{\theta \in [k]} \op{Loss}^{\ell}_{\theta}(T) \leq  R^{\Phi}_{\ell}\coloneqq \inf_{\bm{q} \in \Delta_k} \max_{\theta \in [k]} D_{\Phi}(\bm{e}_{\theta}, \bm{q})/\eta^{\Phi}_{\ell},\label{rbound} 
	\end{align}
for initial distribution over experts $\bm{q}^0= \argmin_{\bm{q} \in \Delta_k} \max_{\theta \in [k]} D_{\Phi}(\bm{e}_{\theta}, \bm{q})$, where $\bm{e}_{\theta}$ is the $\theta$th basis element of $\mathbb{R}^k$, and any substitution function $\mathfrak{S}_{\ell}$.
\end{theorem}
	Looking at Algorithm \ref{GAA2}, it is clear that the \textsc{GAA} is divided into two steps; 1) a \emph{substitution step} which consists of finding a prediction $\bm{a}_* \in \mathcal{A}$ satisfying the mixability condition \eqref{6:e} using a substitution function $\mathfrak{S}_{\ell}$; and 2) an \emph{update step} where a new distribution over experts is computed. 
	Except for the case of the AA with the $\log$-loss (which reduces to Bayesian updating \cite{Vovk2001}), there is not a unique choice of substitution function in general. An example of substitution function $\mathfrak{S}_{\ell}$ is the \emph{inverse loss} \citep{Williamson2014}. 
 Kamalaruban et al. \cite{Kamalaruban2015} discuss other alternatives depending on the curvature of the Bayes risk. Although the choice of $\mathfrak{S}_{\ell}$ can affect the performance of the algorithm to some extent \citep{Kamalaruban2015}, the regret bound in \eqref{rbound} remains unchanged regardless of $\mathfrak{S}_{\ell}$. 
On the other hand, the update step is well defined 
and corresponds to a \emph{mirror descent step} \cite{Reid2015} (we later use this fact to suggest a new algorithm).

\vspace{-.3cm}
		\begin{minipage}[t]{.5\textwidth}
	\begin{algorithm}[H]
\SetKwInOut{Input}{input}\SetKwInOut{Output}{output}
\Input{$\bm{q}^0 \in \Delta_k$; $\eta >0$; A $\eta$-mixable loss $\ell\colon \mathcal{A}\rightarrow [0,+\infty]^n$; A substitution function $\mathfrak{S}_{\ell}$.}
\Output{Learner's predictions $(\bm{a}^t_*)$}
\BlankLine
\For{$t=1$ \KwTo $T$}{
Observe $A^t = \bm{a}^t_{1:k} \in \mathcal{A}^k$\;
$\bm{a}^t_* \leftarrow \mathfrak{S}_{\ell}\left(-\frac{1}{\eta} \log \sum_{\theta \in [k]}q_{\theta}^{t-1}e^{{-\eta \ell(\bm{a}_{\theta}^t)}}\right)$\;
Observe outcome $x^t\in [n]$\;
$q_{\theta}^{t} \leftarrow  \dfrac{q_{\theta}^{t-1} \exp(-\eta \ell_{x^t}(\bm{a}_{\theta}^t)) }{\inner{\bm{q}^{t-1}}{\exp({-\eta \ell_{x^t}(A^t)})}}, \forall \theta \in [k]$\;
\vspace{5.1px}
}
\caption{Aggregating Algorithm}\label{AA2}
\end{algorithm}
\end{minipage}
\hspace{-3px}
	\begin{minipage}[t]{.5\textwidth}
	\begin{algorithm}[H]
\SetKwInOut{Input}{input}\SetKwInOut{Output}{output}
\Input{$\bm{q}^0 \in \Delta_k$; A $\Delta$-differentiable entropy $\Phi\colon\mathbb{R}^k \rightarrow \mathbb{R} \cup \{+\infty\}$; $\eta>0$; A $(\eta, \Phi)$-mixable loss $\ell\colon \mathcal{A}\rightarrow [0,+\infty]^n$; A substitution function $\mathfrak{S}_{\ell}$.}
\Output{Learner's predictions $(\bm{a}^t_*)$}
\BlankLine
\For{$t=1$ \KwTo $T$}{
Observe $A^t = \bm{a}^t_{1:k} \in \mathcal{A}^k$\; 
$\bm{a}^t_* \leftarrow \mathfrak{S}_{\ell}\left( \left[ \M^{\eta}_{\Phi}(\ell_x(A^t), \bm{q}^{t-1})\right]^{\T}_{1 \leq x \leq n} \right)$\;
Observe outcome $x^t\in [n]$\;
$\bm{q}^{t}\leftarrow \argmin\limits_{\bm{\mu}\in \Delta_k}  \inner{\bm{\mu}}{\ell_{x^t}(A^t)} + \frac{1}{\eta} D_{\Phi}(\bm{\mu}, \bm{q}^{t-1})  $\;  
}
\caption{Generalized Aggregating Algorithm}\label{GAA2}
\end{algorithm}
\end{minipage}

We conclude this subsection with two new and important results which will lead to a characterization of $\Phi$-mixability. The first result shows that $(\eta, \Se)$-mixability is equivalent to $\eta$-mixability, and the second rules out losses and entropies for which $\Phi$-mixability is not possible.
\begin{theorem}
\label{16:}
Let $\eta>0$.  A loss $\ell \colon \mathcal{A} \rightarrow [0,+\infty]^n$ is $\eta$-mixable if and only if $\ell$ is $(\eta,\Se)$-mixable.
\end{theorem}
\begin{proposition}
\label{14:}
Let $\Phi\colon \mathbb{R}^k \rightarrow \mathbb{R} \cup \{+\infty\}$ be an entropy and $\ell \colon \mathcal{A}  \rightarrow [0, +\infty]^n$. If $\ell$ is $\Phi$-mixable, then the Bayes risk satisfies $\br_{\ell} \in C^1(]0,+\infty[^n)$. If, additionally, $\br_{\ell}$ is twice differentiable on $]0,+\infty[^n$, then $\Phi$ must be strictly convex on $\Delta_k$. 
\end{proposition}
It should be noted that since the Bayes risk of a loss $\ell$ must be differentiable for it to be $\Phi$-mixable for some entropy $\Phi$, Theorem \ref{5:} says that we can essentially work with a proper support loss $\sell$ of $\ell$. This will be crucial in the proof of the sufficient condition of $\Phi$-mixability (Theorem \ref{18:}).

\subsection{A Characterization of $\Phi$-Mixability}
\label{mainresults}
	In this subsection, we first show that given an entropy $\Phi \colon \mathbb{R}^k\rightarrow \mathbb{R} \cup \{+\infty \}$ and a loss $\ell \colon \mathcal{A} \rightarrow [0,+\infty]^n$ satisfying certain regularity conditions, $\ell$ is $\Phi$-mixable if and only if 
	\begin{align}\boxed{\mbox{$ \underline{\eta_{\ell}} \Phi - \Se$ is convex on $\Delta_k$.}}\label{15:e} \end{align} 
		\begin{theorem}
	\label{17:} Let $\eta>0$, $\ell \colon \mathcal{A} \rightarrow [0, +\infty]^n$ a $\eta$-mixable loss, and $\Phi \colon \mathbb{R}^k \rightarrow \mathbb{R} \cup \{+ \infty\}$ an entropy. If $\eta \Phi - \Se$ is convex on $\Delta_k$, then $\ell$ is $\Phi$-mixable.
\end{theorem}
 The converse of Theorem \ref{17:} also holds under additional smoothness conditions on $\Phi$ and $\ell$;

	\begin{theorem}
		\label{18:} Let $\ell \colon \mathcal{A} \rightarrow [0,+\infty]^n$ be a loss such that $\br_{\ell}$ is twice differentiable on $]0,+\infty[^n$, and $\Phi \colon \mathbb{R}^k \rightarrow \mathbb{R} \cup \{+ \infty\}$ an entropy such that $\tilde{\Phi} \coloneqq \Phi \circ \amalg_k$ is twice differentiable on $\Int \tilde{\Delta}_k$. Then $\ell$ is $\Phi$-mixable only if $\underline{\eta_{\ell}} \Phi - \Se$ is convex on $\Delta_k$.
	\end{theorem}	
	As consequence of Theorem \ref{18:}, if a loss $\ell$ is not classically mixable, i.e. $\eta_{\ell} =0$, it cannot be $\Phi$-mixable for any entropy $\Phi$. This is because $\eta_{\ell} \Phi - \Se \stackrel{*}{=} \underline{\eta_{\ell}} \Phi - \Se = - \Se$ is not convex (where equality `*' is due to Theorem \ref{8:}).
	
We need one more result before arriving at \eqref{15:e}; Recall that the mixability constant $\eta_{\ell}$ is defined as the supremum of the set $\mathfrak{H}_{\ell} \coloneqq$ \{$\eta\geq 0\colon \ell$ is $\eta$-mixable\}. The next lemma essentially gives a sufficient condition for this supremum to be attained when $\mathfrak{H}_{\ell}$ is non-empty --- in this case, $\ell$ is $\eta_{\ell}$-mixable.
\begin{lemma}
\label{attained}
Let $\ell\colon \mathcal{A}\rightarrow [0,+\infty]^n$ be a loss. If $\Dom \ell =\mathcal{A}$, then either $\mathfrak{H}_{\ell} =\varnothing$ or $\eta_{\ell} \in \mathfrak{H}_{\ell}$.
\end{lemma}
\begin{theorem}
\label{19:}
Let $\ell$ and $\Phi$ be as in Theorem \ref{18:} with $\Dom \ell =\mathcal{A}$. Then $\eta_{\ell} = \underline{\eta_{\ell}}$. Furthermore, $\ell$ is $\Phi$-mixable if and only if $\underline{\eta_{\ell}} \Phi - \Se$ is convex on $\Delta_k$.
\end{theorem}
\begin{proof}
Suppose now that $\ell$ is mixable. By Lemma \ref{attained}, it follows that $\ell$ is $\eta_{\ell}$-mixable, and from Theorem \ref{16:}, $\ell$ is $(\eta_{\ell}^{-1}\Se)$-mixable. Substituting $\Phi$ for $\eta_{\ell}^{-1}\Se$ in Theorem \ref{18:} implies that $(\underline{\eta_{\ell}}/\eta_{\ell} -1)\Se$ is convex on $\Ri \Delta_k$. Thus, $\eta_{\ell}\leq \underline{\eta_{\ell}}$, and  since from Theorem \ref{8:} $\underline{\eta_{\ell}} \leq \eta_{\ell}$, we conclude that $\eta_{\ell} = \underline{\eta_{\ell}}$. 

From Theorem \ref{18:}, if $\ell$ is $\Phi$-mixable then $\underline{\eta_{\ell}} \Phi - \Se$ is convex on $\Delta_k$. Now suppose that $\underline{\eta_{\ell}} \Phi - \Se$ is convex on $\Delta_k$. This implies that $\underline{\eta_{\ell}} >0$, and thus from Theorem \ref{8:}, $\ell$ is $\underline{\eta_{\ell}}$-mixable. Now since $\ell$ is $\underline{\eta_{\ell}}$-mixable and $\underline{\eta_{\ell}} \Phi - \Se$ is convex on $\Delta_k$, Theorem \ref{17:} implies that $\ell$ is $\Phi$-mixable.
\end{proof}
Note that the condition `$\Dom \ell = \mathcal{A}$' is in practice not a restriction to finite losses --- see Remark \ref{inf-loss}. Theorem \ref{19:} implies that under the regularity conditions of Theorem \ref{18:}, the Bayes risk $\br_{\ell}$ [resp. $(\br_{\ell}, \Phi)$] contains all necessary information to characterize classical [resp. generalized] mixability. 
	\begin{corollary}[The Generalized Mixability Constant]
		\label{20:}
		Let $\ell$ and $\Phi$ be as in Theorem \ref{19:}. Then the generalized mixability constant (see Definition \ref{9:}) is given by 
		\begin{align}
		\eta_{\ell}^{\Phi}  = \underline{\eta_{\ell}}	 \inf_{\substack{\tbm{q}\in \Int \tilde{\Delta}_k}}	  \lambda_{\min}(\mathsf{H}\tilde{\Phi}(\tbm{q}) (\mathsf{H}\tilde{\Se}(\tbm{q}))^{-1}),
		\label{16:e}
		\end{align}
		where $\tilde{\Phi} \coloneqq \Phi \circ \amalg_k, \tilde{\Se} = \Se \circ \amalg_k$, and $\amalg_k$ is defined in \eqref{amalg}.
	\end{corollary}
	Observe that when $\Phi=\Se$, \eqref{16:e} reduces to $\eta^{\Se}_{\ell}= \underline{\eta_{\ell}}$ as expected from Theorem \ref{16:} and Theorem \ref{19:}.

\subsection{The (In)dependence Between $\ell$ and $\Phi$ and the Fundamental Nature of $\Se$}
So far, we showed that the $\Phi$-mixability of losses satisfying Assumption \ref{B:} is characterized by the convexity of $\eta \Phi - \Se$, where $\eta \in ]0,\eta_{\ell}]$ (see Theorems \ref{17:} and \ref{18:}). As a result, and contrary to what was conjectured previously \citep{Reid2015}, the generalized mixability condition does not induce a correspondence between losses and entropies; for a given loss $\ell$, there is no particular entropy $\Phi^{\ell}$ --- specific to the choice of $\ell$ --- which minimizes the regret of the GAA. Rather, the Shannon entropy $\Se$ minimizes the regret regardless of the choice of $\ell$ (see Theorem \ref{21:} below). This reflects one fundamental aspect of the Shannon entropy.

Nevertheless, given a loss $\ell$ and entropy $\Phi$, the curvature of the loss surface $\mathcal{S}_{\ell}$ determines the maximum `learning rate' $\eta^{\Phi}_{\ell}$ of the GAA; the curvature of $\mathcal{S}_{\ell}$ is linked to $\underline{\eta_{\ell}}$ through the Hessian of the Bayes risk (see Theorem \ref{38:} in Appendix \ref{b1:}), which is in turn linked to $\eta^{\Phi}_{\ell}$ through \eqref{16:e}.

Given a loss $\ell$, we now use the expression of $\eta^{\Phi}_{\ell}$ in \eqref{16:e} to explicitly compare the regret bounds $R^{\Phi}_{\ell}$ and $R^{\Se}_{\ell}$ achieved with the \textsc{GAA} (see \eqref{rbound}) using entropy $\Phi$ and the Shannon entropy $\Se$, respectively. 
	\begin{theorem}
		\label{21:}
		Let $\Se, \Phi \colon \mathbb{R}^k\rightarrow \mathbb{R}\cup \{+\infty\}$, where $\Se$ is the Shannon entropy and $\Phi$ is an entropy such that $\tilde{\Phi}\coloneqq \Phi \circ \amalg_k$ is twice differentiable on $\Int \tilde{\Delta}_k$. A loss $\ell\colon \mathcal{A}\rightarrow [0,+\infty[^n$ with $\br_{\ell}$ twice differentiable on $]0,+\infty[^n$, is $\Phi$-mixable only if $R^{\Se}_{\ell}\leq R^{\Phi}_{\ell}$.
	\end{theorem}
Theorem \ref{21:} is consistent with Vovk's result \cite[\S 5]{Vovk1998} which essentially states that the regret bound $R^{\Se}_{\ell} = \eta_{\ell}^{-1} \log k$ is in general tight for $\eta$-mixable losses.

\section{Adaptive Generalized Aggregating Algorithm}
In this section, we take advantage of the similarity between the GAA's update step and the mirror descent algorithm (see Appendix \ref{s2.2}) to devise a modification to the \textsc{GAA} leading to improved regret bounds in certain cases. The \textsc{GAA} can be modified in (at least) two immediate ways; 1) changing the learning rate at each time step to speed-up convergence; and 2) changing the entropy, i.e. the regularizer $\Phi$, at each time step --- similar to the \emph{adaptive} mirror descent algorithm \cite{Steinhardt2014,DBLP:journals/ml/OrabonaCC15}. In the former case, one can use Corollary \ref{20:} to calculate the maximum `learning rate' under the $\Phi$-mixability constraint. Here, we focus on the second method; changing the entropy at each round. Algorithm \ref{AAA} displays the modified \textsc{GAA} --- which we call the \emph{Adaptive Generalized Aggregating Algorithm} (AGAA) --- in its most general form. In Algorithm \ref{AAA}, $\Phi^{\star}(\bm{z}) \coloneqq \sup_{\bm{q}\in\Delta_k } \inner{\bm{q}}{\bm{z}} - \Phi(\bm{q})$ is the \emph{entropic dual} of $\Phi$.   
\begin{algorithm}
\SetKwFunction{subfunc}{SubstitutionFunction}
\SetKwInOut{Input}{input}\SetKwInOut{Output}{output}
\Input{$\bm{\theta}^1 =\bm{0}  \in \mathbb{R}^k$; A $\Delta$-differentiable entropy $\Phi\colon \mathbb{R}^k \rightarrow \mathbb{R}\cup\{+\infty\}$; $\eta>0$; A $(\eta, \Phi)$-mixable loss $\ell\colon \mathcal{A}\rightarrow [0,+\infty[^n$; A substitution function $\mathfrak{S}_{\ell}$; A protocol of choosing $\bm{\beta}^t$ at round $t$.}
 \Output{Learner's predictions $(\bm{a}^t_*)$}
\BlankLine
\For{$t=1$ \KwTo $T$}{
Let $\Phi_t(\bm{w}) \coloneqq \Phi(\bm{w}) - \Inner{\bm{w}}{\bm{\beta}^t - \bm{\theta}^t }$\tcp*{New entropy}
Observe $A^t \coloneqq \bm{a}^t_{1:k} \in \mathcal{A}^k$ \tcp*{Experts' predictions}
$\bm{a}_*^t\leftarrow \mathfrak{S}_{\ell}\left( \left[\M^{\eta}_{\Phi_{t}}(\ell_x(A^t), \nabla \Phi_t^{\star}(\bm{\theta}^{t}))\right]^{\T}_{1\leq x\leq n}\right)$ \tcp*{Learner's prediction}
Observe $x^t \in [n]$ and pick some $\bm{v}^t \in \mathbb{R}^k$\;
$\bm{\theta}^{t+1} \leftarrow \bm{\theta}^t - \eta \ell_{x^t}(A^t)$\;
}
\caption{Adaptive Generalized Aggregating Algorithm (AGAA)}\label{AAA}
\end{algorithm}
 \newline Given a $(\eta, \Phi)$-mixable loss $\ell$, we verify that Algorithm \ref{AAA} is well defined; for simplicity, assume that $\Dom \ell = \mathcal{A}$ and $\br_{\ell}$ is twice differentiable on $]0,+\infty[^n$. From the definition of an entropy, $|\Phi| <+ \infty$ on $\Delta_k$, and thus the entropic dual $\Phi_t^{\star}$ is defined and finite on all $\mathbb{R}^k$ (in particular at $\bm{\theta}^t$). On the other hand, from Proposition \ref{14:}, $\Phi$ is strictly convex on $\Delta_k$ which implies that $\Phi^{\star}$ (and thus $\Phi_t^{\star}$) is differentiable on $\mathbb{R}^k$ (see e.g. \cite[Thm. E.4.1.1]{Hiriart-Urruty}). It remains to check that $\ell$ is $(\eta, \Phi_t)$-mixable. Since for $\eta>0$, $(\eta, \Phi_t)$-mixability is equivalent to $(\frac{1}{\eta}\Phi_t)$-mixability (by definition), Theorem \ref{19:} implies that $\ell$ is $(\eta, \Phi_t)$-mixable if and only if $\underline{\eta_{\ell}}\eta^{-1}\Phi_t -\Se$ is convex on $\Delta_k$. This is in fact the case since $\Phi_t$ is an affine transformation of $\Phi$, and we have assumed that $\ell$ is $(\eta, \Phi)$-mixable. 

In what follows, we focus on a particular instantiation of Algorithm \ref{AAA} where we choose $\bm{\beta}^t \coloneqq -\eta\sum_{s=1}^{t-1} (\ell_{x^s}(A^s) + \bm{v}^s)$, for some (arbitrary for now) $(\bm{v}^s) \subset \mathbb{R}^k$. The $(\bm{v}^t)$ vectors act as correction terms in the update step of the AGAA. Using standard duality properties (see Appendix \ref{s2.1}), it is easy to show that the \textsc{AGAA} reduces to the \textsc{GAA} except for the update step where the new distribution over experts at round $t \in [T]$ is now given by \begin{align*}\bm{q}^t= \nabla \Phi^{\star} ( \nabla \Phi(\bm{q}^{t-1}) - \eta \ell_{x^t}(A^t) - \eta \bm{v}^t ).\end{align*}
\begin{theorem}
\label{rboundAGAA}
Let $\Phi\colon \mathbb{R}^k \rightarrow \mathbb{R} \cup\{+\infty\}$ be a $\Delta$-differentiable entropy. Let $\ell\colon \mathcal{A}\rightarrow [0,+\infty]^n$ be a loss such that $\br_{\ell}$ is twice differentiable on $]0,+\infty[^n$. Let $\bm{\beta}^t = - \eta \sum_{s=1}^{t-1} (\ell_{x^s}(A^s) + \bm{v}^s)$, where $\bm{v}^s\in \mathbb{R}^k$ and $A^s \coloneqq \bm{a}^s_{1:k} \in \mathcal{A}^k$. If $\ell$ is $(\eta, \Phi)$-mixable then for initial distribution $\bm{q}^0= \argmin_{\bm{q} \in \Delta_k} \max_{\theta \in [k]} D_{\Phi}(\bm{e}_{\theta}, \bm{q})$ and any sequence $(x^t, \bm{a}^t_{1:k})_{t=1}^T$, the \textsc{AGAA} achieves the regret
	\begin{align}
\forall \theta \in[k], \quad \op{Loss}^{\ell}_{\textsc{AGAA}}(T) -  \op{Loss}^{\ell}_{\theta}(T) \leq  R^{\Phi}_{\ell} + \Delta R_{\theta}(T), \label{98:e}
	\end{align}
	where $\Delta R_{\theta}(T) \coloneqq \sum_{t=1}^{T-1} ( v^t_{\theta} -\inner{\bm{v}^t}{\bm{q}^{t}})$.
\end{theorem}
Theorem \ref{rboundAGAA} implies that if the sequence $(\bm{v}^t)$ is chosen such that $\Delta R_{\theta^*}(T)$ is negative for the best expert $\theta^*$ (in hindsight), then the regret bound `$R^{\Phi}_{\ell} + \Delta R_{\theta^*}(T)$' of the \textsc{AGAA} is lower than that of the \textsc{GAA} (see \eqref{rbound}), and ultimately that of the AA (when $\Phi=\Se$). Unfortunately, due to Vovk's result \cite[\S 5]{Vovk1998} there is no ``universal'' choice of $(\bm{v}^t)$ which guarantees that $\Delta R_{\theta^*}(T)$ is always negative. However, there are cases where this term is expected to be negative. 

Consider a dataset where it is typical for the best experts (i.e., the $\theta^*$'s) to perform poorly at some point during the game, as measured by their average loss, for example. Under such an assumption, choosing the correction vectors $\bm{v}^t$ to be negatively proportional to the average losses of experts, i.e. $\bm{v}^t\coloneqq - \frac{\alpha}{t} \sum_{s=1}^{t} \ell_{x^s}(A^s)$ (for small enough $\alpha>0$), would be consistent with the idea of making $\Delta R_{\theta^*}(T)$ negative. 
To see this, suppose expert $\theta^*$ is performing poorly during the game (say at $t<T$), as measured by its instantaneous and average loss. At that point the distribution $\bm{q}^{t}$ would put more weight on experts performing better than $\theta^*$, i.e. having a lower average loss. And since $v^t_{\theta}$ is negatively proportional to the average loss of expert $\theta$, the quantity $v^t_{\theta^*} - \inner{\bm{v}^t}{\bm{q}^t}$ would be negative --- consistent with making $\Delta R_{\theta^*}(T)<0$. On the other hand, if expert $\theta^*$ performs well during the game (say close to the best) then $v^t_{\theta^*} - \inner{\bm{v}^t}{\bm{q}^t} \simeq 0$, since $\bm{q}^t$ would put comparable weights between $\theta^*$ and other experts (if any) with similar performance. 
\begin{example}
\label{cumloss}
\begin{proof}[(A Negative Regret)]
One can construct an example that illustrates the idea above. Consider the Brier game $\mathfrak{G}^2_{\ell_{\op{Brier}}}(\Delta_2, 2)$; a probability game with 2 experts $\{\theta_1, \theta_2\}$, 2 outcomes $\{0,1\}$, and where the loss $\ell_{\op{Brier}}$ is the \emph{Brier} loss \citep{Vovk2009} (which is $1$-mixable). Assume that; expert $\theta_1$ consistently predicts $\op{Pr}(x=0)=1/2$; expert $\theta_2$ predicts $\op{Pr}(x=0)=1/4$ during the first $50$ rounds, then switches to predicting $\op{Pr}(x=0)=3/4$ thereafter; the outcome is always $x=0$. A straightforward simulation using the \textsc{AGAA} with the Shannon entropy, Vovk's substitution function for the Brier loss \cite{Vovk2009}, $\bm{\beta}^t$ as in Theorem \ref{rboundAGAA} with $\bm{v}^t \coloneqq - \frac{1}{8t} \sum^{t}_{s=1} \ell_{\op{Brier}}(x^s, A^s)$, yields $R^{\Phi}_{\ell_{\op{Brier}}} + \Delta R_{\theta^*}(T) \simeq -5,$ $\forall T \geq 150$, where in this case $\theta^*=\theta_2$ is the best expert for $ T\geq 150$. The learner then does \emph{better} than the best expert. If we use the AA instead, the learner does worse than $\theta_2$ by $\simeq R^{\Se}_{\ell_{\op{Brier}}} = \log 2$. 
\end{proof}
\end{example}
In real data, the situation described above --- where the best expert does not necessarily perform optimally during the game --- is typical, especially when the number of rounds $T$ is large. We have tested the aggregating algorithms on real data as studied by Vovk \cite{Vovk2009}. We compared the performance of the AA with the AGAA, and found that the \textsc{AGAA} outperforms the AA, and in fact achieved a negative regret on two data sets. Details of the experiments are in Appendix \ref{Experiment}.

As pointed out earlier, there are situations where $\Delta R_{\theta^*}(T)\geq 0$ even for the choice of $(\bm{v}^t)$ in Example \ref{cumloss}, and this could potentially lead to a large positive regret for the AGAA. There is an easy way to remove this risk at a small price; the outputs of the \textsc{AGAA} and the AA can themselves be considered as expert predictions. These predictions can in turn be passed to a new instance of the AA to yield a \emph{meta prediction}. The resulting worst case regret is guaranteed not to exceed that of the original AA instance by more than $\eta^{-1} \log 2$ for an $\eta$-mixable loss. We test this idea in Appendix \ref{Experiment}.

\section{Discussion and Future Work}
	\label{Discussion}
In this work, we derived a characterization of $\Phi$-mixability, which enables a better understanding of when a constant regret is achievable in the game of prediction with expert advice. Then, borrowing techniques from mirror descent, we proposed a new ``adaptive'' version of the generalized aggregating algorithm. We derived a regret bound for a specific instantiation of this algorithm and discussed certain situations where the algorithm is expected to perform well. We empirically demonstrated the performance of this algorithm on football game predictions (see Appendix \ref{Experiment}).

Vovk \cite[\S 5]{Vovk1998} essentially showed that given an $\eta$-mixable loss there is no algorithm that can achieve a lower regret bound than $\eta^{-1}\log k$ on all sequences of outcomes. There is no contradiction in trying to design algorithms which perform well in expectation (maybe better than the AA) on ``typical'' data while keeping the worst case regret close to $\eta^{-1}\log k$. This was the motivation behind the AGAA.
In future work, we will explore other choices for the correction vector $\bm{v}^t$ with the goal of lowering the (expected) bound in \eqref{98:e}. In the present work, we did not study the possibility of varying the learning rate $\eta$. One might obtain better regret bounds using an adaptive learning rate as is the case with the mirror descent algorithm. Our Corollary \ref{20:} is useful in that it gives an upper bound on the maximal learning rate under the $\Phi$-mixability constraint. Finally, although our Theorem \ref{21:} states that worst-case regret of the \textsc{GAA} is minimized when using the Shannon entropy, it would be interesting to study the dynamics of the \textsc{AGAA} with other entropies.

\begin{table}[h]
\caption{A short list of the main symbols used in the paper}
\label{notation}
\centering
\begin{tabular}{ll}
\toprule
Symbol & Description \\
\midrule
$\ell$ & A loss function defined on a set $\mathcal{A}$ and taking values in $[0,+\infty]^n$ (see Sec. \ref{Preliminaries})\\
$\mathscr{S}_{\ell}$& The finite part of the superprediction set of a loss $\ell$ (see Sec. \ref{Preliminaries})\\
$\underline{\ell}$ & The support loss of a loss $\ell$ (see Def. \ref{supportloss})\\
$\br_{\ell}$ & The Bayes risk corresponding to a loss $\ell$ (see Definition \ref{3:}) \\
$\tbr_{\ell}$ & The composition of the Bayes risk with an affine function; $\tbr_{\ell} \coloneqq \br_{\ell} \circ \amalg_n $ (see \eqref{amalg}) \\
$\op{S}$ & The Shannon Entropy (see \eqref{shannon})\\
$\eta_{\ell}$ &The mixability constant of $\ell$ (see Def. \ref{etamix}) ; essentially the largest $\eta$ s.t. $\ell$ is $\eta$-mixable. \\
$\underline{\eta_{\ell}}$ & Essentially the largest $\eta$ such that $\eta \br_{\ell} - \br_{\log}$ is convex (see \eqref{5:e} and \cite{DBLP:journals/jmlr/ErvenRW12}) \\
$\eta_{\ell}^{\Phi}$ & The generalized mixability constant (see Def. \ref{9:}); the largest $\eta$ s.t. $\ell$ is $(\eta, \Phi)$-mixable. \\
$\mathfrak{S}_{\ell}$ & A substitution function of a loss $\ell$ (see Sec. \ref{aa})\\
$R^{\Phi}_{\ell}$ & The regret achieved by the \textsc{GAA} using entropy $\Phi$ (see \eqref{rbound} and Algorithm \ref{GAA2})\\
\bottomrule
\end{tabular}
\end{table}

\subsubsection*{Acknowledgments}
This work was supported by the Australian Research Council and DATA61.

\bibliography{biblio}
\bibliographystyle{plain}

\newpage

\tableofcontents
\newpage

	\begin{appendix}
	\section{Notation and Preliminaries}
	\label{s2.1}
	For $n\in \mathbb{N}$, we define $\tilde{n} = n -1$. We denote $[n] \coloneqq \{1, \dots, n\}$ the set of integers between $1$ and $n$. Let $\inner{\cdot}{\cdot}$ denote the standard inner product in $\mathbb{R}^n$ and $\norm{\cdot}$ the corresponding norm. Let $I_n$ and $\bm{1}_n$ denote the $n \times n$ identity matrix and the vector of all ones in $\mathbb{R}^n$. Let $\bm{e}_1, \dots, \bm{e}_n$ denote the \emph{standard basis} for $\mathbb{R}^n$. For a set $\cI \subsetneq \mathbb{N}$ and $\bm{r}_1,\dots, \bm{r}_n  \in \mathbb{R}^k$, we denote $[\bm{r}_i]_{i\in \cI} \coloneqq [\bm{r}_{i_1},\dots, \bm{r}_{i_k}] \in \mathbb{R}^{n \times k}$, where $\cI =\{i_1,\dots, i_k\}$ and $i_1< \dots < i_k$. We denote its transpose by $[\bm{r}_i]^{\mathsf{T}}_{i \in \cI} \in \mathbb{R}^{k \times n}$. For two vectors $\bm{p}, \bm{q} \in \mathbb{R}^n$, we write $\bm{p}\leq \bm{q}$ [resp. $\bm{p}< \bm{q} $], if $\forall i \in [n], p_i \leq q_i$ [resp. $p_i < q_i$]. We also denote $\bm{p} \odot \bm{q} =[p_i q_i]^{\mathsf{T}}_{1 \leq i \leq n} \in \mathbb{R}^n$ the \emph{Hadamard product} of $\bm{p}$ and $\bm{q}$. If $(\bm{c}_k)$ is a sequence of vectors in $\mathcal{C}\subseteq \mathbb{R}^n$, we simply write $(\bm{c}_k) \subset \mathcal{C}$. For a sequence $(\bm{v}_m) \subset \mathbb{R}^n$, we write $\bm{v}_m \stackrel{m\to \infty}{\to} \bm{v}$ or $\lim_{m\to \infty} \bm{v}_m = \bm{v}$, if $\forall i \in [n], \lim_{m\to \infty} [\bm{v}_m]_i = v_i$. For a square matrix $A \in \mathbb{R}^{n\times n}$, $\lambda_{\min}(A)$ [resp. $\lambda_{\max}(A)$] denotes its minimum  [resp. maximum] eigenvalue. For $k \geq 1$, $\bm{u} \in [0, + \infty[^k$ and $\bm{w} \in \mathbb{R}^k$, we define $\log \bm{u} \coloneqq [\log u_i]^{\mathsf{T}}_{1 \leq  i \leq k} \in \mathbb{R}^k$ and $\exp \bm{w} \coloneqq [\exp w_i]^{\mathsf{T}}_{1 \leq  i \leq k} \in \mathbb{R}^k$.
	
	Let $\Delta_n \coloneqq \{ \bm{p} \in [0,1]^n: \inner{\bm{p}}{\bm{1}_n}=1  \}$ be the \emph{probability simplex} in $\mathbb{R}^n$. We also define $\tilde{\Delta}_n \coloneqq \{\tbm{p} \in [0,+\infty[^{\tilde{n}}: \inner{\tbm{p}}{\bm{1}_{\tilde{n}}} \leq 1 \}$. We will use the notations $\Delta^k_n \coloneqq (\Delta_n)^k$ and $\tilde{\Delta}^k_n \coloneqq (\tilde{\Delta}_n)^k$. For $\cI \subseteq [n]$, the set $\Delta_{\cI} = \{\bm{q}\in \Delta_n: q_{i} =0, \forall i \in [n]\setminus \cI \}$ is a $|\cI|$-\emph{face} of $\Delta_n$. We denote $\Pi^n_{\cI}: \mathbb{R}^n \to \mathbb{R}^{|\cI|}$ the linear projection operator satisfying $\Pi^n_{\cI}\bm{u} = [u_{i}]^{\mathsf{T}}_{ i \in  \cI}$. If there is no ambiguity from the context, we may simply write $\Pi_{\cI}$ instead of $\Pi^n_{\cI}$. It is easy to verify that $ \Pi_{\cI}\Pi_{\cI}^{\mathsf{T}} = I_{|\cI|}$ and that $\bm{q}\mapsto \Pi_{\cI}\bm{q}$ is a bijection from $\Delta_{\cI}\subseteq \Delta_n$ to $\Delta_{|\cI|}$. In the special case where $\cI=[\tilde{n}]$, we write $\Pi_n \coloneqq \Pi^n_{[\tilde{n}]}$ and we define the affine operator $\amalg_n: \mathbb{R}^{\tilde{n}} \to \mathbb{R}^n$ by $\amalg_n(\bm{u}) \coloneqq [u_1, \dots, u_{\tilde{n}}, 1 - \inner{\bm{u}}{\bm{1}_{\tilde{n}}}]^{\mathsf{T}}=  J_n \bm{u} + \bm{e}_n $, where $J_n \coloneqq  \begin{bsmallmatrix} I_{\tilde{n}} \\ -\bm{1}^{\mathsf{T}}_{\tilde{n}} \end{bsmallmatrix} \in \mathbb{R}^{n \times \tilde{n}}$. 
	
	For $\bm{u}\in \mathbb{R}^n$ and $c \in \mathbb{R}$, we denote $\mathcal{H}_{\bm{u},c}\coloneqq \{\bm{y} \in \mathbb{R}^n: \inner{\bm{y}}{\bm{u}} \leq c \}$ and $\mathcal{B}(\bm{u}, c) \coloneqq \{\bm{v} \in \mathbb{R}^n: \norm{\bm{u} - \bm{v}} \leq c \}$. $\mathcal{H}_{\bm{u},c}$ is a closed half space and $\mathcal{B}(\bm{u}, c)$ is the $c$\emph{-ball} in $\mathbb{R}^n$ centered at $\bm{u}$.  Let $\mathcal{C} \subseteq \mathbb{R}^n$ be a non-empty set. We denote $\Int \mathcal{C}$, $\Ri \mathcal{C}$, $\Bd \mathcal{C}$, and $\Rbd \mathcal{C}$ the \emph{interior}, \emph{relative interior}, \emph{boundary}, and \emph{relative boundary} of a set $\mathcal{C}\in \mathbb{R}^n $, respectively \citep{Hiriart-Urruty}. We denote the \emph{indicator function} of $\mathcal{C}$ by $\I_\mathcal{C}$, where for $\bm{u} \in \mathcal{C},$ $\I_\mathcal{C}(\bm{u})=0$, otherwise $\I_{\mathcal{C}}(\bm{u}) = +\infty$. The \emph{support function} of $\mathcal{C}$ is defined by 
	\begin{align*}
	\sigma_\mathcal{C}(\bm{u}) \coloneqq \sup_{\bm{s} \in \mathcal{C}} \Inner{\bm{u}}{\bm{s}},\;  \bm{u} \in \mathbb{R}^n.
	\end{align*}

Let $f\colon \mathbb{R}^n \rightarrow \mathbb{R} \cup \{+\infty\}$. We denote $\Dom f \coloneqq \{\bm{u} \in \mathbb{R}^n: f(\bm{u}) < +\infty \}$ the \emph{effective domain} of $f$. The function $f$ is \emph{proper} if $\Dom f \neq \varnothing$. The function $f$ is \emph{convex} if $\forall (\bm{u}, \bm{v}) \in \mathbb{R}^n$ and $\lambda \in ]0,1[$, $f(\lambda \bm{u} + (1-\lambda) \bm{v}) \leq  \lambda f(\bm{u}) + (1 - \lambda) f(\bm{v})$. When the latter inequality is strict for all $\bm{u} \neq \bm{v}$, $f$ is \emph{strictly convex}. When $f$ is convex, it is \emph{closed} if it is \emph{lower semi-continuous}; that is, for all $\bm{u} \in \mathbb{R}^n$, $\lim \inf_{\bm{v} \to \bm{u}} f(\bm{v}) \geq f(\bm{u})$. The function $f$ is said to be 1-\emph{homogeneous} if $\forall(\bm{u}, \alpha) \in \mathbb{R}^n \times ]0,+\infty[, f(\alpha\bm{u}) = \alpha f(\bm{u})$, and it is said to be 1-\emph{coercive} if $ \frac{f(\bm{u})}{\norm{\bm{u}}} \to +\infty$ as $\norm{\bm{u}} \to \infty$. Let $f$ be proper. The \emph{sub-differential} of $f$ is defined by 
	\begin{align*}
	\partial f(\bm{u}) \coloneqq \{\bm{s}^* \in \mathbb{R}^n: f(\bm{v}) \geq f(\bm{u}) +  \Inner{\bm{s}^*}{\bm{v} - \bm{u}}, \forall \bm{v} \in \mathbb{R}^n  \}.
	\end{align*}
 Any element $\bm{s} \in \partial f(\bm{u})$ is called a \emph{sub-gradient} of $f$ at $\bm{u}$. We say that $f$ is \emph{directionally differentiable} if for all $(\bm{u}, \bm{v}) \in \Dom f \times \mathbb{R}^n$ the limit $ \lim_{t \downarrow 0} \frac{f(\bm{u} + t \bm{v}) - f(\bm{u})}{t}$ exists in $[-\infty,+\infty]$. In this case, we denote the limit by $f'(\bm{u}; \bm{v})$. When $f$ is convex, it is directionally differentiable \citep{Rockafellar1997a}. Let $f$ be proper and directionally differentiable. The \emph{divergence} generated by $f$ is the map $D_{f}\colon \mathbb{R}^n \times \Dom f \rightarrow [0,+\infty]$ defined by 
			\begin{align*}
			D_{f}(\bm{v}, \bm{u})\coloneqq \left\{ \begin{array}{ll} f(\bm{v}) - f(\bm{u}) - f'(\bm{u};\bm{v} - \bm{u}), &\mbox{ if } \bm{v} \in \Dom f; \\ +\infty, & \mbox{ otherwise.}  \end{array} \right. 
			\end{align*}
For $\cI\subset [n]$ and $f_{\cI} \coloneqq f \circ \Pi_{\cI}^{\mathsf{T}}$, it is easy to verify that $f_{\cI}'(\Pi_{\cI}\bm{p}; \Pi_{\cI}\bm{q} - \Pi_{\cI}\bm{p})=f'(\bm{p}; \bm{q} - \bm{p}), \forall (\bm{p}, \bm{q}) \in \Delta_{\cI}$. In this case, it holds that $D_f(\bm{q}, \bm{p}) = D_{f_{\cI}}(\Pi_{\cI}\bm{q},\Pi_{\cI}\bm{p})$.
  If $f$ is differentiable [resp. twice differentiable] at $\bm{u} \in \Int \Dom f$, we denote $\nabla f(\bm{u}) \in \mathbb{R}^n$ [resp. $\mathsf{H} f (\bm{u}) \in \mathbb{R}^{n \times n}$] its \emph{gradient} vector [resp. \emph{Hessian} matrix] at $\bm{u}$. A vector-valued function $g\colon \mathbb{R}^n \rightarrow \mathbb{R}^{m}$ is differentiable at $\bm{u}$ if for all $i \in [m]$, $g_i$ is differentiable at $\bm{u}$. In this case, the \emph{differential} of $g$ at $\bm{u}$ is the linear operator $\mathsf{D} g(\bm{u}): \mathbb{R}^n \rightarrow \mathbb{R}^m$ defined by $\mathsf{D} g(\bm{u}) \coloneqq [\nabla g_i(\bm{u})]^{\mathsf{T}}_{1\leq i\leq m}$. If $f$ has $k$ continuous derivatives on a set $\Omega \subset \mathbb{R}^k$, we write $f \in C^k(\Omega)$.

We define $\tilde{f}: \mathbb{R}^{\tilde{n}} \to \mathbb{R} \cup \{+\infty\}$ by $\tilde{f} \coloneqq f\circ \amalg_n + \I_{\tilde{\Delta}_n} $. That is,
	\begin{align}    
	\label{1:e}
	\tilde{f}(\tbm{u}) \coloneqq  \left\{ \begin{array}{ll}  f(J_n \tbm{u} + \bm{e}_n), & \mbox{for }\tbm{u}\in \tilde{\Delta}_n; \\ +\infty, & \mbox{for } \tbm{u} \in \mathbb{R}^{n-1} \setminus \tilde{\Delta}_n. \end{array} \right.
	\end{align}
	If $\tilde{f}$ is directionally differentiable, then $f'(\bm{p}, \bm{q} - \bm{p})= \tilde{f}'(\tbm{p}, \tbm{q} - \tbm{p})$, for $\bm{p}, \bm{q} \in \Delta_n$. If $\tilde{f}$ is differentiable at $\tbm{p}=\Pi_n(\bm{p})$, then $\tilde{f}'(\tbm{p}, \tbm{q} - \tbm{p})= \inner{\nabla \tilde{f}(\tbm{p})}{\tbm{q}-\tbm{p}}$. If, additionally, $f$ is differentiable at $\bm{p} \in \Ri \Delta_k$, the chain rule yields $\nabla \tilde{f}(\tbm{p}) = J^{\mathsf{T}}_n  \nabla f(\bm{p})$. Since $J_n(\tbm{p} - \tbm{q}) = \amalg_n(\tbm{p} - \tbm{q}) = \bm{p} - \bm{q}$, it also follows that $\inner{\tbm{p} - \tbm{q}}{\nabla \tilde{f}(\tbm{p})}= \Inner{\bm{p} - \bm{q}}{\nabla {f}(\bm{p})}$. 
	
		The \emph{Fenchel dual} of a (proper) function $f$ is defined by $f^{\ast}(\bm{v}) \coloneqq \sup_{\bm{u} \in \Dom f} \inner{\bm{u}}{\bm{v}} - f(\bm{u})$, and it is a closed, convex function on $\mathbb{R}^n$ \citep{Hiriart-Urruty}. The following proposition gives some useful properties of the Fenchel dual which will be used in several proofs. 
	\begin{proposition}[\cite{Hiriart-Urruty}]
		\label{1:}
		Let $f,h:\mathbb{R}^n\rightarrow \mathbb{R} \cup \{+\infty\}$. If $f$ and $h$ are proper and there are affine functions minorizing them on $\mathbb{R}^n$, then for all $ \bm{v}_0\in \mathbb{R}^n$
		\begin{align*}
		\begin{array}{llcl}
		(i)\quad &	g(\bm{u}) = f(\bm{u}) +r, \; \forall \bm{u}  &\implies&  g^*(\bm{v}) = f^*(\bm{v}) - r, \; \forall \bm{v}\\
		(ii)\quad&	g(\bm{u}) = f(\bm{u}) +\inner{\bm{v}_0}{\bm{u}}, \; \forall \bm{u}  &\implies &g^{\ast}(\bm{v}) = f^{\ast}(\bm{v}- \bm{v}_0), \; \forall \bm{v}\\ 
		(iii)\quad&	 f \leq h&\implies & f^{\ast}\geq h^{\ast},\\
(iv)\quad&	 \mbox{$\bm{s} \in \partial f^*(\bm{v})$}, \bm{v}\in \mathbb{R}^n & \implies& f^*(\bm{v}) =\inner{\bm{v}}{\bm{s}} - f(\bm{s}), \\ 
	(v)\quad&	  g(\bm{u}) = f(t\bm{u}),\; t>0, \forall \bm{u} &\implies &g^{\ast}(\bm{v}) = f^{\ast}(\bm{v}/t), \\ 
		\end{array}
		\end{align*}
	\end{proposition}
			A function $\Phi \colon \mathbb{R}^k\rightarrow \mathbb{R} \cup \{+\infty\}$ is an \emph{entropy} if it is closed, convex, and $\Delta_k \subseteq \Dom \Phi$. Its \emph{entropic dual} $\Phi^{\star}: \mathbb{R}^k \rightarrow  \mathbb{R} \cup \{+ \infty\}$ is defined by $\Phi^{\star}(\bm{z}) \coloneqq \sup_{\bm{q} \in \Delta_{k}} \inner{\bm{q}}{\bm{z}} - \Phi(\bm{q}), \bm{z}\in \mathbb{R}^k$. For the remainder of this paper, we consider entropies defined on $\mathbb{R}^k$, where $k\geq 2$.

Let $\Phi \colon \mathbb{R}^k\rightarrow \mathbb{R} \cup \{+ \infty \}$ be an entropy and $\Phi_{\Delta} \coloneqq \Phi + \I_{\Delta_k}$. In this case, $\Phi^{\star} =  \Phi^*_{\Delta}$. It is clear that $\Phi_{\Delta}$ is 1-coercive, and therefore, $\Dom \Phi^{\star} = \Dom \Phi_{\Delta}^* = \mathbb{R}^k$ \citep[Prop. E.1.3.8]{\Hi}. The entropic dual of $\Phi$ can also be expressed using the Fenchel dual of $\tilde{\Phi}: \mathbb{R}^{k-1} \rightarrow \mathbb{R} \cup \{+\infty\}$ defined by \eqref{1:e} after substituting $f$ by $\Phi$ and $n$ by $k$. In fact,
	\begin{align}
	\Phi^{\star}(\bm{z}) &= \sup_{\tbm{q} \in \tilde{\Delta}_{k}} \Inner{J_k \tbm{q}+ \bm{e}_k}{\bm{z}} - \Phi(J_k \tbm{q}+ \bm{e}_k), \nonumber \\
	& = \Inner{\bm{e}_k}{\bm{z}} + \sup_{\tbm{q} \in \tilde{\Delta}_{k}} \Inner{\tbm{q}}{J^{\mathsf{T}}_k \bm{z}} - \tilde{\Phi}(\tbm{q}),  \nonumber \\
	& = \Inner{\bm{e}_k}{\bm{z}} + \tilde{\Phi}^{\ast}(J^{\mathsf{T}}_k\bm{z}), \label{2:e}
	\end{align}
	where \eqref{2:e} follows from the fact that $\Dom \tilde{\Phi} = \tilde{\Delta}_k$. Note that when $\Phi$ is an entropy, $\tilde{\Phi}$ is a closed convex function on $\mathbb{R}^{k-1}$. Hence, it holds that $\tilde{\Phi}^{**} = \tilde{\Phi}$ \citep{Rockafellar1997a}. 
	
		The \emph{Shannon entropy} by $\Se(\bm{q}) \coloneqq \sum_{\substack{i\in[k]:q_i \neq 0 }} q_i \log q_i$,\footnote{The Shannon entropy is usually defined with a minus sign. However, it will be more convenient for us to work without it.} if $\bm{q}\in [0, +\infty[^k$; and $+\infty$ otherwise.

We will also make use of the following lemma.
\begin{lemma}[\cite{bernstein11}]
\label{31:}
$\forall m \geq 1, \forall A, B \in \mathbb{R}^{m \times m}$, $\lambda_{\max}(AB)=\lambda_{\max}(BA)$ and $\lambda_{\min}(AB)=\lambda_{\min}(BA)$.
\end{lemma}

\section{Technical Lemmas}
\label{A:ap}
This appendix presents technical lemmas which will be needed in various proofs of results from the main body of the paper.

For an open convex set $\Omega$ in $\mathbb{R}^n$ and $\alpha>0$, a function $\phi\colon \Omega\rightarrow \mathbb{R}$ is said to be $\alpha$-\emph{strongly convex} if $\bm{u} \mapsto \phi(\bm{u}) - \alpha \norm{\bm{u}}^{2}$ is convex on $\Omega$ \citep{merentes2010remarks}. The next lemma is a characterization of a generalization of $\alpha$-strong convexity, where $\bm{u} \mapsto \norm{\bm{u}}^2$ is replaced by any strictly convex function. 
	\begin{lemma}
		\label{22:}
		Let $\Omega \subseteq \mathbb{R}^n$ be an open convex set. Let $\phi, \psi\colon \Omega \rightarrow \mathbb{R}$ be twice differentiable.

If $\psi$ is strictly convex, then $\forall \bm{u} \in\Omega$, $\mathsf{H} \psi(\bm{u})$ is invertible, and for any $\alpha >0$
		\begin{align}
		\forall \bm{u} \in \Omega,\; \lambda_{\min} (\mathsf{H} \phi(\bm{u}) (\mathsf{H} \psi (\bm{u}))^{-1} ) \geq \alpha \iff  \phi - \alpha \psi \mbox{ is convex}, \label{17:e}
		\end{align}
Furthermore, if $\alpha>1$, then the left hand side of \eqref{17:e} implies that $\phi - \psi$ is strictly convex.
	\end{lemma}
\begin{proof}
Suppose that $\inf_{\bm{u} \in \Omega}  \lambda_{\min} (\mathsf{H} \phi(\bm{u}) (\mathsf{H} \psi(\bm{u}))^{-1} ) \geq \alpha$. Since $g$ is strictly convex and twice differentiable on $\Omega$, $\mathsf{H} \psi(\bm{u})$ is symmetric positive definite, and thus invertible. Therefore, there exists a symmetric positive definite matrix $G \in \mathbb{R}^{n \times n}$ such that $G G = \mathsf{H} \psi(\bm{u})$. Lemma \ref{31:} implies
			\begin{eqnarray*}
			& \inf_{\bm{u} \in \Omega}  \lambda_{\min} (\mathsf{H} \phi(\bm{u}) (\mathsf{H} \psi(\bm{u}))^{-1} ) & \geq \alpha, \\ 
\iff & \inf_{\bm{u} \in \Omega}  \lambda_{\min} (G^{-1} \mathsf{H} \phi(\bm{u}) G^{-1})&\geq \alpha, \\
			 \iff & \forall \bm{u} \in \Omega, \forall \bm{v} \in \mathbb{R}^n \setminus \{\bm{0}\},  \frac{\bm{v}^{\mathsf{T}} G^{-1} (\mathsf{H} \phi(\bm{u})) G^{-1} \bm{v}}{ \bm{v}^{\mathsf{T}} \bm{v}}& \geq \alpha, \\
			\iff & \forall \bm{u} \in \Omega, \forall \bm{w} \in \mathbb{R}^n \setminus \{\bm{0}\}, \bm{w}^{\mathsf{T}}  (\mathsf{H} \phi(\bm{u})) \bm{w} &\geq \alpha \bm{w}^{\mathsf{T}} G G \bm{w}= \bm{w}^{\mathsf{T}} (\alpha \mathsf{H} \psi(\bm{u})) \bm{w}, \\
			 \iff & \forall \bm{u} \in \Omega,  \mathsf{H} \phi(\bm{u}) & \succeq \alpha \mathsf{H} \psi(\bm{u}),  \\
			 \iff & \forall \bm{u} \in \Omega,  \mathsf{H} (\phi- \alpha \psi)(\bm{u})  &\succeq 0,
			\end{eqnarray*}
			where in the third and fifth lines we used the definition of minimum eigenvalue and performed the change of variable $\bm{w} = G^{-1}\bm{v}$, respectively.  To conclude the proof of \eqref{17:e}, note that the positive semi-definiteness of $\mathsf{H} (\phi - \alpha \psi)$ is equivalent to the convexity of $\phi-\alpha \psi$ \citep[Thm B.4.3.1]{Hiriart-Urruty}.
			
			Finally, note that the equivalences established above still hold if we replace $\alpha$, ``$\geq$'', and ``$\succeq$ ''  by $1$, ``$>$'', and ``$\succ$'' , respectively. The strict convexity of $\phi-\psi$ then follows from the positive definiteness of $\mathsf{H} (\phi - \psi)$ (\ibid).
\end{proof}

The following result due to \cite{Chernov2010} will be crucial to prove the convexity of the superprediction set (Theorem \ref{6:}).
		\begin{lemma}[\cite{Chernov2010}]\label{24:}
			Let $\Delta(\Omega)$ be the set of distributions over some set $\Omega \subseteq \mathbb{R}$. Let a function $Q: \Delta(\Omega) \times \Omega \rightarrow \mathbb{R}$ be such that $Q(\cdot, \omega)$ is continuous for all $\omega \in \Omega$. If for all
			$\bm{\pi} \in \Delta(\Omega)$ it holds that $\mathbb{E}_{\omega \sim \bm{\pi}} Q(\bm{\pi},\omega) \leq r$, where $r\in \mathbb{R}$ is some constant, then
			\begin{align*}
			\exists \bm{\pi} \in \Delta(\Omega),  \forall \omega \in \Omega, \;   Q(\bm{\pi}, \omega) \leq r.
			\end{align*}
		\end{lemma}
		Note that when $\Omega$ in the lemma above is $[n]$, $\Delta([n]) \equiv \Delta_n$.

The next crucial lemma is a slight modification of a result due to \cite{Chernov2010}.
\begin{lemma}
\label{23:}
Let $f\colon \Ri \Delta_n \times [n] \rightarrow \mathbb{R}$ be a continuous function in the first argument and such that $\forall (\bm{q},x) \in \Ri \Delta_n \times [n], -\infty < f(\bm{q},x)$. Suppose that
$\forall \bm{p}\in \Ri \Delta_n,  \mathbb{E}_{x \sim \bm{p}}[f(\bm{p}, x)] \leq 0$,
then 
\begin{align*}
 \forall \epsilon >0, \exists \bm{p}_{\epsilon} \in \Ri \Delta_n, \forall x\in[n], f(\bm{p}_{\epsilon},x) \leq \epsilon.
\end{align*}
\end{lemma}
\begin{proof}
Pick any $\delta>0$ such that $\delta(n-1)<1$, and $c_0 <0$ such that $\forall (\bm{q},x) \in \Ri \Delta_n \times [n], c_0 \leq f(\bm{q},x)$. We define $\Delta_n^{\delta} \coloneqq \{\bm{p} \in \Delta_n:  \forall x\in[n], p_x \geq \delta\}$ and $g(\bm{q}, \bm{p}) \coloneqq \mathbb{E}_{x \sim \bm{q}}[f(\bm{p}, x)]$. For a fixed $\bm{q}$, $\bm{p} \mapsto g(\bm{q},\bm{p})$ is continuous, since $f$ is continuous in the first argument. For a fixed $\bm{p}$,  $\bm{q} \mapsto g(\bm{q},\bm{p})$ is linear, and thus concave. Since $\Delta_n^{\delta}$ is convex and compact, $g$ satisfies Ky Fan's minimax Theorem \citep[Thm. 11.4]{Agarwal2001}, and therefore, there exists $\bm{p}^{\delta}\in \Delta_n^{\delta}$ such that 
\begin{align}
\label{19:e}
\forall \bm{q} \in \Delta_n^{\delta}, \;\; \mathbb{E}_{x \sim \bm{q}}[f(\bm{p}^{\delta}, x)] = g(\bm{q}, \bm{p}^{\delta}) \leq \sup_{\bm{\mu} \in \Delta^{\delta}_n} g(\bm{\mu}, \bm{\mu})  = \sup_{\bm{\mu} \in \Delta^{\delta}_n} \mathbb{E}_{x \sim \bm{\mu}}[f(\bm{\mu}, x)] \leq 0.
\end{align}
For $x_0\in [n]$, let $\hat{\bm{q}} \in \Delta^{\delta}_n$ be such that $\hat{q}_{x_0} = 1 - \delta(n-1) $ and $\hat{q}_{x} =\delta$ for $x\neq x_0$ (this is a legitimate distribution since $\delta (n-1)<1$ by construction). Substituting $\hat{\bm{q}}$ for $\bm{q}$ in \eqref{19:e} gives 
\begin{align*}
\begin{array}{lrl}
&(1 - \delta (n-1)) f(\bm{p}^{\delta}, x_0) + \delta \sum_{x\neq x_0} f(\bm{p}^{\delta}, x ) &\leq 0, \\
\implies & (1 - \delta (n-1)) f(\bm{p}^{\delta}, x_0) & \leq -c_0 \delta (n-1), \\
\implies & f(\bm{p}^{\delta}, x_0) &\leq [-c_0 \delta (n-1)]/[1 - \delta (n-1)].
\end{array}
\end{align*}
Choosing $\delta^* \coloneqq \epsilon/ [(-c_0 + \epsilon)(n-1)]$, and $\bm{p}_{\epsilon} \coloneqq \bm{p}^{\delta^*}$ gives the desired result.

\end{proof}

		\begin{lemma}
			\label{25:}
			Let $f, g\colon I \rightarrow \mathbb{R}^n$, where $I \subseteq \mathbb{R}$ is an open interval containing $0$. Suppose $g$ [resp. $f$] is continuous [resp. differentiable] at $0$. Then $t \mapsto \inner{f(t)}{g(t)}$ is differentiable at $0$ if and only if $t \mapsto \inner{f(0)}{g(t)}$ is differentiable at $0$, and we have
			\begin{align*}
			\left. \frac{d}{dt}  \inner{f(t)}{g(t)} \right|_{t=0} = \Inner{\left.\frac{d}{dt}  f(t)\right|_{t=0}}{g(0)} +\left. \frac{d}{dt}  \Inner{f(0)}{g(t)}\right|_{t=0}.
			\end{align*}
		\end{lemma}	
		
		\begin{proof}
			We have 
			\begin{align*}
			\frac{\Inner{f(t)}{g(t)} - \Inner{f(0)}{g(0)}}{t}  & =  \frac{\Inner{f(t)}{g(t)} - \Inner{f(0)}{g(t)}}{t} + \frac{\Inner{f(0)}{g(t)} - \Inner{f(0)}{g(0)}}{t},  \\
			& = \Inner{\frac{f(t) - f(0)}{t}}{g(t)}  + \frac{\Inner{f(0)}{g(t)} - \Inner{f(0)}{g(0)}}{t}.
			\end{align*}
			
			But since $g$ [resp. $f$] is continuous [resp. differentiable] at $0$, the first term on the right hand side of the above equation converges to $ \inner{\left.\frac{d}{dt}  f(t)\right|_{t=0}}{g(0)}$ as $t \to 0$. Therefore, $\frac{1}{t} (\inner{f(0)}{g(t)} - \inner{f(0)}{g(0)})$ admits a limit when $t\to 0$ if and only if $\frac{1}{t}(\inner{f(t)}{g(t)} - \inner{f(0)}{g(0)})$ admits a limit when $t\to 0$. This shows that $t \mapsto \inner{f(0)}{g(t)}$ is differentiable at $0$  if an only if $t \mapsto \inner{f(t)}{g(t)}$ is differentiable at $0$, and in this case the above equation yields
			\begin{align*}
			\left. \frac{d}{dt}  \Inner{f(t)}{g(t)}\right|_{t=0} &= \lim_{t \to 0}  \frac{\Inner{f(t)}{g(t)} - \Inner{f(0)}{g(0)}}{t}, \\
			& = \lim_{t \to 0} \left( \Inner{\frac{f(t) - f(0)}{t}}{g(t)}  + \frac{\Inner{f(0)}{g(t)} - \Inner{f(0)}{g(0)}}{t}  \right), \\
			& =\Inner{\left.\frac{d}{dt}  f(t)\right|_{t=0}}{g(0)} +\left. \frac{d}{dt}  \Inner{f(0)}{g(t)}\right|_{t=0}. 
			\end{align*}
		\end{proof}
		
		Note that the differentiability of $t \mapsto \inner{f(0)}{g(t)}$ at $0$ does not necessarily imply the differentiability of $g$ at 0. Take for example $n=3$, $f(t)  = \bm{1}/3$ for $t \in ]-1,1[$, and \[ g(t) = \left\{ \begin{array}{ll} -t \bm{e}_1  + t \frac{\bm{1}}{3}, &  \mbox{ if } t \in ]-1,0[; \\
		-t \frac{\bm{1}}{3} + t \bm{e}_2, & \mbox{ if } t \in [0, 1[. 
		\end{array}  \right.\]
		
		Thus, the function $t \mapsto \inner{f(0)}{g(t)} = 0$ is differentiable at $0$ but $g$ is not. The preceding Lemma will be particularly useful in settings where it is desired to compute the derivative $\der{\inner{f(0)}{g(t)}}$ without any explicit assumptions on the differentiability of $g(t)$ at $0$. For example, this will come up when computing $\der{\inner{\bm{p}}{D\tilde{\ell}(\tbm{\alpha}^t)\bm{v} }}$, where $\bm{v} \in \mathbb{R}^{n-1}$ and $t \mapsto \tbm{\alpha}^t$ is smooth curve on $\Int \tilde{\Delta}_n$, with the only assumption that $\tilde{L}_{\ell}$ is twice differentiable at $\tbm{\alpha}^0 \in \Int \tilde{\Delta}_n$.

\begin{lemma}
		\label{26:}
		Let $\ell \colon \Delta_n \rightarrow [0,+\infty]^n$ be a proper loss. For any $\bm{p}\in \Ri \Delta_n$, it holds that
		\begin{align*}
		\ell\mbox{ is continuous at }\bm{p} \stackrel{(i)}{\iff} \br_{\ell}\mbox{ is differentiable at }\bm{p} \stackrel{(ii)}{\iff} \partial [-\br_{\ell}] (\bm{p}) = \{\nabla \br_{\ell} (\bm{p})\} = \{\ell(\bm{p})\}.
		\end{align*}
	\end{lemma}
		\begin{proof}
			[$\stackrel{(i)}{\iff}$] This equivalence has been shown before by \cite{DBLP:journals/jmlr/WilliamsonVR16}.
			
			[$\stackrel{(ii)}{\iff}$] Since $\br_{\ell}(\bm{p}) = - \sigma_{\sps}(-\bm{p})$, for all $\bm{p}\in \Ri \Delta_n$, it follows that $\br_{\ell}$ is differentiable at $\bm{p}$ if and only if $\partial [-\br_{\ell}](\bm{p}) = \partial \sigma_{\sps}(-\bm{p}) = \{- \nabla \sigma_{\sps} (-\bm{p})\}= \{\nabla \br_{\ell}(\bm{p})\}$ \citep[Cor. D.2.1.4]{Hiriart-Urruty}. 
			It remains to show that $\nabla \br_{\ell} (\bm{r}) = \ell(\bm{r})$ when $\br_{\ell}$ is differentiable at $\bm{r} \in \Ri \Delta_n$. Let $\bm{\alpha}^t_x =   \bm{r} + t \bm{e}_x$ and $\tbm{\alpha}_x^t = \Pi_n(\bm{\alpha}^t_x)$, where $(\bm{e}_x)_{x\in [n]}$ is the standard basis of $\mathbb{R}^n$.  For $x \in [n]$, the functions $f_x(t) \coloneqq \bm{\alpha}_x^t$ and $g_x(t) \coloneqq \tilde{\ell}(\tbm{\alpha}_x^t)$ satisfy the conditions of Lemma \ref{25:}. Therefore, $h_x(t) \coloneqq \inner{f_x(0)}{g_x(t)} = \inner{\bm{r}}{\tilde{\ell}(\tbm{\alpha}^t_x)}$ is differentiable at $0$ and 
			\begin{align*}
			\nabla \tbr(\bm{r}) \bm{e}_x &= \Der{\tbr(\bm{\alpha}_x^t)} = \Der{\Inner{f_x(t)}{g_x(t)}}, \\
			&=\Inner{\bm{e}_x}{\tilde{\ell}(\tbm{r})}
			+ \Der{h_x(t)},\\
			&= \tilde{\ell}_x(\tbm{r}), 
			\end{align*}
			where the last equality holds because $h_x$ attains a minimum at $0$ due to the properness of $\ell$. The result being true for all $x \in [n]$ implies that $\nabla \tbr (\tbm{r}) = \tilde{\ell}(\tbm{r}) = \ell(\bm{r})$.
		\end{proof}
		
		The next Lemma is a restatement of earlier results due to \cite{DBLP:journals/jmlr/ErvenRW12}. Our proof is more concise due to our definition of the Bayes risk in terms of the support function of the superprediction set. 
		
		\begin{lemma}[\cite{DBLP:journals/jmlr/ErvenRW12}]
			\label{27:}
			Let $\ell \colon \Delta_n \rightarrow [0,+\infty]^n$ be a proper loss whose Bayes risk is twice differentiable on $ ]0, +\infty[^n$ and let $X_{\bm{p}}=I_{\tilde{n}} - \bm{1}_{\tilde{n}} \tbm{p}^{\mathsf{T}} $. The following holds
			\begin{enumerate}[label=(\roman*)]  
\item $\forall \bm{p} \in \Ri \Delta_n, \inner{\bm{p}}{\mathsf{D} \tilde{\ell}(\tbm{p})} =\bm{0}_{\tilde{n}}^{\mathsf{T}}$.
				\item  $\forall \tbm{p} \in  \Int \tilde{\Delta}_n$, $\mathsf{D} \tilde{\ell} (\tbm{p})  = \begin{bsmallmatrix}
				X_{\bm{p}} \\ - \tbm{p}^{\mathsf{T}}
				\end{bsmallmatrix} \mathsf{H} \tbr_{\ell}(\tbm{p})$.
\item  $\forall \tbm{p} \in  \Int \tilde{\Delta}_n,$ $\mathsf{H} \tbr_{\log} (\tbm{p}) = - (X_{\bm{p}})^{-1} (\diag{\tbm{p}})^{-1}$.
			\end{enumerate}
		\end{lemma}
		\begin{proof}
			[We show (i) and (ii)] Let $\bm{p} \in \Ri \Delta_n$ and $f(\tbm{q}) \coloneqq \inner{\bm{p}}{\tilde{\ell}(\tbm{q})}=\inner{\bm{p}}{\nabla \br_{\ell}(\bm{q})}$, where the equality is due to Lemma \ref{26:}. Since $\br_{\ell}$ is twice differentiable $]0,+\infty[^n$, $f$ is differentiable on $\Int \tilde{\Delta}_n$ and we have $\mathsf{D} f (\tbm{q}) = \inner{\bm{p}}{\mathsf{D} \tilde{\ell} (\tbm{p})}$. Since $\ell$ is proper, $f$ reaches a minimum at $\tbm{p} \in \Int \Delta_n$, and thus $\inner{\bm{p}}{\mathsf{D} \tilde{\ell} (\tbm{p})}= \bm{0}_{\tilde{n}}^{\mathsf{T}}$ (this shows (i)).
			On the other hand, we have $\nabla \tbr_{\ell}(\tbm{p}) = J_n^{\mathsf{T}}  \nabla \br_{\ell} (\bm{p}) = J_n^{\mathsf{T}} \tilde{\ell}(\tbm{p})$. By differentiating and using the chain the rule, we get $\mathsf{H} \tbr_{\ell}(\tbm{p}) =  [ \mathsf{D} \tilde{\ell} (\tbm{p})]^{\mathsf{T}} J_n$. This means that $\forall i \in [\tilde{n}],$ $[\mathsf{H}\tbr_{\ell}(\tbm{p})]_{\bcdot, i} = \nabla \tilde{\ell}_i (\tbm{p}) - \nabla \tilde{\ell}_n (\tbm{p})$, and thus $\sum_{i=1}^{\tilde{n}}p_i [\mathsf{H}\tbr_{\ell}(\tbm{p})]_{\bcdot, i} =  \sum_{i=1}^{\tilde{n}} p_i \nabla \tilde{\ell}_i (\tbm{p}) - (1- p_n)  \nabla \tilde{\ell}_n (\tbm{p}) $. On the other hand, it follows from point (i) of the lemma that $\sum_{i=1}^n p_i \nabla \tilde{\ell}_i(\tbm{p}) =\bm{0}_{\tilde{n}}$. Therefore, $[\mathsf{H} \tbr_{\ell}(\tbm{p})] \tbm{p}= - \nabla \tilde{\ell}_n (\tbm{p})$ and, as a result, $\forall i \in [\tilde{n}]$,  $[\mathsf{H}\tbr_{\ell}(\tbm{p})]_{\bcdot, i} -[\mathsf{H} \tbr_{\ell}(\tbm{p})]\tbm{p} =\nabla \tilde{\ell}_i (\tbm{p}) $. The last two equations can be combined as $\mathsf{D} \tilde{\ell} (\tbm{p})  = \begin{bsmallmatrix}
			X_{\bm{p}} \\ - \tbm{p}^{\mathsf{T}}
			\end{bsmallmatrix} \mathsf{H} \tbr_{\ell}(\tbm{p})$.

			[We show (iii)] It follows from $(ii)$, since $\forall i \in [\tilde{n}], \nabla [\tilde{\ell}_{\log}]_i (\tbm{p})= \frac{1}{p_i} \bm{e}_i $, for $\tbm{p} \in \Int \tilde{\Delta}_n$.
			
		\end{proof}

In the next lemma we state a new result for proper losses which will be crucial to prove a necessary condition for $\Phi$-mixability (Theorem \ref{18:}) --- one of the main results of the paper.
	\begin{lemma}
		\label{28:}
		Let $\ell \colon \Delta_n\rightarrow [0,+\infty]^n$ be a proper loss whose Bayes risk is twice differentiable on $]0, +\infty[^n$. For $\bm{v} \in \mathbb{R}^{n-1}$ and $\tbm{p}\in \Int \tilde{\Delta}_n$, 
		\begin{align}
		\Inner{\bm{p}}{(\mathsf{D} \tilde{\ell}(\tbm{p}) \bm{v}) \odot (\mathsf{D} \tilde{\ell}(\tbm{p}) \bm{v})}&= -\bm{v}^{\mathsf{T}} \mathsf{H} \tbr_{\ell}(\tbm{p})  [\mathsf{H} \tbr_{\log}(\tbm{p})]^{-1} \mathsf{H} \tbr_{\ell}(\tbm{p}) \bm{v}, \label{20:e}
		\end{align}	
		where $\bm{p}=\amalg_n(\tbm{p})$ and $\br_{\log}$ is the Bayes risk of the $\log$ loss. 
		
		Furthermore, if $t \mapsto \tbm{\alpha}^t$ is a smooth curve in $\Int \tilde{\Delta}_n$ and satisfies $\tbm{\alpha}^0 = \tbm{p}$ and $\Der{\tbm{\alpha}^t}=\bm{v}$, then $t \mapsto \inner{\bm{p}}{ \mathsf{D} \tilde{\ell}(\tbm{\alpha}^t) \bm{v}}$ is differentiable at $0$ and we have 
		\begin{align}
		\Der{ \Inner{\bm{p}}{\mathsf{D} \tilde{\ell}(\tbm{\alpha}^t) \bm{v}}}  &= -\bm{v}^{\mathsf{T}} \mathsf{H} \tbr_{\ell}(\tbm{p}) \bm{v}. \label{21:e}  
		\end{align}
	\end{lemma}
	\begin{proof}
	 We know from Lemma \ref{27:} that for $\tbm{p} \in  \Int \tilde{\Delta}_n$, we have $\mathsf{D} \tilde{\ell} (\tbm{p})  = \begin{bsmallmatrix}
			X_{\bm{p}} \\ - \tbm{p}^{\mathsf{T}}
			\end{bsmallmatrix} \mathsf{H} \tbr_{\ell}(\tbm{p})$, where $X_{\bm{p}}=I_{n-1} - \bm{1}_{n-1} \tbm{p}^{\mathsf{T}}$. Thus, we can write
			\begin{align}
			\Inner{\bm{p}}{\mathsf{D} \tilde{\ell}(\tbm{p}) \bm{v} \odot \mathsf{D} \tilde{\ell}(\tbm{p}) \bm{v}} & =  \bm{v}^{\mathsf{T}} (\mathsf{D} \tilde{\ell}(\tbm{p}))^{\mathsf{T}} \diag{\bm{p}} \mathsf{D} \tilde{\ell}(\tbm{p}) \bm{v}, \nonumber \\
			& =  \bm{v}^{\mathsf{T}}  (\mathsf{H} \tbr_{\ell}(\tbm{p}))^{\mathsf{T}}  [
			X_{\bm{p}}^{\mathsf{T}}, \; - \tbm{p}
			]  \diag{\bm{p}} \smat{
			X_{\bm{p}} \\ - \tbm{p}^{\mathsf{T}}
			} \mathsf{H} \tbr_{\ell}(\tbm{p}) \bm{v}.  \label{22:e}
			\end{align}
			Observe that $[
			X_{\bm{p}}^{\mathsf{T}},  - \tbm{p}]  \diag{\bm{p}}  = [
			I_{n-1} - \tbm{p} \bm{1}^{\mathsf{T}}_{n-1},  - \tbm{p}
			] \diag{\bm{p}} =[
			\diag{\tbm{p}} - \tbm{p} \tbm{p}^{\mathsf{T}}, \; - \tbm{p}p_n
			]$. Thus, 
			\begin{align}
			[
			X_{\bm{p}}^{\mathsf{T}},  - \tbm{p}
			]  \diag{\bm{p}} \smat{
			X_{\bm{p}} \\ - \tbm{p}^{\mathsf{T}}
			}  & = [
			\diag{\tbm{p}} - \tbm{p} \tbm{p}^{\mathsf{T}},  - \tbm{p}p_n] \smat{
			I_{n-1} - \bm{1}_{n-1} \tbm{p}^{\mathsf{T}}  \\ - \tbm{p}^{\mathsf{T}}
			}, \nonumber \\
			& =\diag{\tbm{p}} - \tbm{p} \tbm{p}^{\mathsf{T}}  - \tbm{p} \tbm{p}^{\mathsf{T}} + \tbm{p} \tbm{p}^{\mathsf{T}} (1 - p_n) + p_n \tbm{p} \tbm{p}^{\mathsf{T}},  \nonumber \\
			& = \diag{\tbm{p}} - \tbm{p} \tbm{p}^{\mathsf{T}},  \nonumber \\
			& = \diag{\tbm{p}} X_{\bm{p}},\nonumber \\
			& = -(\mathsf{H} \tbr_{\log} (\tbm{p}))^{-1}, \label{23:e}
			\end{align}
			where the last equality is due to Lemma \ref{27:}. The desired result follows by combining \eqref{22:e} and \eqref{23:e}.
			
			[\textbf{We show \eqref{21:e}}] Let $\tbm{p}\in \Int \tilde{\Delta}_n$, we define $\tbm{\alpha}^t \coloneqq \tbm{p} + t \bm{v}$, $\bm{\alpha}^t \coloneqq \amalg_n(\tbm{\alpha}^t) = \bm{p} + t J_n \bm{v}$, and $r(t) \coloneqq \bm{\alpha}^t / \norm{\bm{\alpha}^t}$, where $t \in \{s: \tbm{p} + s \bm{v} \in \Int \tilde{\Delta}_n \}$. Since $t \mapsto r(t)$ is differentiable at $0$ and $t \mapsto \mathsf{D} \tilde{\ell} (\tbm{\alpha}^t)\bm{v}$ is continuous at $0$, it follows from Lemma \ref{22:} that
			\begin{align*} 
			\Der{\Inner{r(0)}{\mathsf{D} \tilde{\ell} (\tbm{\alpha}^t) \bm{v}}}& =\Der{\Inner{r(t)}{\mathsf{D} \tilde{\ell} (\tbm{\alpha}^t) \bm{v}}} - \Inner{\dot r(0) }{\mathsf{D} \tilde{\ell}(\tbm{p})\bm{v}}, \\
			& = - \Inner{\dot r(0)}{\mathsf{D} \tilde{\ell}(\tbm{p})\bm{v}},
			\end{align*}
			where the second equality holds since, according to Lemma \ref{27:}, we have $ \inner{\bm{\alpha}^t}{\mathsf{D} \tilde{\ell} (\tbm{\alpha}^t) \bm{v}}=0$. 
			Since $r(0)=\bm{p}/ \norm{\bm{p}}$, $\dot r(0)= \norm{\bm{p}}^{-1}(I_{n} - r(0) [r(0)]^{\mathsf{T}})  J_n \bm{v}$, and $J_n = \smat{ I_{n-1} \\ -\bm{1}^{\mathsf{T}}_{n-1}}$, we get 
			\begin{align}
			\norm{\tbm{p}} \Der{\Inner{r(0)}{\mathsf{D} \tilde{\ell} (\tbm{\alpha}^t) \bm{v}}}&= - \Inner{ \left(I_{n} - r(0) [r(0)]^{\mathsf{T}}\right)  J_n \bm{v}}{\mathsf{D} \tilde{\ell} (\tbm{p}) \bm{v}}, \nonumber  \\
			& = -\Inner{J_n \bm{v}}{\mathsf{D} \tilde{\ell}(\tbm{p}) \bm{v}}, \label{24:e}  \\
			& = -\Inner{J_n \bm{v}}{ \smat{
				X_{\bm{p}}\nonumber \\ - \tbm{p}^{\mathsf{T}}
				} \mathsf{H} \tbr_{\ell}(\tbm{p}) \bm{v}},  \nonumber \\
			& = -\bm{v}^{\mathsf{T}} \mathsf{H} \tbr_{\ell}(\tbm{p}) \bm{v}, \nonumber
			\end{align}
			where the passage to \eqref{24:e} is due to $r(0)= \bm{p}/\norm{\bm{p}} \perp \mathsf{D} \tilde{\ell}(\tbm{p})$. In the last equality we used the fact that $J_n^{\mathsf{T}} \smat{
			X_{\bm{p}} \\ - \tbm{p}^{\mathsf{T}}
			} = [
			I_{n-1}, - \bm{1}_{n-1}
			]\smat{
			I_{n-1} - \bm{1}_{n-1}\tbm{p} \\ - \tbm{p}^{\mathsf{T}}
			}= I_{n-1}$.
		\end{proof}
		
\begin{proposition}
		\label{12:}
		Let $\Phi \colon \mathbb{R}^k \rightarrow \mathbb{R} \cup \{ +\infty\}$ be an entropy and $\ell \colon \mathcal{A} \rightarrow [0, +\infty]^n$ a closed admissible loss. If $\ell$ is $\Phi$-mixable, then $\forall \cI \subseteq [k]$ with $|\cI| >1$, $\ell$ is $\Phi_{\cI}$-mixable and
\begin{align} 
\label{8:e}
\forall \bm{q} \in \Rbd \Delta_{\cI}, \forall \hat{\bm{q}} \in \Ri \Delta_{\cI}, \;  \Phi'(\bm{q}; \hat{\bm{q}} - \bm{q}) = -\infty.
\end{align}
\end{proposition}

Given an entropy $\Phi \colon \mathbb{R}^k \rightarrow \mathbb{R} \cup\{+\infty\}$ and a loss $\ell \colon \mathcal{A} \rightarrow [0,+\infty]$, we define 
			\[  \m_{\Phi}(x, A, \bm{a},  \hat{\bm{q}}, \bm{\mu}) \coloneqq  \inner{\bm{\mu}}{ \ell_x(A)} +  D_{\Phi} (\bm{\mu}, \hat{\bm{q}}) - \ell_x(\bm{a}), \]
			where $x \in [n]$, $A\in \mathcal{A}^k$, $\bm{a}\in \mathcal{A}$, and $\bm{q}, \hat{\bm{q}} \in \Delta_k$. Reid et al. \cite{Reid2015} showed that
			$\ell$ is $\Phi$ mixable if and only if \begin{align*} \widehat{\m_{\Phi}} \coloneqq \inf_{A \in \mathcal{A}^k, \hat{\bm{q}} \in \Delta_k } \sup_{\bm{a}_* \in \mathcal{A}} \inf_{ \bm{\mu} \in \Delta_k, x\in[n] }  \m_{\Phi}(x, A, \bm{a}, \hat{\bm{q}}, \bm{\mu})  \geq 0.\end{align*}
		
		\begin{proof}[\textbf{Proof of Proposition \ref{12:}}]

[\textbf{We show that $\ell$ is $\Phi_{\cI}$-mixable}]		
Let $\cI \subseteq [k]$, with $|\cI|>1$, $A \in \mathcal{A}^k$, and $\bm{q} \in \Delta_{\cI}$. Since $\ell$ is $\Phi$-mixable, the following holds
\begin{align}
\exists \bm{a}_* \in \Delta_n, \forall x\in [n], \; \ell_x(\bm{a}_{*}) &\leq \inf_{\hat{\bm{q}} \in \Delta_k} \Inner{\hat{\bm{q}}}{\ell_x(A)} + D_{\Phi}(\hat{\bm{q}}, \bm{q}), \label{33:e}  \\
& \leq \inf_{\hat{\bm{q}} \in \Delta_{\cI}} \Inner{\hat{\bm{q}}}{\ell_x(A)} + D_{\Phi}(\hat{\bm{q}}, \bm{q}), \label{34:e} \\
& =  \inf_{\hat{\bm{q}} \in \Delta_{\cI}} \Inner{\Pi_{\cI}^{}\hat{\bm{q}}}{\Pi_{\cI}^{}\ell_x(A)} + D_{\Phi_{\cI}}(\Pi_{\cI}^{}\hat{\bm{q}}, \Pi_{\cI}^{}\bm{q}), \nonumber \\
& =  \inf_{\hat{\bm{\mu}} \in \Delta_{|\cI|} } \Inner{\hat{\bm{\mu}}}{\ell_x(A\Pi_{\cI}^{\mathsf{T}})} + D_{\Phi_{\cI}}(\hat{\bm{\mu}}, \Pi_{\cI}^{}\bm{q}), \label{35:e}
\end{align}
where in \eqref{33:e} we used the fact that $\Phi_{\cI} ( \Pi_{\cI}^{} \bm{q}) = \Phi(\bm{q}), \forall \bm{q} \in \Delta_{\cI}$. Given that $A \mapsto A\Pi_{\cI}^{\mathsf{T}}$ [resp. $\bm{q} \mapsto \Pi_{\cI}^{} \bm{q}$] is onto from $\mathcal{A}^k$ to $\mathcal{A}^{|\cI|}$ [resp. from $\Delta_{\cI}$ to $\Delta_{|\cI|}$], \eqref{35:e} implies that $\ell$ is $\Phi_{\cI}$-mixable. 

[\textbf{We show \eqref{8:e}}]
Suppose that there exists $\hat{\bm{q}} \in \Rbd \Delta_k$ and $\bm{q} \in \Ri \Delta_k$ such that $|\Phi'(\hat{\bm{q}}; \bm{q} - \hat{\bm{q}})| < +\infty$. Let $f\colon [0, \epsilon] \to \mathbb{R}$ be defined by $f(\lambda) \coloneqq  \Phi(\hat{\bm{q}} + \lambda (\bm{q} - \hat{\bm{q}}))$, where $\epsilon>0$ is such that $\hat{\bm{q}} + \epsilon (\bm{q} - \hat{\bm{q}}) \in \Ri \Delta_k$. The function $f$ is closed and convex on $\Dom f = [0, \epsilon]$ and $\lim_{\lambda \downarrow 0 } \frac{f(\lambda) - f(0)}{\lambda} = f'(0; 1) = \Phi'(\hat{\bm{q}}; \bm{q} - \hat{\bm{q}})$ which is finite by assumption. Using this and the fact that $\lambda f'(0;1) = f'(0; \lambda)$, we have $\lim_{\lambda \downarrow 0} \lambda^{-1} (f(\lambda) - f(0) - f'(0;\lambda)) = 0$. Substituting $f$ by its expression in terms of $\Phi$ in the latter equality gives  
				\begin{align}
				\label{36:e}
				\lim_{\lambda \downarrow 0 }  \lambda^{-1} D_{\Phi} (\hat{\bm{q}} + \lambda (\bm{q} - \hat{\bm{q}}),\hat{\bm{q}}) = 0. 
				\end{align}
				
				Let $\eta>0$ and $\theta^* \in [k]$ be such that $\hat{q}_{\theta^*} = 0$. Suppose that $\ell$ is an admissible, $\Phi$-mixable loss. The fact that $\ell$ is admissible implies that there exists $(x_0, x_1, \bm{a}_0,\bm{a}_1) \in [n]\times [n] \times \mathcal{A}\times \mathcal{A}$ such that \citep{Reid2015}
\begin{align}
\label{3:e}
\bm{a}_1 \in \argmin \{\ell_{x_0}(\bm{a}): \ell_{x_1}(\bm{a}) = \inf_{\hat{\bm{a}} \in \mathcal{A}} \ell_{x_1} (\hat{\bm{a}}) \} \mbox{ and } \inf_{\bm{a} \in \mathcal{A}} \ell_{x_0} (\bm{a}) = \ell_{x_0}(\bm{a_0}) < \ell_{x_0} (\bm{a}_1).
\end{align}				
	In particular, it holds that $\ell_{x_0}(\bm{a}_0) < \ell_{x_0}(\bm{a}_1)$. Fix $A \in \mathcal{A}^k$, such that $ A_{\bcdot, \theta^*}= \bm{a}_0$ and $A_{\bcdot, \theta} = \bm{a}_1$ for $\theta \in [k]\setminus \{ \theta^*\}$.
				Let \[\bm{a}_* \coloneqq \argmax_{\bm{a} \in \Delta_n} \inf_{ \bm{\mu} \in \Delta_k, x\in[n] }   \m_{\Phi}(x, A, \bm{a}, \hat{\bm{q}}, \bm{\mu}),\] with $\hat{\bm{q}} \in \Rbd \Delta_k$ as in \eqref{36:e}. Note that $\bm{a}_*$ exists since $\ell$ is closed. 
				
				If $\bm{a}_*$ is such that $\ell_{x_1}(\bm{a}_*) > \ell_{x_1}(\bm{a}_1)$, then taking $\bm{\mu} =  \hat{\bm{q}}$ puts all weights on experts predicting $\bm{a}_1$, while $D_{\Phi}(\bm{\mu}, \hat{\bm{q}}) = 0$. Therefore, \[ \widehat{\m_{\Phi}} \leq \inf_{\bm{\mu} \in \Delta_k, x\in[n]}  \m_{\Phi}(x, A, \bm{a}_*, \hat{\bm{q}}, \bm{\mu})  \leq  \m_{\Phi}(x_1, A, \bm{a}, \hat{\bm{q}}, \hat{\bm{q}}) <0. \] This contradicts the $\Phi$-mixability of $\ell$. Therefore, $\ell_{x_1}(\bm{a}_*)=\ell_{x_1}(\bm{a}_1)$, which by \eqref{3:e} implies $\ell_{x_0}(\bm{a}_*)\geq \ell_{x_0}(\bm{a}_1)$. For $\bm{q}^{\lambda}=\hat{\bm{q}}+\lambda(\bm{q}-\hat{\bm{q}})$, with $\bm{q} \in \Ri \Delta_k$ as in \eqref{33:e} and $\lambda \in [0, \epsilon]$,
				\begin{align*}
				\widehat{\m_{\Phi}} & \leq \inf_{\substack{\bm{\mu} \in \Delta_k, x\in[n]}}  \m_{\Phi}(x, A, \bm{a}_*, \hat{\bm{q}}, \bm{\mu}), \\ &  \leq  \m_{\Phi}(x_0, A, \bm{a}, \hat{\bm{q}}, \bm{q}^{\lambda}),\\ & =  \inner{\bm{q}^{\lambda}}{ \ell_{x_0}(A)} +  D_{\Phi} (\bm{q}^{\lambda}, \hat{\bm{q}}) - \ell_{x_0}(\bm{a}_*), \\
				& = (1 - \lambda q_{\theta^*}) \ell_{x_0}(\bm{a}_1) + \lambda q_{\theta^*} \ell_{x_0} (\bm{a}_0) + D_{\Phi} (\bm{q}^{\lambda}, \hat{\bm{q}}) - \ell_{x_0}(\bm{a}_*), \\
				& \leq (1 - \lambda q_{\theta^*}) \ell_{x_0}(\bm{a}_*) + \lambda q_{\theta^*} \ell_{x_0} (\bm{a}_0) + D_{\Phi} (\bm{q}^{\lambda}, \hat{\bm{q}}) - \ell_{x_0}(\bm{a}_*), \\
				& = \lambda q_{\theta^*} ( \ell_{x_0}(\bm{a}_0) - \ell_{x_0}(\bm{a}_*) ) + D_{\Phi} (\hat{\bm{q}} + \lambda (\bm{q} - \hat{\bm{q}}), \hat{\bm{q}}).
				\end{align*}
				
				Since $q_{\theta^*} >0 $ ($\bm{q} \in \Ri \Delta_k$) and $ \ell_{x_0}(\bm{a}_0) <  \ell_{x_0}(\bm{a}_1) \leq  \ell_{x_0}(\bm{a}_*) $, \eqref{33:e} implies that there exists $\lambda_*>0$ small enough such that $\lambda_* q_{\theta^*} ( \ell_{x_0}(\bm{a}_0) - \ell_{x_0}(\bm{a}_*) ) +  D_{\Phi} (\hat{\bm{q}} + \lambda_{*} (\bm{q} - \hat{\bm{q}}), \hat{\bm{q}})<0$. But this implies that $\widehat{\m_{\Phi}}<0$ which contradicts the $\Phi$-mixability of $\ell$. Therefore, $\Phi'(\hat{\bm{q}};  \bm{q}-\hat{\bm{q}})$ is either equal to $+\infty $ or $-\infty$. The former case is not possible. In fact, since $\Phi$ is convex, it must have non-decreasing slopes; in particular, it holds that $\Phi'(\hat{\bm{q}}; \bm{q} - \hat{\bm{q}}) \leq \Phi(\bm{q} - \hat{\bm{q}}) - \Phi(\hat{\bm{q}})$. Since $\Phi$ is finite on $\Delta_k$ (by definition of an entropy), we have $\Phi'(\hat{\bm{q}}; \bm{q} - \hat{\bm{q}}) < +\infty$. Therefore, we have just shown that 
\begin{align}
\label{37:e}
\forall \hat{\bm{q}} \in \Rbd \Delta_{k}, \forall \bm{q} \in \Ri \Delta_{k}, \;  \Phi'(\hat{\bm{q}}; \bm{q} - \hat{\bm{q}}) = -\infty.
\end{align}

Now suppose that $(\hat{\bm{q}}, \bm{q}) \in \Rbd \Delta_{\cI} \times \Ri \Delta_{\cI}$ for $\cI \subseteq [k]$, with $|\cI|>1$. Note that in this case, we have $(\Phi_{\cI})'(\Pi_{\cI}^{} \hat{\bm{q}}; \Pi_{\cI}^{} (\bm{q} -\hat{\bm{q}}))= \Phi'(\hat{\bm{q}}; \bm{q} - \hat{\bm{q}})$. We showed in the first step of this proof that under the assumptions of the proposition, $\ell$ must be $\Phi_{\cI}$-mixable. Therefore, repeating the steps above that lead to \eqref{37:e} for $\Phi$, $\hat{\bm{q}}$, and $\bm{q}$ substituted by $\Phi_{\cI}$, $\Pi_{\cI}^{}\bm{q} \in \Rbd \Delta_{|\cI|}$, and  $\Pi_{\cI}^{}\bm{q} \in \Ri \Delta_{|\cI|}$, we obtain $\Phi'(\hat{\bm{q}}; \bm{q} - \hat{\bm{q}}) =\Phi'_{\cI}(\Pi_{\cI}^{} \hat{\bm{q}}; \Pi_{\cI}^{} (\bm{q} -\hat{\bm{q}})) =-\infty$. This shows \eqref{8:e}.
\end{proof}

\begin{lemma}
\label{29:}
For $\eta>0$, $\Se_{\eta} \coloneqq \eta^{-1} \Se$ satisfies \eqref{8:e} for all $\cI \subseteq [k]$ such that $|\cI|>1$, where $\Se$ is the Shannon entropy.
\end{lemma}
\begin{proof}
Let $\cI\subseteq [k]$ such that $|\cI|>1$. Let $(\hbm{q}, \bm{q}) \in \Rbd \Delta_{\cI} \times \Ri \Delta_{\cI}$ and $\bm{q}^{\lambda} \coloneqq \hbm{q} + \lambda (\bm{q}-\hbm{q}),$ for $\lambda \in ]0,1[$. Let $\mathfrak{I} \coloneqq \{j \in \cI: \hat{q}_j \neq 0\}$ and $\mathfrak{K}\coloneqq \cI \setminus \mathfrak{I}$. We have 
\begin{align}
\Se(\hat{\bm{q}}; \bm{q} - \hat{\bm{q}}) &= \lim_{\lambda \downarrow 0} \lambda^{-1} \left[  \sum\nolimits_{\theta\in \cI} q^{\lambda}_{\theta} \log q^{\lambda}_{\theta}-  \sum\nolimits_{\theta'\in \mathfrak{I}} \hat{q}_{\theta'} \log \hat{q}_{\theta'} \right], \nonumber \\
& = \lim_{\lambda \downarrow 0} \lambda^{-1}  \left[ \sum\nolimits_{\theta\in \mathfrak{I}} (q^{\lambda}_{\theta} \log q^{\lambda}_{\theta} -\hat{q}_{\theta} \log \hat{q}_{\theta} ) +  \sum\nolimits_{\theta'\in \mathfrak{K}} q^{\lambda}_{\theta'} \log q^{\lambda}_{\theta'} \right]. \label{sums:e}
\end{align}
Observe that the limit of either summation term inside the bracket in \eqref{sums:e} is equal to zero. Thus, using l'Hopital's rule we get 
\begin{align}  \Se(\hat{\bm{q}}; \bm{q} - \hat{\bm{q}})  &= \lim_{\lambda \downarrow 0}\left[\sum\nolimits_{\theta\in \mathfrak{I}} [(q_{\theta} - \hat{q}_{\theta}) \log q^{\lambda}_{\theta} +(q_{\theta} - \hat{q}_{\theta})]  + \sum\nolimits_{\theta'\in \mathfrak{K}} [ q_{\theta'} \log q^{\lambda}_{\theta'} +q_{\theta'} ]    \right], \nonumber \\
& = \sum\nolimits_{\theta\in \mathfrak{I}} (q_{\theta} - \hat{q}_{\theta}) \log \hat{q}_{\theta}  + \sum\nolimits_{\theta'\in \mathfrak{K}} q_{\theta'} \left[  \lim_{\lambda \downarrow 0}  \log q^{\lambda}_{\theta'} \right], \label{25:e}
 \end{align} 
where in \eqref{25:e} we used the fact that $\sum_{\theta \in \mathfrak{I}} (q_{\theta} - \hat{q}_{\theta}) + \sum_{\theta' \in \mathfrak{K}} q_{\theta'} = 0$. Since for all $\theta' \in \mathfrak{K}$, $\lim_{\lambda \downarrow 0} q^{\lambda}_{\theta'} =0$, the right hand side of \eqref{25:} is equal to $-\infty$. Therefore $\Se$ satisfies \eqref{8:e}. Since $\Se_{\eta} =\eta^{-1} \Se$, it is clear that $\Se_{\eta}$ also satisfies \eqref{8:e}.

\end{proof}

\begin{lemma}
\label{infdiv}
Let $\Phi: \mathbb{R}^k\rightarrow \mathbb{R}\cup \{+\infty\}$ be an entropy satisfying \eqref{8:e} for all $\cI \subseteq [k]$ such that $|\cI|>1$. Then for all such $\cI$, it holds that
\begin{align*}
\forall \bm{q}\in \Delta_{\cI},  \forall \bm{\mu} \in \Delta_k \setminus \Delta_{\cI}, \;D_{\Phi}(\bm{\mu},\bm{q})=+\infty.
\end{align*}
\end{lemma}

\begin{proof}
Let $\bm{\mu} \in \Delta_k \setminus \Delta_{\cI}$ and $\mathfrak{I} \coloneqq \{\theta \in [k]:  \mu_{\theta}\neq 0 \} \cup \cI$. In this case, we have $ \bm{q} \in \Rbd \Delta_{\mathfrak{I}}$ and $\bm{q}+2^{-1}(\bm{\mu} -\bm{q}) \in \Ri \Delta_{\mathfrak{I}}$. Thus, since $\Phi$ satisfies \eqref{8:e} and $\Phi'(\bm{q}; \cdot)$ is 1-homogeneous \citep[Prop. D.1.1.2]{\Hi}, it follows that $2^{-1}\Phi'( \bm{q}; \bm{\mu}- \bm{q}) =  \Phi'( \bm{q}; 2^{-1} (\bm{\mu}- \bm{q}))= -\infty$. Hence $D_{\Phi}(\bm{\mu}, \bm{q})=+\infty$.
\end{proof}

\begin{lemma}
\label{13:}
Let $\Phi \colon \mathbb{R}^k \rightarrow \mathbb{R} \cup \{ +\infty\}$ be an entropy satisfying \eqref{8:e} for all $\cI \subseteq [k]$ such that $|\cI|>1$. If $\Phi$ satisfies \eqref{8:e}, then $\partial \tilde{\Phi}(\tbm{q}) = \varnothing, \forall \tbm{q} \in \Bd \tilde{\Delta}_k$. Furthermore, $\forall \cI \subseteq [k]$ such that $|\cI|>1$, \begin{align*}\forall \bm{d}\in \mathbb{R}^k, \forall \bm{q}\in \Ri \Delta_{\cI}, \;\M_{\Phi}(\bm{d}, \bm{q}) = \M_{\Phi_{\cI}}(\Pi_{\cI}^{} \bm{d}, \Pi_{\cI}^{}\bm{q}).\end{align*}
\end{lemma}	
\begin{proof}
Let $\bm{\mu} \in \Rbd \Delta_k$. Since $\Phi$ satisfies \eqref{8:e}, it follows that $\forall \bm{q} \in \Ri \Delta_k, \tilde{\Phi}(\tbm{\mu}; \tbm{q} - \tbm{\mu})=\Phi'(\bm{\mu}; \bm{q}-\bm{\mu}) = -\infty$. Therefore, $\partial \tilde{\Phi}'(\tbm{\mu})=\varnothing$ \citep[Thm. 23.4]{Rockafellar1997a}.

Let $\bm{d} \in \mathbb{R}^n$, $\cI \subseteq [k]$, with $|\cI|>1$, and $\bm{q} \in \Ri \Delta_{\cI}$. Then 
\begin{align}
\M_{\Phi_{\cI}}(\Pi_{\cI}^{} \bm{d}, \Pi_{\cI}^{}\bm{q}) &=  \inf_{\bm{\pi} \in \Delta_{|\cI|} } \Inner{\bm{\pi}}{\Pi_{\cI}^{}\bm{d}} + D_{\Phi_{\cI}}(\bm{\pi}, \Pi_{\cI}^{}\bm{q}), \nonumber \\
&=\inf_{\bm{\mu} \in \Delta_{\cI}} \Inner{\bm{\mu}}{\bm{d}} + D_{\Phi}(\bm{\mu}, \bm{q}), \nonumber  \\
 &\leq \inf_{\bm{\mu} \in \Delta_k} \Inner{\bm{\mu}}{\bm{d}} + D_{\Phi}(\bm{\mu}, \bm{q}), \label{38:e}  \\
& = \M_{\Phi} (\bm{d}, \bm{q}). \nonumber
\end{align}

To complete the proof, we need to show that \eqref{38:e} holds with equality. For this, it suffices to prove that $\forall \bm{\mu} \in \Delta_k \setminus \Delta_{\cI}$, $D_{\Phi}(\bm{\mu},\bm{q})=+\infty$. This follows from Corollary \ref{infdiv}.
\end{proof}

	\begin{lemma}
		\label{15:}
		Let $\Phi \colon \mathbb{R}^k \rightarrow \mathbb{R} \cup \{+\infty\}$ be an entropy satisfying \eqref{8:e} for all $\cI \subseteq [k]$ such that $|\cI|>1$. Let $x\in [n], \bm{d} \in \mathbb{R}^k$, and $\bm{q} \in \Delta_k$. The infimum in \begin{align} \M_{\Phi}(\bm{d}, \bm{q}) = \inf_{\bm{\mu}\in \Delta_k} \inner{\bm{\mu}}{\bm{d}} +D_{\Phi}(\bm{\mu}, \bm{q}) \label{infimum}\end{align} is attained at some $\bm{q}_* \in \Delta_k$. Furthermore, if $\bm{q}\in \Ri \Delta_k$ and $\bm{q}_*$ is the infimum of \eqref{infimum} then for any $\bm{s}_{\bm{q}}^* \in \argmax \{\inner{\bm{s}}{\tbm{q}_* - \tbm{q}}: \bm{s} \in \partial \tilde{\Phi} (\tbm{q}) \}$, we have 
		\begin{gather}	
	 \noindent  \tbm{q}_*  \in  \partial \tilde{\Phi}^* (\bm{s}_{\bm{q}}^* - J_k^{\mathsf{T}} \bm{d}),  \label{9:e} \\
		\M_{\Phi}(\bm{d}, \bm{q})  = d_k + \tilde{\Phi}^*(\bm{s}_{\bm{q}}^*) - \tilde{\Phi}^*(\bm{s}_{\bm{q}}^* - J^{\mathsf{T}}_k \bm{d}). \label{10:e}
		\end{gather}
	\end{lemma}
\begin{proof}		
Let $\bm{q}\in \Ri \Delta_k$. Since $\tbm{q}\in \Int \Dom \tilde{\Phi} = \Int \tilde{\Delta}_k$, the function $\tbm{\mu} \mapsto- \tilde{\Phi}'(\tbm{q}; \tbm{\mu} - \tbm{q})$ is lower semicontinuous \citep[Cor. 24.5.1]{Rockafellar1997a}. Given that $\tbm{\mu} \mapsto \inner{\amalg_k(\tbm{\mu})}{\bm{d}} + \tilde{\Phi}(\tbm{\mu}) - \tilde{\Phi}(\tbm{q})$ is a closed convex function, it is also lower semi-continuous. Therefore, the function \begin{align*}\tbm{\mu} \mapsto \inner{\amalg_k(\tbm{\mu})}{\bm{d}} + \tilde{\Phi}(\tbm{\mu}) - \tilde{\Phi}(\tbm{q})  - \tilde{\Phi}'(\tbm{q}; \tbm{\mu} - \tbm{q})\end{align*} is lower semicontinuous, and thus attains its minimum on the compact set $\tilde{\Delta}_k$ at some point $\tbm{q}_*$. Using the fact that $D_{\Phi}(\bm{\mu}, \bm{q}) = D_{\tilde{\Phi}}(\tbm{\mu}, \tbm{q})$, we get that \begin{align} \label{MDA} \bm{q}_* \coloneqq \amalg_k(\tbm{q}_*) = \argmin_{\bm{\mu} \in \Delta_k}\inner{\bm{\mu}}{\bm{d}} + D_{\Phi}(\bm{\mu}, \bm{q}). \end{align} 
If $\bm{q}\in \Rbd \Delta_k$, then either $\bm{q}$ is a vertex of $\Delta_k$ or there exists $\cI \subsetneq [k]$ such that $\bm{q}\in \Ri \Delta_{\cI}$. In the former case, it follows from \eqref{8:e} that $D_{\Phi}(\bm{\mu}, \bm{q}) =+\infty$ for all $\bm{\mu} \in \Delta_k  \setminus \{\bm{q}\}$, and thus the infimum of \eqref{infimum} is trivially attained at $\bm{\mu}=\bm{q}$. Now consider the alternative --- $\bm{q} \in \Delta_{\cI}$ with $|\cI|>1$. Using Corollary \ref{infdiv}, we have $D_{\Phi}(\bm{\mu}, \bm{q})=+\infty$ for all $\bm{\mu} \in \Delta_k \setminus \Delta_{\cI}$. Therefore, \begin{align}\M_{\Phi}(\bm{d}, \bm{q}) &= \inf_{\bm{\mu}\in \Delta_{\cI}} \inner{\bm{\mu}}{\bm{d}} +D_{\Phi}(\bm{\mu}, \bm{q}), \nonumber\\ 
& = \inf_{\hbm{\mu}\in \Delta_{|\cI|}} \inner{\hbm{\mu}}{\Pi_{\cI} \bm{d}} +D_{\Phi_{\cI}}(\hbm{\mu}, \Pi_{\cI}\bm{q}), \label{lastone} \end{align}
where $\Phi_{\cI}\coloneqq \Phi \circ \Pi_{\cI}$. Since $\Pi_{\cI} \bm{q}\in \Ri \Delta_{|\cI|}$, we can use the same argument as the previous paragraph with $\Phi$ and $\bm{q}$ replaced by $\Phi_{\cI}$ and $\Pi_{\cI} \bm{q}$, respectively, to show that the infimum in \eqref{lastone} is attained at some $\hbm{q}_* \in \Delta_{|\cI|}$. Thus, $\bm{q}_* \coloneqq \Pi_{\cI}^{\T}\hbm{q} \in \Delta_k$ attains the infimum in \eqref{infimum}.

Now we show the second part of the lemma. Let $\bm{q}\in \Ri \Delta_k$ and $\bm{q}_*$ be the infimum of \eqref{infimum}. Since $\tilde{\Phi}$ is convex and $\tbm{q} =\Pi_k(\bm{q}) \in \Int \tilde{\Delta}_k = \Int \Dom \tilde{\Phi} $, we have $\partial \tilde{\Phi} (\tbm{q}) \neq \varnothing$ \citep[Thm. 23.4]{Rockafellar1997a}. This means that there exists $\bm{s}_{\bm{q}}^* \in \partial \tilde{\Phi}(\tbm{q})$ such that $\inner{\bm{s}_{\bm{q}}^*}{ \tbm{q}_* - \tbm{q}} = \tilde{\Phi}'(\tbm{q}; \tbm{q}_*- \tbm{q})$ \cite[p.166]{Hiriart-Urruty}. We will now show that $\bm{s}_{\bm{q}}^* - J^{\mathsf{T}}_k \bm{d}\in \partial \tilde{\Phi}(\tbm{q}_*)$, which will imply that $\tbm{q}_* \in \partial \tilde{\Phi}^*(\bm{s}_{\bm{q}}^* - J_k^{\mathsf{T}}\bm{d})$ (\ibid, Cor. D.1.4.4).  Let $\bm{q}_* = \argmin_{\bm{\mu}\in \Delta_k}\inner{\bm{\mu}}{\bm{d}} +D_{\Phi}(\bm{\mu}, \bm{q})$. Thus, for all $ \bm{\mu} \in \Delta_k$,
				\begin{eqnarray*}
				&\inner{\bm{\mu}}{ \bm{d}} + \tilde{\Phi}(\tbm{\mu})- \tilde{\Phi}(\tbm{q})  - \tilde{\Phi}'(\tbm{q}; \tbm{\mu} - \tbm{q})  \geq \inner{\bm{q}_*}{\bm{d}} + \tilde{\Phi}(\tbm{q}_*)- \tilde{\Phi}(\tbm{q})  -  \inner{\bm{s}_{\bm{q}}^*}{\tbm{q}_*- \tbm{q}}, \\
				\implies& \tilde{\Phi}(\tbm{\mu}) \geq \tilde{\Phi}(\tbm{q}_*) - \inner{\tbm{\mu} - \tbm{q}_*}{J^{\mathsf{T}}_k \bm{d}} + \inner{\bm{s}_{\bm{q}}^*}{\tbm{q}- \tbm{q}_*} +\Phi'(\tbm{q}; \tbm{\mu}- \tbm{q}), \\
				\implies &   \tilde{\Phi}(\tbm{\mu}) \geq \tilde{\Phi}(\tbm{q}_*) - \inner{\tbm{\mu} - \tbm{q}_*}{J^{\mathsf{T}}_k \bm{d}} + \inner{\bm{s}_{\bm{q}}^*}{\tbm{q}- \tbm{q}_*} + \inner{\bm{s}_{\bm{q}}^*}{\tbm{\mu}- \tbm{q}}, \\
				 \implies & \tilde{\Phi}(\tbm{\mu}) \geq \tilde{\Phi}(\tbm{q}_*) + \inner{\tbm{\mu} - \tbm{q}_*}{\bm{s}_{\bm{q}}^* -J^{\mathsf{T}}_k \bm{d}},
				\end{eqnarray*}
				where in the second line we used the fact that $\forall \bm{q}\in \Delta_k, \inner{\bm{q}}{\bm{d}}  =\inner{\tbm{q}}{J_k^{\mathsf{T}} \bm{d}} + d_k$, and in third line we used the fact that $\forall \bm{s}\in \partial \tilde{\Phi}(\tbm{q}), \; \inner{\bm{s}}{\tbm{\mu} - \tbm{q}} \leq \tilde{\Phi}'(\tbm{q}; \tbm{\mu} - \tbm{q})$ (\ibid). This shows that $\bm{s}_{\bm{q}}^* - J^{\mathsf{T}}_k \bm{d}\in \partial \tilde{\Phi}(\tbm{q}_*)$.
				
				Substituting $\tilde{\Phi}'(\tbm{q}; \tbm{q}_* - \tbm{q})$ by $\inner{\bm{s}_{\bm{q}}^*}{\bm{q}_* - \bm{q}}$ in the expression for $\M_{\Phi}(\bm{d}, \bm{q})$, we get 
				\begin{align}
				\M_{\Phi}(\bm{d}, \bm{q}) &=d_k +\inner{\tbm{q}_*}{J_k^{\mathsf{T}} \bm{d}} + \tilde{\Phi}(\tbm{q}_*)- \tilde{\Phi}(\tbm{q})  -  \inner{\bm{s}_{\bm{q}}^*}{\tbm{q}_*- \tbm{q}},  \nonumber  \\
				&=d_k  + \inner{\bm{s}_{\bm{q}}^*}{\tbm{q}} - \tilde{\Phi}(\tbm{q})  - [ \inner{\bm{s}_{\bm{q}}^* -  J_k^{\mathsf{T}} \bm{d}}{\tbm{q}_*} - \tilde{\Phi}(\tbm{q}_*)], \nonumber \\
				& = d_k  + \tilde{\Phi}^*(\bm{s}_{\bm{q}}^*) - \tilde{\Phi}^*(\bm{s}_{\bm{q}}^* - J_k^{\mathsf{T}} \bm{d}), \nonumber
				\end{align}
				where in the last line we used the fact that $\tilde{\Phi}$ is a closed convex function, and thus $\forall \tbm{q} \in \tilde{\Delta}_k$, $\bm{s} \in \partial \tilde{\Phi}(\tbm{q}) \implies \tilde{\Phi}^*(\bm{s}) = \inner{\bm{s}}{\tbm{q}}-  \tilde{\Phi}(\tbm{q}) $ (\ibid, Cor. E.1.4.4).

\end{proof}

\begin{lemma}
\label{30:}
Let $\bm{q} \in \Delta_k$. For any sequence $(\bm{d}_m)$ in $[0,+\infty[^k$ converging to $\bm{d} \in [0, +\infty]^k$ coordinate-wise and any entropy $\Phi \colon \mathbb{R}^k \rightarrow \mathbb{R} \cup \{+\infty\}$ satisfying \eqref{8:e} for $\cI \subseteq [k]$ such that $|\cI|>1$, 
\begin{align}\lim_{m\to \infty}\M_{\Phi}(\bm{d}_m, \bm{q}) = \M_{\Phi}(\bm{d}, \bm{q}).\end{align}
\end{lemma}
\begin{proof}[\textbf{Proof of Lemma \ref{30:}}]
Let $\bm{q} \in \Delta_k$ and $\Phi \colon \mathbb{R}^k \rightarrow \mathbb{R} \cup \{+\infty\}$ be an entropy as in the statement of the Lemma. Let $(\bm{d}_m) \subset \mathbb{R}^k$ such that $\bm{d}_m \stackrel{m\to \infty}{\to} \bm{d}\in \mathbb{R}^k$. in $[0,+\infty[^k$. Let $\cI \coloneqq \{\theta \in [k]: d_{\theta} < +\infty \}$. If $\cI = \varnothing$ then the result holds trivially since, on the one hand, $\Mix_{\Phi}(\bm{d}, \bm{q})=+\infty$ and on the other hand $\Mix_{\Phi}(\bm{d}_m, \bm{q})\geq \min_{\theta \in [k]} d_{m,\theta}  \stackrel{m\to \infty}{\to} +\infty$.

Assume now that $\cI \neq \varnothing$. Then 
\begin{align}
\Mix_{\Phi}(\bm{d}_m, \bm{q}) &= \inf_{\bm{\mu}\in \Delta_k} \inner{\bm{\mu}}{\bm{d}_m} + D_{\Phi}(\bm{\mu}, \bm{q}), \label{inf2}\\
& \leq   \inf_{\hbm{\mu}\in \Delta_{\cI}} \inner{\hbm{\mu}}{  \bm{d}} + D_{\Phi}(\hbm{\mu},  \bm{q}), \label{subset}  \\
& < +\infty, \label{ineq}
\end{align}
where the last inequality stems from the fact that $\Pi_{\cI} \bm{d}_m$ is a finite vector in $\mathbb{R}^{|\cI|}$. Therefore, \eqref{ineq} implies that the sequence $\alpha_m \coloneqq \Mix_{\Phi}(\bm{d}_m, \bm{q})$ is bounded. We will show that $(\alpha_m)$ converges in $\mathbb{R}$ and that its limit is exactly $\Mix_{\Phi}(\bm{d}, \bm{q})$. Let $(\hat{\alpha}_m)$ be any convergent subsequence of $(\alpha_m)$, and let $(\hbm{d}_m)$ be the corresponding subsequence of $(\bm{d}_m)$. Consider the infimum in \eqref{inf2} with $\bm{d}_m$ is replaced by $\hbm{d}_m$. From Lemma \ref{15:}, this infimum is attained at some $\bm{q}_m \in \Delta_k$. Since $\Delta_k$ is compact, we may assume without loss of generality that $\bm{q}_{m}$ converges to some $\bar{\bm{q}} \in \Delta_k$. Observe that $\bar{\bm{q}}$ must be $\Delta_{\cI}$; suppose that $\exists \theta_* \in \bar{\cI}$ such that $\bar{q}_{\theta_*}>0$. Then \begin{align}
\hat{\alpha}_{m} &\geq \inner{\bm{q}_{m}}{\hbm{d}_{m}}, \nonumber \\
& \geq q_{m,\theta_*} \hat{d}_{m,\theta_*}  \stackrel{m \to \infty}{\to} +\infty. \nonumber
\end{align}
This would contradict the fact that $\alpha_{m}$ is bounded, and thus $\bar{\bm{q}} \in \Delta_{\cI}$. Using this, we get 
\begin{align}
\Mix_{\Phi}(\hbm{d}_{m}, \bm{q}) &\;\;= \inner{\bm{q}_{m}}{\hbm{d}_{m}} + D_{\Phi}(\bm{q}_{m}, \bm{q}), \nonumber \\
&\;\; \geq \inner{\Pi_{\cI} \bm{q}_m}{\Pi_{\cI}\hbm{d}_m} + D_{\Phi}(\bm{q}_m, \bm{q}), \nonumber \\
& \stackrel{m\to \infty}{\to} \inner{\Pi_{\cI}\bar{\bm{q}}}{\Pi_{\cI}\bm{d}} + D_{\Phi}(\bar{\bm{q}}, \bm{q}), \nonumber \\
&\;\; = \inner{\bar{\bm{q}}}{\bm{d}}  + D_{\Phi}(\bar{\bm{q}}, \bm{q}), \label{blast} \\
&\;\; \geq \inf_{\hbm{\mu}\in \Delta_{\cI}} \inner{\hbm{\mu}}{  \bm{d}} + D_{\Phi}(\hbm{\mu},  \bm{q}). \label{llast}
\end{align}
where in \eqref{blast} we use the fact that $\bar{\bm{q}} \in \Delta_{\cI}$. Combining \eqref{llast} with \eqref{subset} shows that $\hat{\alpha}_m$ converges to $\Mix_{\Phi}(\bm{d}, \bm{q}) = \inf_{\hbm{\mu}\in \Delta_{\cI}} \inner{\hbm{\mu}}{  \bm{d}} + D_{\Phi}(\hbm{\mu},  \bm{q})$. Since $(\hat{\alpha}_m)$ was any convergent subsequence of $(\alpha_m)$ (which is bounded), the result follows.
\end{proof}

\section{Proofs of Results in the Main Body}	
	\label{proofsmain}
\subsection{Proof of Theorem \ref{4:}}
\textbf{Theorem \ref{4:}}
\emph{Any loss $\ell \colon \mathcal{A} \rightarrow [0, +\infty]^n$ such that $\Dom \ell \neq \varnothing$, has a proper support loss $\underline{\ell}$ with the same Bayes risk, $\br_{\ell}$, as $\ell$. }

\begin{proof}
We will construct a proper support loss $\sell$ of $\ell$.

Let $\bm{p}\in \Ri \Delta_n$ ($-\bm{p} \in \Int \Dom \sigma_{\sps}$). Since the support function of a non-empty set is closed and convex, we have $\sigma^{**}_{\sps} = \sigma_{\sps}$ \citep[Prop. C.2.1.2]{Hiriart-Urruty}. Pick any $\bm{v} \in \partial \sigma_{\sps} (-\bm{p})=\partial \sigma^{**}_{\sps} (-\bm{p}) \neq \varnothing$. Since $\sigma^*_{\sps} = \I_{\sps}$ \citep{Rockafellar1997a}, we can apply Proposition \ref{1:}-(iv) with $f$ replaced by $\sigma^*_{\sps}$ to obtain $\inner{-\bm{p}}{\bm{v}}= \sigma_{\sps}(-\bm{p}) + \I_{\sps}(\bm{v})$. The fact that $\inner{-\bm{p}}{\bm{v}}$ and $\sigma_{\sps}(-\bm{p})$ are both finite implies that $\I_{\sps}(\bm{v})=0$. Therefore, $\bm{v} \in {\sps}$ and $\inner{\bm{p}}{\bm{v}}= - \sigma_{\sps}(-\bm{p})= \br_{\ell}(\bm{p})$. Define $\sell(\bm{p}) \coloneqq \bm{v} \in \sps$.

Now let $\bm{p} \in \Rbd \Delta_n$ and $\bm{q} \coloneqq \bm{1}_n/n$. Since the $\br_{\sell}$ is a closed concave function and $\bm{q} \in \Int \Dom \br_{\sell}$, it follows that $\br_{\sell}(\bm{p} + m^{-1} (\bm{q} - \bm{p})) \stackrel{m\to\infty}{\to} \br_{\sell}(\bm{p})$ \citep[Prop. B.1.2.5]{Hiriart-Urruty}. Note that $\bm{q}_m \coloneqq \bm{p} + m^{-1} (\bm{q} - \bm{p}) \in \Ri \Delta_n, \forall m \in \mathbb{N}$. Now let $v_{x,m} \coloneqq \sell_x(\bm{q}_m)$, where $\sell(\bm{q}_m)$ is as constructed in the previous paragraph. If $(v_{1,m})$ is bounded [resp. unbounded], we can extract a subsequence $(v_{1, \varphi_1(m)})$ which converges [resp. diverges to $+\infty$], where $\varphi_1: \mathbb{N} \rightarrow \mathbb{N}$ is an increasing function.
By repeating this process for $(v_{2, \varphi_1(m)})$ and so on, we can construct an increasing function $\varphi \coloneqq \varphi_n \circ \dots \circ \varphi_1: \mathbb{N} \rightarrow \mathbb{N}$, such that $\bm{v}_m \coloneqq [v_{x,\varphi(m)}]^{\mathsf{T}}_{x\in [n]}$ has a well defined (coordinate-wise) limit in $[0,+\infty]^n$. Define $\sell(\bm{p}) \coloneqq \lim_{m\to \infty} \bm{v}_m$. By continuity of the inner product, we have 
\begin{align*}
 \inner{  \bm{p}   }{   \sell(\bm{p})   }= \lim_{m \to \infty} \inner{  \bm{q}_{\varphi(m)}   }{  \bm{v}_m  }  &= \lim_{m\to \infty} \inner{\bm{q}_{\varphi(m)}}{\sell(\bm{q}_{\varphi(m)})}, \\ &= \lim_{m\to \infty} \br_{\ell}(\bm{q}_{\varphi(m)}) = \br_{\ell}(\bm{p}).
\end{align*}   
By construction, $\forall m \in \mathbb{N}, \bm{p}_m \coloneqq \bm{q}_{\varphi(m)} \in \Ri \Delta_n$ and $\sell(\bm{p}_m) = \bm{v}_m \stackrel{m \to \infty}{\to} \sell(\bm{p})$. Therefore, $\sell$ is support loss of $\ell$.

It remains to show that it is proper; that is $\forall \bm{p} \in \Delta_n, \forall \bm{q}\in \Delta_n$, $\inner{\bm{p}}{\underline{\ell}(\bm{p})} \leq \inner{\bm{p}}{\underline{\ell}(\bm{q})}$. Let $\bm{q}\in \Ri \Delta_n$. We just showed that $\forall \bm{p} \in \Delta_n, \inner{\bm{p}}{\underline{\ell}(\bm{p})} = \br_{\ell}(\bm{p})$ and that $\underline{\ell}(\bm{q}) \in \sps$. Using the fact that $\br_{\ell}(\bm{p}) =\inf_{\bm{z} \in \sps}\inner{\bm{p}}{\bm{z}}$, we obtain $\inner{\bm{p}}{\underline{\ell}(\bm{p})} \leq \inner{\bm{p}}{\underline{\ell}(\bm{q})}$. 

Now let $\bm{q}\in \Rbd \Delta_k$. Since $\sell$ is a support loss, we know that there exists a sequence $(\bm{q}_m) \subset \Ri \Delta_n$ such that $\sell(\bm{q}_m)  \stackrel{m \to \infty}{\to} \sell(\bm{q})$. But as we established in the previous paragraph, $\inner{\bm{p}}{\underline{\ell}(\bm{p})} \leq \inner{\bm{p}}{\underline{\ell}(\bm{q}_m)}$. By passing to the limit $m\to \infty$, we obtain $\inner{\bm{p}}{\underline{\ell}(\bm{p})}  \leq \inner{\bm{p}}{\underline{\ell}(\bm{q})}$. Therefore $\sell$ is a proper loss with Bayes risk $\br_{\ell}$.
\end{proof}
		
\subsection{Proofs of Theorem \ref{5:} and Proposition \ref{14:}}
\label{a5:}

For a set $\mathcal{C}$, we denote $\co \mathcal{C}$ and $\cco \mathcal{C}$ its \emph{convex hull} and \emph{closed convex hull}, respectively. 
\begin{definition}[\citep{Hiriart-Urruty}]
\label{32:}
Let $\mathcal{C}$ be non-empty convex set in $\mathbb{R}^n$. We say that $\bm{u} \in \mathcal{C}$ is an extreme point of $\mathcal{C}$ if there are no two different points $\bm{u}_1$ and $\bm{u_2}$ in $\mathcal{C}$ and $\lambda \in ]0,1[$ such that $\bm{u} = \lambda \bm{u}_1 + (1-\lambda)\bm{u}_2$.
\end{definition}

We denote the set of extreme points of a set $\mathcal{C}$ by $\ext \mathcal{C}$.
\begin{lemma}
\label{33:}
Let $\ell \colon \mathcal{A} \rightarrow [0,+\infty]^n$ be a closed loss. Then $\ext \cco \sps \subseteq \mathcal{S}_{\ell}$.
\end{lemma}
\begin{proof}
Since $ \co \sps \subseteq \mathbb{R}^n$ is connected, $\co \sps =  \{\bm{v}+\sum_{k=1}^n \alpha_k \ell(\bm{a}_k) $: $(\bm{a}_{k \in [n]}, \bm{\alpha}, \bm{v}) \in \mathcal{A}^n \times\Delta_n  \times [0,+\infty[^n$\} \citep[Prop. A.1.3.7]{\Hi}. 

We claim that $\cco \sps = \co \sps$. Let $(\bm{z}_m) \coloneqq (\bm{v}_m + \sum_{k=1}^n \alpha_{m,k} \ell(\bm{a}_{m,k}) )$ be a convergent sequence in $[0,+\infty[^n$, where $(\bm{\alpha}_m)$, $([\bm{a}_{m,k}]_{k\in[n]})$ and $(\bm{v}_m)$ are sequences in $\Delta_n$, $\mathcal{A}^n$, and $[0,+\infty[^n$, respectively. Since $\Delta_n$ is compact, we may assume, by extracting a subsequence if necessary, that $\bm{\alpha}_m \stackrel{m \to \infty}{\to} \bm{\alpha}^*\in \Delta_n$. Let $\mathcal{K} \coloneqq \{ k \in [n]:\alpha^*_k \neq 0 \}$. Since $\bm{z}_m$ converges, $([[\ell(\bm{a}_{m,k})]_{k\in \mathcal{K}}, \bm{v}_m]) $ is a bounded sequence in $[0,+\infty[^{n|\mathcal{K}|+n}$. Since $\ell$ is closed, we may assume, by extraction a subsequence if necessary, that $\forall k \in \mathcal{K}$, $\ell(\bm{a}_{m,k}) \stackrel{m\to\infty}{\to} \ell(\bm{a}_k^*)$ and $\bm{v}_m\stackrel{m\to\infty}{\to} \bm{v}^*$, where $[\bm{a}_k^*]_{k\in\mathcal{K}} \in \mathcal{A}^{|\mathcal{K}|}$ and $\bm{v}^*\in[0,+\infty[^n$. Consequently, \begin{align*}\bm{v}^*+\sum_{k=1}^n \alpha^*_k \ell(\bm{a}^*_k)   &= \lim_{m\to\infty}\left[  \bm{v}_{m,k} + \sum_{k\in \mathcal{K}} \alpha_{m,k} \ell(\bm{a}_{m,k})\right],  \\&  \leq  \lim_{m \to \infty} \bm{z}_m   ,\end{align*} 
where the last inequality is coordinate-wise. Therefore, there exists $\bm{v}'\in [0,+\infty[^n$ such that $\lim_{m \to \infty} \bm{z}_m=\bm{v}'+\bm{v}^*+\sum_{k=1}^n \alpha^*_k \ell(\bm{a}^*_k) \in \co\sps$. This shows that $\cco \sps \subset \co \sps$, and thus $\cco \sps = \co \sps$ which proves our first claim.

 By definition of an extreme point, $\ext \cco \sps \subseteq \cco \sps$. Let $\bm{e} \in \ext \cco \sps$ and $(\bm{a}_{k \in [n]}, \bm{\alpha}, \bm{v}) \in \mathcal{A}^n \times\Delta_n  \times [0,+\infty[^n$ such that $\bm{e}=\sum_{k=1}^n \alpha_k \ell(\bm{a}_k) + \bm{v}$. If there exists $i, j \in [n]$ such that $\alpha_i \alpha_j \neq 0$ or $\alpha_i v_j \neq 0$ then $\bm{e}$ would violate the definition of an extreme point. Therefore, the only possible extreme points are of the form $\{\ell(\bm{a}): \bm{a}\in \Dom \ell) \} =\mathcal{S}_{\ell}$.
\end{proof}

\textbf{Theorem \ref{5:}}
\emph{
Let $\ell \colon \mathcal{A} \rightarrow [0, +\infty]^n$ be a loss and $\underline{\ell}$ be a proper support loss of $\ell$. If the Bayes risk $\br_{\ell}$ is differentiable on $]0, +\infty[^n$, then $\sell$ is uniquely defined on $\Ri \Delta_n$ and
\begin{align*}
\begin{array}{lll}
\forall \bm{p}\in \Dom \sell,& \exists \bm{a}_*\in \Dom \ell,&  \ell(\bm{a}_*)= \sell(\bm{p}), \\
\forall \bm{a}\in \Dom \ell, &\exists (\bm{p}_m) \subset \Ri \Delta_n, & \sell(\bm{p}_m) \stackrel{m\to \infty}{\to} \ell(\bm{a})\; \mbox{ coordinate-wise.}
\end{array}
\end{align*}
}
\begin{proof}
Let $\bm{p}\in \Ri \Delta_n$ and suppose that $\br_{\ell}$ is differentiable at $\bm{p}$. In this case, $\sigma_{\sps}$ is differentiable at $-\bm{p}$, which implies \citep[Cor. D.2.1.4]{Hiriart-Urruty} 
\begin{align}
\mathcal{F}(\bm{p}) \coloneqq \partial \sigma_{\sps}(-\bm{p})=\{\nabla \sigma_{\sps}(-\bm{p}) \}.
\label{singleton}
\end{align}
On the other hand, the fact that $\sigma_{\sps} = \sigma_{\cco \sps}$ \citep[Prop. C.2.2.1]{Hiriart-Urruty}, implies $\mathcal{F}(\bm{p}) = \partial \sigma_{\sps}(-\bm{p}) = \partial \sigma_{\cco \sps}(-\bm{p})$. The latter being an \emph{exposed face} of $\cco \sps$ implies that every extreme point of $\mathcal{F}(\bm{p})$ is also an extreme point of $\cco \sps$ \citep[Prop. A.2.3.7, Prop. A.2.4.3]{Hiriart-Urruty}. Therefore, from \eqref{singleton}, $\sell(\bm{p}) = \nabla \sigma_{\sps}(-\bm{p})$ is the only extreme point of $\mathcal{F}(\bm{p}) \subset \cco \sps$. From Lemma \ref{33:}, there exists $\bm{a}_*\in \mathcal{A}$ such that $\ell(\bm{a_*}) = \sell(\bm{p})$. In this paragraph, we showed the following 
\begin{align}
\label{show:e}
\forall \bm{p}\in \Ri \Delta_n, \exists \bm{a}_* \in \Dom \ell, \; \ell(\bm{a}_*) = \sell(\bm{p}).
\end{align}

For the rest of this proof we will assume that $\br_{\ell}$ is differentiable on $]0,+\infty[^n$. Let $\bm{p} \in  \Rbd \Delta_n \cap \Dom \sell$. Since $\sell$ is a support loss, there exists ($\bm{p}_m$) in $\Ri \Delta_n$ such that $(\sell(\bm{p}_m))_m$ converges to $\sell(\bm{p})$. From \eqref{show:e} it holds that $\forall \bm{p}_m \in \Ri \Delta_n, \exists \bm{a}_m \in \mathcal{A}, \ell(\bm{a}_m)= \sell(\bm{p}_m) $. Since $(\ell(\bm{a}_m))_m$ converges and $\ell$ is closed, there exists $\bm{a}_*\in \mathcal{A}$ such that $\ell(\bm{a}_*) = \lim_{m\to \infty} \ell(\bm{a}_m) = \sell(\bm{p})$. 

Now let $\bm{a} \in \Dom \ell$ and $f(\bm{p},x)  \coloneqq \sell_x(\bm{p}) - \ell_x(\bm{a})$. Since $\ell(\bm{a}) \in \sps$ and $\sell$ is proper, we have for all $ \bm{p}\in \Ri \Delta_n,\mathbb{E}_{x\sim \bm{p}} [f(\bm{p}, x)]\leq 0$ and $-\infty<f(\bm{p},x), \forall x \in [n]$. Therefore, Lemma \ref{23:} implies that for all $m \in \mathbb{N}\setminus \{0\}$ there exists $\bm{p}_m \in \Ri \Delta_n$, such that $\forall x \in [n], \sell_x(\bm{p}_m) \leq  \ell_x(\bm{a}) + 1/m$. On one hand, since $(\sell(\bm{p}_m))$ is bounded (from the previous inequality), we may assume by extracting a subsequence if necessary, that $(\sell(\bm{p}_m))_m$ converges. On the other hand, since $\bm{p}_m \in \Ri \Delta_n$, \eqref{show:e} implies that there exists $\bm{a}_m \in \Dom \ell$ such that $\sell(\bm{p}_m) = \ell(\bm{a}_m)$. Since $\ell$ is closed and $(\ell(\bm{a}_m))_m$ converges, there exists $\bm{a}_* \in \mathcal{A}$, such that $\ell(\bm{a}_*) = \lim_{m\to \infty} \ell(\bm{a}_m) =\lim_{m\to \infty}  \sell(\bm{p}_m) \leq \ell(\bm{a})$. But since $\ell$ is admissible, the latter component-wise inequality implies that $\ell(\bm{a}_*) = \ell(\bm{a}) = \lim_{m\to \infty} \sell(\bm{p})$.
\end{proof}

\begin{lemma}
\label{notdiff}
Let $\ell \colon \mathcal{A} \rightarrow [0, +\infty]^n$ be a loss satisfying Assumption \ref{B:}. If $\br_{\ell}$ is not differentiable at $\bm{p}$ then there exist $\bm{a}_0, \bm{a}_1 \in \Dom \ell$, such that $\ell(\bm{a}_0)\neq \ell(\bm{a}_1)$ and $\br_{\ell}(\bm{p})=\inner{\bm{p}}{\ell(\bm{a}_0)} =\inner{\bm{p}}{\ell(\bm{a}_1)}$. 
\end{lemma}
\begin{proof}
Suppose $\br_{\ell}$ is not differentiable at $\bm{p} \in \Ri \Delta_n$. Then from the definition of the Bayes risk, $\sigma_{\sps}$ is not differentiable at $-\bm{p}$. This implies that $\mathcal{F}(\bm{p}) \coloneqq \partial \sigma_{\sps}(-\bm{p})$ has more than one element \cite[Cor. D.2.1.4]{\Hi}. Since $\sigma_{\sps} = \sigma_{\cco \sps}$ (\ibid. Prop. C.2.2.1), $\mathcal{F}(\bm{p}) = \partial \sigma_{\cco \sps}(-\bm{p})$ is a subset of $\cco \sps$ and every extreme point of $\mathcal{F}(\bm{p})$ is also an extreme point of $\cco \sps$ (\ibid, Prop. A.2.3.7). Thus, from Lemma \ref{33:}, we have $\ext \mathcal{F}(\bm{p}) \subset \mathcal{S}_{\ell}$. On the other hand, since $-\bm{p} \in \Int \Dom \sigma_{\sps}$, $\mathcal{F}(\bm{p})$ is a compact, convex set \citep[Thm. 23.4]{Rockafellar1997a}, and thus $\mathcal{F}(\bm{p}) = \co (\ext \mathcal{F}(\bm{p}))$ \citep[Thm. A.2.3.4]{Hiriart-Urruty}. Hence, the fact that $\mathcal{F}(\bm{p})$ has more than one element implies there exists $\bm{a}_0, \bm{a}_1 \in \mathcal{A}$ such that $\ell(\bm{a}_0), \ell(\bm{a}_1) \in \ext  \mathcal{F}(\bm{p}) \subseteq \mathcal{F}(\bm{p})$ and $\ell(\bm{a}_0)\neq \ell(\bm{a}_1)$. Since $\mathcal{F}(\bm{p}) = \partial \sigma_{\sps} (-\bm{p})$, Proposition \ref{1:}-(iv) and the fact that $\sigma_{\sps}^{*}= \I_{\sps}$ imply that $\br_{\ell}(\bm{p}) = \inner{\bm{p}}{\ell(\bm{a}_0)}  = \inner{\bm{p}}{\ell(\bm{a}_1)}$.
\end{proof}

\textbf{Proposition \ref{14:}}
\emph{
Let $\Phi\colon \mathbb{R}^k \rightarrow \mathbb{R} \cup \{+\infty\}$ be an entropy and $\ell \colon \mathcal{A}  \rightarrow [0, +\infty]^n$. If $\ell$ is $\Phi$-mixable, then the Bayes risk satisfies $\br_{\ell} \in C^1(]0,+\infty[^n)$. If, additionally, $\br_{\ell}$ is twice differentiable on $]0,+\infty[^n$, then $\Phi$ must be strictly convex on $\Delta_k$.}

\begin{proof}
Let $\cI= \{1,2\}$. Since $\ell$ is $\Phi$-mixable, it must be $\Phi_{\cI}$-mixable, where $\Phi_{\cI}\coloneqq \Phi_{\cI} \circ \Pi_{\cI}^{\mathsf{T}}: \mathbb{R}^2 \rightarrow \mathbb{R} \cup\{+\infty \}$ (Proposition \ref{12:}). Let $\Psi \coloneqq \Phi_{\cI}$.

For $w \in ]0,+\infty[$ and $z \in \Int \Dom \tilde{\Psi}^* = \mathbb{R}$ (see \cref{s2.2}), we define $(\tilde{\Psi}^*)'_{\infty}(w) \coloneqq \lim_{t\to +\infty} [\tilde{\Psi}^*(z + t w) - \tilde{\Psi}^*(z)]/t$. The value of $(\tilde{\Psi}^*)'_{\infty}(w)$ does not depend on the choice of $z$, and it holds that $(\tilde{\Psi}^*)'_{\infty}(w) = \sigma_{\Dom \tilde{\Psi}}(w) $ and $(\tilde{\Psi}^*)'_{\infty}(-w) = \sigma_{\Dom \tilde{\Psi}}(-w)$ \citep[Prop. C.1.2.2]{\Hi}. 
In our case, we have $\Dom \tilde{\Psi}=[0,1]$ (by definition of $\tilde{\Psi}$), which implies that $ \sigma_{\Dom \tilde{\Psi}}(1) = 1$ and $ \sigma_{\Dom \tilde{\Psi}}(-1) = 0$. Therefore,  $(\tilde{\Psi}^*)'_{\infty}(1) + (\tilde{\Psi}^*)'_{\infty}(-1) =1$. As a result $\tilde{\Psi}^*$ cannot be affine. For all $ \delta>0$, let $g_{\delta}\colon \mathbb{R} \times\{-1, 0, +1\}  \rightarrow \mathbb{R}$ be defined by \begin{align*}g_{\delta}(s, u) \coloneqq [\tilde{\Psi}^*(s + \delta (u+1)/2 ) - \tilde{\Psi}^*(s + \delta (u-1)/2)]/\delta.\end{align*} Since $\tilde{\Psi}^*$ is convex it must have non-decreasing slopes (\ibid, p.13). Combining this with the fact that $\tilde{\Psi}^*$ is not affine implies that \begin{align} \label{39:e}\exists s^*_{\delta} \in \mathbb{R},\; g_{\delta}(s^*_{\delta}, -1) <  g_{\delta}(s^*_{\delta}, +1).\end{align} 
The fact that $\tilde{\Psi}^*$ has non-decreasing slopes also implies that
\begin{align*} 
g_{\delta}(s^*_{\delta}, +1) = [\tilde{\Psi}^*(s^*_{\delta} + \delta) - \tilde{\Psi}^*(s^*_{\delta})]/\delta \leq  \lim_{t\to \infty} [\tilde{\Psi}^*(s^*_{\delta} + t) - \tilde{\Psi}^*(s^*_{\delta})]/t = (\tilde{\Psi}^*)'_{\infty}(1) =1. 
\end{align*}
Similarly, we have $0= -(\tilde{\Psi}^*)'_{\infty}(-1) \leq  g_{\delta}(s^*_{\delta}, -1)$. Let $\tilde{\mu} \in \partial \tilde{\Psi}^*(s^*_{\delta})$. Since $\tilde{\Psi}$ is a closed convex function the following equivalence holds $\tilde{\mu} \in \partial \tilde{\Psi}^*(s^*_{\delta}) \iff s^*_{\delta} \in \partial \tilde{\Psi} (\tilde{\mu})$ (\ibid, Cor. D.1.4.4). Thus, if $\tilde{\mu} \in \{0,1\} = \Bd \tilde{\Delta}_2$, then $\partial \tilde{\Psi} (\tilde{\mu}) \neq \varnothing$, which is not possible since $\ell$ is $\Psi$-mixable (Lemma \ref{13:}).

[\textbf{We show $\br_{\ell}\in C^1(]0,+\infty[^n)$}] We will now show that $\br_{\ell}$ is continuously differentiable on $]0,+\infty[^n$. Since $\br_{\ell}$ is 1-homogeneous, it suffices to check the differentiability on $\Ri \Delta_n$.  Suppose $\br_{\ell}$ is not differentiable at $\bm{p} \in \Ri \Delta_n$. From Lemma \ref{notdiff}, there exists $\bm{a}_0, \bm{a}_1 \in \mathcal{A}$ such that $\ell(\bm{a}_0), \ell(\bm{a}_1) \in \partial \sigma_{\sps}(-\bm{p})$ and $\ell(\bm{a}_0)\neq \ell(\bm{a}_1)$. Let $A \coloneqq [\bm{a}_0, \bm{a}_1] \in \mathbb{R}^{n \times 2}$,
$\delta \coloneqq \min \{ |\ell_x(\bm{a}_0) - \ell_x(\bm{a}_1)|: x\in [n], |\ell_x(\bm{a}_0) - \ell_x(\bm{a}_1)| >0 \}$, and $s^*_{\delta}\in \mathbb{R}$ as in \eqref{39:e}. We denote $g^- \coloneqq g_{\delta}(s^*_{\delta}, -1)$ and $g^+ \coloneqq g_{\delta}(s^*_{\delta}, +1) \in ]0,1]$. 
Let $\tilde{\mu} \in \partial \tilde{\Psi}^*(s^*_{\delta}) \in \Int \tilde{\Delta}_2$ and $\bm{\mu} = \amalg_2(\tilde{\mu})\in \Ri \Delta_2$. From the fact that $\ell$ is $\Psi$-mixable, $J_2^{\mathsf{T}} \ell_x(A)=\ell_x(\bm{a}_0) - \ell_x(\bm{a}_1)$, and \eqref{6:e}, there must exist $\bm{a}_* \in \mathcal{A}$ such that for all $x\in [n]$,
 \begin{align}
\ell_x(\bm{a}_*)  & \leq  \M_{\Psi}(\ell_x(A), \bm{\mu}), \nonumber \\& = \ell_x(a_1) + \tilde{\Psi}^*(s^*_{\delta}) - \tilde{\Psi}^*(s^*_{\delta} - \ell_x(\bm{a}_0) + \ell_x(\bm{a}_1)),\nonumber  \\
\shortintertext{and by letting $\sgn$ be the \emph{sign} function}
& \leq  \ell_x(a_1) + g_{\delta}(s^*_{\delta}, -\sgn[\ell_x(\bm{a}_0) - \ell_x(\bm{a}_1)]) [\ell_x(\bm{a}_0) - \ell_x(\bm{a}_1)],\label{40:e}
\end{align}
where in \eqref{40:e} we used the fact that $\tilde{\Psi}^*$ has non-decreasing slopes and the definition of $\delta$. When $\ell_x(\bm{a}_0) \leq \ell_x(\bm{a}_1)$, \eqref{40:e} becomes $\ell_x(\bm{a}_*) \leq (1- g^+) \ell_x(\bm{a}_1) + g^+ \ell_x(\bm{a}_0)$. Otherwise, we have $\ell_x(\bm{a}_*) \leq (1- g^-) \ell_x(\bm{a}_1) + g^- \ell_x(\bm{a}_0) < (1- g^+) \ell_x(\bm{a}_1) + g^+ \ell_x(\bm{a}_0)$. Since $\ell$ is admissible, there must exist at least one $x \in [n]$ such that $\ell_x(\bm{a}_0) > \ell_x(\bm{a}_1)$. Combining this with the fact that $p_x >0, \forall x\in[n]$ ($\bm{p} \in \Ri \Delta_n$), implies that $\inner{\bm{p}}{\ell(\bm{a}_*)} < \inner{\bm{p}}{(1-g^+)\ell(\bm{a}_1) + g^+ \ell(\bm{a}_0)} = \br_{\ell}(\bm{p})$. This contradicts the fact that $\ell(\bm{a}_*) \in \sps$. Therefore, $\br_{\ell}$ must be differentiable at $\bm{p}$. As argued earlier, this implies that $\br_{\ell}$ must be differentiable on $]0,+\infty[^n$. Combining this with the fact that $\br_{\ell}$ is concave on $]0,+\infty[^n$, implies that $\br_{\ell}$ is continuously differentiable on $]0,+\infty[^n$ (\ibid, Rmk. D.6.2.6).\\

[\textbf{We show $\tilde{\Phi}^* \in C^1(\mathbb{R}^{k-1})$}] 
Suppose that $\tilde{\Phi}^*$ is not differentiable at some $\bm{s}^*\in \mathbb{R}^{k-1}$. Then there exists $\bm{d}\in \mathbb{R}^{k-1} \setminus \{\bm{0}_{\tilde{k}}\}$ such that $-(\tilde{\Phi}^*)'(\bm{s}^*; - \bm{d})  <(\tilde{\Phi}^*)'(\bm{s}^*; \bm{d})$. Since $\bm{s}^* \in \Int \Dom \tilde{\Phi}^*$, $(\tilde{\Phi}^*)'(\bm{s}^*, \cdot)$ is finite and convex \citep[Prop. D.1.1.2]{\Hi}, and thus it is continuous on $\Dom \tilde{\Phi}^* =\mathbb{R}^{k-1}$ (\ibid, Rmk. B.3.1.3). Consequently, there exists $\delta^*>0$ such that \begin{align}
\forall \hbm{d} \in \mathbb{R}^{k-1}, \Norm{\hbm{d} -\bm{d}} \leq \delta^* \implies -(\tilde{\Phi}^*)'(\bm{s}^*; - \hbm{d})  <(\tilde{\Phi}^*)'(\bm{s}^*; \hbm{d}) \label{41:e}
\end{align} 
Let  $g \colon \{-1,1\}\rightarrow \mathbb{R}$ be such that \begin{align*}g(u)\coloneqq  \sup_{\Norm{\hbm{d} -\bm{d}} \leq \delta^*} u \cdot (\tilde{\Phi}^*)'(\bm{s}^*; u \hbm{d}).\end{align*} Note that since $\tilde{\Phi}^*$ has increasing slopes ($\tilde{\Phi}^*$ is convex), $g(1) \leq \sup_{\Norm{\hbm{d} -\bm{d}}\leq \delta^*} (\tilde{\Phi}^*)'_{\infty}(\hbm{d}) = \sup_{\Norm{\hbm{d} -\bm{d}}\leq \delta^*} \sigma_{\Dom \tilde{\Phi}}(\hbm{d}) \leq 1$, where the last inequality holds because $\tilde{\Delta}_k \subset \mathcal{B}(\bm{0}_{\tilde{k}}, 1)$, and thus $\sigma_{\Dom \tilde{\Phi}}(\hbm{d}) = \sigma_{\tilde{\Delta}_k}(\hbm{d})\leq  \sigma_{\mathcal{B}(\bm{0}_{\tilde{k}}, 1)}(\hbm{d}) =1$.
 Let $\Delta g \coloneqq g(1)-g(-1)$. From \eqref{41:e}, it is clear that $\Delta g>0$.

Suppose that $\br_{\ell}$ is twice differentiable on $]0,+\infty[^n$ and let $\sell$ be a support loss of $\ell$. By definition of a support loss, $\forall \bm{p}\in \Ri \Delta_k, \tilde{\sell}(\tbm{p}) = \sell(\bm{p})= \nabla \br_{\ell}(\bm{p})$ (where $\tilde{\sell} \coloneqq \sell \circ \amalg_n$). Thus, since $\br_{\ell}$ is twice differentiable on $]0,+\infty[^n$, $\tilde{\sell}$ is differentiable on $\Int \tilde{\Delta}_n$. Furthermore, $\sell$ is continuous on $\Ri \Delta_k$ given that $\br_{\ell} \in C^1(]0,+\infty[^n)$ as shown in the first part of this proof. We may assume without loss of generality that $\ell$ is not a constant function. Thus, from Theorem \ref{5:}, $\sell$ is not a constant function either. Consequently, the mean value theorem applied to $\sell$ (see e.g. \citep[Thm. 5.10]{Rudin1964}) between any two points in $\Ri \Delta_n$ with distinct images under $\sell$, implies that there exists $(\tbm{p}_*,\bm{v}_*)\in \Int \tilde{\Delta}_n \times \mathbb{R}^{n-1}$, such that $\mathsf{D}\tilde{\sell}(\tbm{p}_*)\bm{v}_* \neq \bm{0}_{\tilde{n}}$. For the rest of the proof let $(\tbm{p}, \bm{v})\coloneqq(\tbm{p}_*, \bm{v}_*)$ and define $\mathfrak{I}\coloneqq \{x \in[n]: \mathsf{D} \tilde{\sell}_x(\tbm{p})\bm{v} \neq 0\}$. From Lemma \ref{27:}, we have $\inner{\bm{p}}{\mathsf{D} \tilde{\sell}(\tbm{p})} =\bm{0}_{\tilde{n}}^{\mathsf{T}}$, which implies that there exists $x\in \mathfrak{I}, \mathsf{D} \tilde{\sell}_x(\tbm{p})\bm{v} >0$. Thus, the set \begin{align}
\mathfrak{K}\coloneqq \left\{x\in \mathfrak{I}\colon \mathsf{D} \tilde{\sell}_x(\tbm{p})\bm{v} > 0 \right\}\label{specialset}
\end{align} 
is non-empty. From this and the fact that $\bm{p}\in \Ri \Delta_n$, it follows that 
\begin{align}
\sum_{x'\in \mathfrak{K}} p_{x'} \mathsf{D} \tilde{\sell}_{x'}(\tbm{p})\bm{v} >0. \label{positivity}
\end{align}
Let $\tbm{p}^t \coloneqq \tbm{p} + t \bm{v}$. From Taylor's Theorem (see e.g. \citep[\S 151]{hardy2008}) applied to the function $t \mapsto \tilde{\sell}(\tbm{p}^t)$, there exists $\epsilon^*>0$ and functions $\delta_x : [-\epsilon^*,\epsilon^*] \rightarrow \mathbb{R}^n$, $x\in [n]$, such that $ \lim_{t \to 0} t^{-1}\delta_x(t) = 0$ and 
\begin{align}
\label{42:e}
\forall |t|\leq \epsilon^*, \quad  \sell_x(\bm{p}^{t}) = \sell_x(\bm{p}) +  t \mathsf{D} \tilde{\sell}_x(\bm{p})\bm{v}  + \delta_x(t). 
\end{align}
For $x\in[n]$, let $\bm{d}_x\in \mathbb{R}^{\tilde{k}}$ and suppose that $\norm{\bm{d}_x -\bm{d}}\leq \delta^*$ (we will define $\bm{d}_x$ explicitly later). By shrinking $\epsilon^*$ if necessary, we may assume that 
\begin{align}
&\forall x\in \mathfrak{I},\forall \theta \in[k], \forall |t| \leq \epsilon^*,\quad d_{\theta} t^{-1}\delta_x(t)\leq\tfrac{\delta^*\left|\mathsf{D} \tilde{\sell}_x(\tbm{p})\bm{v}\right|}{\sqrt{n}},\label{44:e}\\ 
&\forall x\not\in \mathfrak{I}, \quad \tilde{\Phi}^*(\bm{s}^*) -\tilde{\Phi}^*\left(\bm{s}^*-\left[\delta_x\left( \epsilon^*\tfrac{d_{\theta}}{\norm{\bm{d}}} \right) \right]_{\theta\in[\tilde{k}]}\right)\leq  \epsilon^* \tfrac{\Delta g}{4\norm{\bm{d}}} \sum_{x'\in \mathfrak{K}} p_{x'} \mathsf{D} \tilde{\sell}_{x'}(\tbm{p})\bm{v},\label{450:e} \\
&\forall x\in[n], \quad \tilde{\Phi}^*(\bm{s}^*) - \tilde{\Phi}^*\left(\bm{s}^* -\epsilon^*  \tfrac{\mathsf{D} \tilde{\sell}_x(\tbm{p})\bm{v}}{\norm{\bm{d}}}  \bm{d}_{x}\right)  \leq  -(\tilde{\Phi}^*)'\left(\bm{s}^*;-\epsilon^* \tfrac{\mathsf{D} \tilde{\sell}_x(\tbm{p})\bm{v}}{\norm{\bm{d}}}  \bm{d}_{x}\right) \nonumber \\ & \hspace{7cm} + \epsilon^* \tfrac{\Delta g}{4\norm{\bm{d}}}   \sum_{x'\in \mathfrak{K}} p_{x'} \mathsf{D} \tilde{\sell}_{x'}(\tbm{p})\bm{v},  \label{firstineq}
\end{align}  
where \eqref{firstineq} is satisfied for small enough $\epsilon^*$ because of \eqref{positivity} and the fact that
\begin{align}
\tfrac{1}{\epsilon} \left(\tilde{\Phi}^*(\bm{s}^*) - \tilde{\Phi}^*(\bm{s}^* - \epsilon \tfrac{\mathsf{D} \tilde{\sell}_x(\tbm{p})\bm{v}}{\norm{\bm{d}}}  \bm{d}_{x} \right) \underset{\epsilon \to 0}{\rightarrow} -(\tilde{\Phi}^*)'\left(\bm{s}^*;-\tfrac{\mathsf{D} \tilde{\sell}_x(\tbm{p})\bm{v}}{\norm{\bm{d}}}  \bm{d}_{x}\right), \nonumber
\end{align}
and \eqref{450:e} is also satisfied for small enough $\epsilon^*$ because $\tilde{\Phi}^*(\bm{s}^*) -\tilde{\Phi}^*\left(\bm{s}^*-\left[\delta_x\left( \epsilon\tfrac{d_{\theta}}{\norm{\bm{d}}} \right) \right]_{\theta\in[\tilde{k}]}\right) = O\left(\max_{\{\theta \in [\tilde{k}]\}}\left|\delta_x\left( \epsilon\tfrac{d_{\theta}}{\norm{\bm{d}}} \right)\right| \right) = o(\epsilon)$, where the first equality is due to the fact that $(\lambda, \bm{z})\mapsto \tfrac{1}{\lambda} \left(\tilde{\Phi}^*(\bm{s}^*) - \tilde{\Phi}^*(\bm{s}^* - \lambda \bm{z} \right)$ is uniformly bounded on compact subsets of $\mathbb{R} \times \mathbb{R}^{\tilde{k}}$ (by continuity of the directional derivative $(\tilde{\Phi}^*)'(\bm{s}^*;\cdot)$).

If $\mathsf{D} \tilde{\sell}_x(\tbm{p})\bm{v} \leq0$, then by the positive homogeneity of the directional derivative, the definition of the function $g$, and \eqref{firstineq}, we get 
\begin{align}
\tilde{\Phi}^*(\bm{s}^*) - \tilde{\Phi}^*\left(\bm{s}^* -\epsilon^*  \tfrac{\mathsf{D} \tilde{\sell}_x(\tbm{p})\bm{v}}{\norm{\bm{d}}}  \bm{d}_{x}\right)  \leq \epsilon^* \tfrac{\mathsf{D} \tilde{\sell}_x(\tbm{p})\bm{v}}{\norm{\bm{d}}}  g(1) + \epsilon^* \tfrac{\Delta g}{4\norm{\bm{d}}}   \sum_{x'\in \mathfrak{K}} p_{x'} \mathsf{D} \tilde{\sell}_{x'}(\tbm{p})\bm{v}. \label{secondineq}
\end{align}
On the other hand, if $\mathsf{D} \tilde{\sell}_x(\tbm{p})\bm{v}>0$, then from the monotonicity of the slopes of $\tilde{\Phi}^*$, the positive homogeneity of the directional derivative, and the definition of the function $g$, it follows that 
\begin{align}
\tfrac{1}{\epsilon^*} \left(\tilde{\Phi}^*(\bm{s}^*) - \tilde{\Phi}^*(\bm{s}^* - \epsilon^* \tfrac{\mathsf{D} \tilde{\sell}_x(\tbm{p})\bm{v}}{\norm{\bm{d}}}  \bm{d}_{x} \right)& \leq - (\tilde{\Phi}^*)'\left(\bm{s}^*;-\tfrac{\mathsf{D} \tilde{\sell}_x(\tbm{p})\bm{v}}{\norm{\bm{d}}}  \bm{d}_{x}\right), \nonumber \\
& =-\tfrac{\mathsf{D} \tilde{\sell}_x(\tbm{p})\bm{v}}{\norm{\bm{d}}} (\tilde{\Phi}^*)'\left(\bm{s}^*;-  \bm{d}_{x}\right), \nonumber \\
& \leq \tfrac{\mathsf{D} \tilde{\sell}_x(\tbm{p})\bm{v}}{\norm{\bm{d}}} g(-1), \nonumber  \\
& = \tfrac{\mathsf{D} \tilde{\sell}_x(\tbm{p})\bm{v}}{\norm{\bm{d}}}\left( -\Delta g  + g(1)\right).
\label{otherside}
\end{align} 
Let $\lambda_{\theta} \coloneqq \epsilon^* \frac{d_{\theta}}{\norm{\bm{d}}}$, for $\theta \in [\tilde{k}]$. From Theorem \ref{5:}, there exists $[\bm{a}_{\theta}]_{\theta \in [k]} \in \mathcal{A}^{k}$, such that \begin{align}\label{taylor}\ell(\bm{a}_k) =\sell(\bm{p})\quad \text{ and } \quad \forall \theta \in [\tilde{k}], \; \ell(\bm{a}_{\theta}) = \sell(\bm{p}^{\lambda_{\theta}}) = \sell(\bm{p}) + \epsilon^*\tfrac{d_{\theta}}{\norm{\bm{d}}}  \mathsf{D} \tilde{\sell}(\tbm{p})\bm{v} + \delta\left(\epsilon^*\tfrac{d_{\theta}}{\norm{\bm{d}}}\right),\end{align} where $[\delta(\cdot)]_x \coloneqq \delta_x(\cdot)$ for $x\in[n]$. 

From the fact that $\ell$ is $\Phi$-mixable, it follows that there exists $\bm{a}_* \in \mathcal{A}$ such that for all $x\in [n]$,
 \begin{align}
\ell_x(\bm{a}_*)  \leq  \M_{\Phi}(\ell_x(\bm{a}_{1:k}), \bm{\mu}) & = \ell_x(\bm{a}_k) + \tilde{\Phi}^*(\bm{s}^*) - \tilde{\Phi}^*\left(\bm{s}^* - J_k^{\mathsf{T}}\ell_x(\bm{a}_{1:k})\right).\label{46:e}
\end{align}
For $x\in[n]$, we now define $\bm{d}_x\in \mathbb{R}^{\tilde{k}}$ explicitly as \begin{align*}\forall \theta \in [\tilde{k}], \quad d_{x,\theta} \coloneqq \left\{ \begin{matrix} d_{\theta} + \frac{\norm{\bm{d}}}{\epsilon^*[\mathsf{D} \tilde{\sell}_x(\tbm{p}) \bm{v}]}  \delta_x\left( \epsilon^*\frac{d_{\theta}}{\norm{\bm{d}}}\right), & \text{ if } x\in \mathfrak{I} \\ d_{\theta}, & \text{ otherwise.} \end{matrix}\right.
\end{align*}
From \eqref{44:e}, we have $\Norm{\bm{d}_{x} - \bm{d}} \leq \delta^*, \forall x\in [n]$. Furthermore, from \eqref{taylor} and the fact that for all $x\in[n]$, $J_k^{\mathsf{T}} \ell_x(\bm{a}_{1:k}) = [\ell_x(\bm{a}_{\theta}) -\ell_x(\bm{a}_k) ]_{\theta \in[\tilde{k}]}$, we have 
\begin{align}J_k^{\mathsf{T}} \ell_x(\bm{a}_{1:k}) = \left\{  \begin{matrix} \epsilon^*\frac{\mathsf{D} \tilde{\sell}_x(\tbm{p})\bm{v}}{\norm{\bm{d}}}  \bm{d}_{x},& \text{ if } x\in \mathfrak{I}; \\  \left[\delta_x\left( \epsilon^*\frac{d_{\theta}}{\norm{\bm{d}}}\right)\right]_{\theta\in[\tilde{k}]},& \text{ otherwise.} \end{matrix} \right.\label{losses}
\end{align} 
Using this, together with \eqref{secondineq} and \eqref{otherside}, we get $\forall x\in\mathfrak{I}$,
\begin{align}
 \tilde{\Phi}^*(\bm{s}^*) - \tilde{\Phi}^*\left(\bm{s}^* - J_k^{\mathsf{T}}\ell_x(\bm{a}_{1:k})\right) &= \tilde{\Phi}^*(\bm{s}^*) - \tilde{\Phi}^*\left(\bm{s}^* -\epsilon^*  \tfrac{\mathsf{D} \tilde{\sell}_x(\tbm{p})\bm{v}}{\norm{\bm{d}}}  \bm{d}_{x}\right),\nonumber  \\& \leq   \epsilon^*\tfrac{\mathsf{D} \tilde{\sell}_x(\tbm{p})\bm{v}}{\norm{\bm{d}}} g(1) -\epsilon^* \Delta g \tfrac{\mathsf{D} \tilde{\sell}_x(\tbm{p})\bm{v} }{\norm{\bm{d}}} \mathbbm{1}_{\{\mathsf{D} \tilde{\sell}_x(\tbm{p})\bm{v} >0 \}} \nonumber  \\ & \hspace{2cm}   + \epsilon^* \tfrac{\Delta g}{4 \norm{\bm{d}}} \mathbbm{1}_{\{\mathsf{D} \tilde{\sell}_x(\tbm{p})\bm{v} \leq 0 \}} \sum_{x'\in \mathfrak{K}} p_{x'} \mathsf{D} \tilde{\sell}_{x'}(\tbm{p})\bm{v}. \label{47:e}  
\end{align}
Combining \eqref{46:e}, \eqref{losses}, and \eqref{47:e} yields
\begin{align}
\inner{\bm{p}}{\ell(\bm{a}_*)}& \leq \inner{\bm{p}}{\ell(\bm{a}_k)} + \tfrac{\epsilon^*}{\Norm{\bm{d}}} \inner{\bm{p}}{ \mathsf{D} \tilde{\sell}(\tbm{p})\bm{v}}g(1) - \tfrac{3\epsilon^*\Delta g}{4\norm{\bm{d}}} \sum_{x'\in \mathfrak{K}} p_{x'} \mathsf{D} \tilde{\sell}_{x'}(\tbm{p})\bm{v} \nonumber \\ & \hspace{3cm} + \sum_{x\not\in \mathfrak{I}} p_x \left(\tilde{\Phi}^*(\bm{s}^*) - \tilde{\Phi}^*\left(\bm{s}^* -\left[ \delta_x\left(\epsilon^*  \tfrac{d_{\theta}}{\norm{\bm{d}}}  \right)\right]_{\theta\in[\tilde{k}]}\right)\right),  \nonumber \shortintertext{using \eqref{450:e} and the fact that $\inner{\bm{p}}{\mathsf{D}\tilde{\sell}(\tbm{p})}=\bm{0}_{\tilde{n}}^{\mathsf{T}}$ (see Lemma \ref{27:}), we get}
& \leq \inner{\bm{p}}{\ell(\bm{a}_k)} -\tfrac{\epsilon^* \Delta g}{2\norm{\bm{d}}}\sum_{x'\in \mathfrak{K}} p_{x'} \mathsf{D} \tilde{\sell}_{x'}(\tbm{p})\bm{v}, \label{48:e}\\
& < \inner{\bm{p}}{\sell(\bm{p})}, \label{49:e}
\end{align}
where in \eqref{49:e} we used \eqref{positivity} and the fact that $\sell(\bm{p})=\ell(\bm{a}_k)$ (see \eqref{46:e}). Equation \ref{49:e} shows that $\ell(\bm{a}^*) \not\in \sps$, which is a contradiction.
\end{proof}

\subsection{Proof of Theorem \ref{8:}}
\label{a8:}

\textbf{Theorem \ref{8:}}
\emph{
Let $\eta>0$, and let $\ell \colon \mathcal{A} \rightarrow [0,+\infty]^n$ a loss. Suppose that $\Dom \ell = \mathcal{A}$ and that $\br_{\ell}$ is twice differentiable on $]0, +\infty[^n$. If $
\underline{\eta_{\ell}}>0$ then $\ell$ is $\underline{\eta_{\ell}}$-mixable. In particular, $\eta_{\ell} \geq \underline{\eta_{\ell}}$.
}
\begin{proof}
Let $\eta \coloneqq \underline{\eta_{\ell}}$. We will show that $\exp(-\eta  \sps)$ is convex, which will imply that $\ell$ is $\eta$-mixable \citep{Chernov2010}.

Since $\underline{\eta_{\ell}} = \inf_{\tbm{p} \in \Int \tilde{\Delta}_n} (\lambda_{\max} ([\mathsf{H} \tbr_{\log}(\tbm{p})]^{-1} \mathsf{H} \tbr_{\ell} (\tbm{p})))^{-1} >0$, $\eta \br_{\ell} - \br_{\log}$ is convex on $\Ri \Delta_n$ \citep[Thm. 10]{DBLP:journals/jmlr/ErvenRW12}. 
Let $\bm{p} \in \Ri \Delta_n$ and define \[{\Lambda}(\bm{r}) \coloneqq \br_{\log}(\bm{r}) + \inner{ \bm{r}}{\eta \sell(\bm{p}) - \ell_{\log}(\bm{p})}, \;\bm{r}\in \Ri \Delta_n.\] 
Since $\Lambda$ is equal to $\br_{\log}$ plus an affine function, it follows that $\eta \br_{\ell} -\Lambda$ is also convex on $\Ri \Delta_n$. On the one hand, since $\sell$ and $\ell_{\log}$ are proper losses, we have $\inner{\bm{p}}{\sell(\bm{p})} = \br_{\ell}(\bm{p})$ and $\inner{\bm{p}}{\ell_{\log}(\bm{p})} = \br_{\log}(\bm{p})$ which implies that \begin{align}\label{26:e}\eta \br_{\ell}(\bm{p}) -{\Lambda}(\bm{p})=0.\end{align} 
On the other hand, since $\br_{\ell}$ and $\br_{\log}$ are differentiable we have $\sell(\bm{p}) = \nabla \br_{\ell}(\bm{p})$ and $\nabla \br_{\log}(\bm{p}) = \ell_{\log}(\bm{p})$, which yields $\eta \nabla \br_{\ell}(\bm{p}) - \nabla {\Lambda}(\bm{p}) =\bm{0}_n$. This implies that $\eta \br_{\ell} - \Lambda$ attains a minimum at $\bm{p}$ \citep[Thm. D.2.2.1]{\Hi}. Combining this fact with \eqref{26:e} gives $\eta \br_{\ell}(\bm{r}) \geq {\Lambda}(\bm{r}), \forall \bm{r} \in \Ri \Delta_n$, or equivalently $-\eta \br_{\ell} \leq- {\Lambda}$. By Proposition \ref{1:}-(iii), this implies \begin{align}\label{27:e}[-\eta \br_{\ell}]^{\ast} \geq {[-\Lambda}]^{\ast}.\end{align} 
Using Proposition \ref{1:}-(ii), we get $[-{\Lambda}]^{\ast}(\bm{s}) = [-\br_{\log}]^*(\bm{s} - \ell_{\log}(\bm{p}) +\eta\sell(\bm{p}))$ for $\bm{s}\in \mathbb{R}^n$. 
		Since $-\eta \br_{\ell} (\bm{u}) = - \br_{\ell} (\eta \bm{u})=\sigma_{\sps} (-\eta \bm{u})$ and $\sigma^*_{\sps} = \I_{\sps}$, Proposition \ref{1:}-(v) implies $[-\eta \br_{\ell} ]^*(\bm{s})= \I_{\sps}( - \bm{s}/\eta)$. Similarly, we have $[-\br_{\log}]^{\ast}(\bm{s}) = \I_{\mathscr{S}_{\log}}( - \bm{s})$. Therefore, \eqref{27:e} implies 
			\[ \forall \bm{s} \in \mathbb{R}^n, \quad {\I}_{\sps} (-\bm{s}/\eta) \geq  {\I}_{\mathscr{S}_{\log}} (-\bm{s} + \ell_{\log}(\bm{p}) - \eta \sell(\bm{p})). \]

			This inequality implies that if $\bm{s} \in -\eta \sps$, then $\bm{s} \in -\mathscr{S}_{\log} + \ell_{\log}(\bm{p}) - \eta \sell(\bm{p})$. In particular, if $\bm{u} \in e^{-\eta \sps}$ then \begin{align}\bm{u} \in e^{-\mathscr{S}_{\log} + \ell_{\log}(\bm{p}) -\eta \sell(\bm{p})} \subseteq \mathcal{\mathsf{\mathcal{H}}}_{\tau(\bm{p}),1} =  \{\bm{v}\in \mathbb{R}^n: \inner{\bm{v}}{\bm{p} \odot e^{\eta \sell(\bm{p})}  } \leq 1 \}.\label{setinc:e}  \end{align} 
To see the set inclusion in \eqref{setinc:e}, consider $\bm{s}\in -\mathscr{S}_{\log} + \ell_{\log}(\bm{p}) - \eta \sell(\bm{p})$, then by definition of the superprediction set $\mathscr{S}_{\log}$ there exists $\bm{r}\in \Delta_n$ and $\bm{v} \in [0,+\infty[^n$, such that $\bm{s} = \log \bm{r} - \log {\bm{p}} -  \eta \sell(\bm{p}) - \bm{v} $. Thus, \begin{align}\inner{e^{\bm{s}}}{\bm{p} \odot e^{\eta \sell(\bm{p})}} = \inner{\bm{r}}{e^{-\bm{v}}} \leq 1, \label{28:e}\end{align} 
where the inequality is true because $\bm{r} \in \Delta_n$ and $\bm{v}\in [0,+\infty[^n$. The above argument shows that $e^{-\eta \sps}\subseteq\mathcal{\mathsf{\mathcal{H}}}_{\tau(\bm{p}),1}$, where $\tau(\bm{p})\coloneqq \bm{p} \odot e^{\eta \underline{\ell}(\bm{p})}$. Furthermore, $e^{-\eta \sps}\subseteq \mathcal{\mathsf{\mathcal{H}}}_{\tau(\bm{p}),1} \cap ]0,+\infty[^n$,
since all elements of $e^{-\eta \sps}$ have non-negative, finite components. The latter set inclusion still holds for $\hat{\bm{p}} \in \Rbd \Delta_n$. In fact, from the definition of a support loss, there exists a sequence $(\bm{p}_m)$ in $\Ri \Delta_n$ converging to $\hat{\bm{p}}$ such that $\sell(\bm{p}_m) \stackrel{m \to \infty}{\to} \sell(\hat{\bm{p}})$. Equation \ref{28:e} implies that for $\bm{u} \in  e^{-\eta \sps}$, $\inner{\bm{u}}{ \bm{p}_m \odot e^{\eta \sell(\bm{p}_m)}}\leq 1$. Since the inner product is continuous, by passage to the limit, we obtain $\inner{\bm{u}}{\hat{\bm{p}} \odot e^{\eta \sell(\hat{\bm{p}})}}\leq 1$. Therefore, \begin{align}
e^{-\eta \sps} \subseteq \bigcap_{\bm{p} \in \Delta_n}\mathcal{\mathsf{\mathcal{H}}}_{\tau(\bm{p}),1} \cap ]0,+\infty[^n\label{29:e}.\end{align}
			
			Now suppose $\bm{u} \in \bigcap_{\bm{p} \in \Delta_n}\mathcal{\mathsf{\mathcal{H}}}_{\tau(\bm{p}),1} \cap ]0,+\infty[^n$; that is, for all $\bm{p}\in \Delta_n$,
			\begin{align}
			1\geq  	\Inner{\bm{u}}{\bm{p} \odot e^{\eta \sell (\bm{p})}} =\Inner{\bm{p}}{\bm{u} \odot e^{\eta \sell (\bm{p}) }}  &= \Inner{\bm{p}}{e^{\eta \sell (\bm{p})  + \log \bm{u} }}, \nonumber \\
			& \geq e^{\Inner{\bm{p}}{\eta \sell(\bm{p})} + \Inner{\bm{p}}{\log \bm{u}}}, \label{30:e}
			\end{align}
			where the first equality is obtained merely by expanding the expression of the inner product, and the second inequality is simply Jensen's Inequality. Since $\bm{u} \mapsto e^{\bm{u}}$ is strictly convex, the Jensen's inequality in \eqref{30:e} is strict unless $\exists (c, \bm{p})\in  \mathbb{R} \times \Delta_n $, such that 
\begin{align}
\label{31:e}
\eta \sell (\bm{p}) + \log \bm{u} = c \bm{1}_n.
\end{align}
 By substituting \eqref{31:e} into \eqref{30:e}, we get $1\geq \exp(c)$, and thus $c\leq 0$. Furthermore, \eqref{31:e} together with the fact that $\bm{u} \in ]0,+\infty[^n$ imply that $\bm{p}\in \Dom \sell$, and thus there exists $\bm{a}\in \Dom \ell$ such that $\ell(\bm{a})=\sell(\bm{p})$ (Theorem \ref{5:}). Using this and rearranging \eqref{31:e}, we get $\bm{u} = \exp(-\eta \ell(\bm{a}) + c \bm{1})$. Since $c\leq 0$, this means that $\bm{u} \in \exp(-\eta \sps)$. Suppose now that \eqref{31:e} does not hold. In this case, \eqref{30:e} must be a strict inequality for all  $\bm{p} \in \Delta_n$.  By applying the $\log$ on both side of \eqref{30:e}, 
 \begin{align}\forall \bm{p} \in \Delta_n, \eta\br_{\ell}(\bm{p}) + \inner{\bm{p}}{ \log \bm{u}} =  \inner{\bm{p}}{\eta \sell (\bm{p})} + \inner{\bm{p}}{\log \bm{u}} <0 \label{32:e}.\end{align} 
Since $\bm{p}\mapsto \br_{\ell}(\bm{p}) =-\sigma_{\sps}(-\bm{p})$ is a closed concave function, the map $g\colon \bm{p}\mapsto \eta \br_{\ell}(\bm{p}) + \inner{\bm{p}}{ \log \bm{u}}$ is also closed and concave, and thus upper semi-continuous. Since $\Delta_n$ is compact, the function $g$ must attain its maximum in $\Delta_n$. Due to \eqref{32:e} this maximum is negative; there exists $c_1>0$ such that \begin{align}\forall \bm{p}\in \Delta_k,  \inner{\bm{p}}{\eta \sell(\bm{p})} - \inner{\bm{p}}{- \log \bm{u}}\leq- c_1. \label{stuff2} \end{align} 
Let $f(\bm{p},x) \coloneqq \eta \sell_x(\bm{p}) + \log u_x + c_1$, for $x\in [n]$. It follows from \eqref{stuff2} that for all $\bm{p}\in \Delta_n$, $\mathbb{E}_{x\sim \bm{p}} f(\bm{p}, x)\leq 0$ and $\forall x \in[n], -\infty<f(\bm{p}, x)$. Thus, Lemma \ref{25:} applied to $f$ with $\epsilon = c_1/2$, implies that there exists $\bm{p}_* \in \Ri \Delta_n$, such that $\eta \sell(\bm{p}_*) \leq -\log \bm{u} -c_1/2 \leq -\log \bm{u}$. From this inequality, $\bm{p}_*\in \Dom \sell$, and therefore, there exists $\bm{a}_*\in \Dom \ell$ such that $\ell(\bm{a}_*)=\sell(\bm{p}_*)$ (Theorem \ref{5:}). This shows that $\eta \ell(\bm{a}_*) \leq  -\log \bm{u}$, which implies that $\bm{u} \in \exp{-\eta \sps}$. Therefore, $ \bigcap_{\bm{p} \in \Delta_n}\mathcal{\mathsf{\mathcal{H}}}_{\tau(\bm{p}),1} \cap ]0,+\infty[^n \subseteq e^{-\eta \sps}$. Combining this with \eqref{29:e} shows that $e^{-\eta \sps} = \bigcap_{\bm{p} \in \Delta_n} \mathcal{\mathsf{\mathcal{H}}}_{\tau(\bm{p}),1} \cap ]0,+\infty[^n$. Since $e^{-\eta \sps}$ is the intersection of convex set, it is a itself convex set. Since $\Dom \ell = \mathcal{A}$ by assumption, it follows that $\sps = \sps^{\infty}$, and thus $e^{-\eta \sps^{\infty}}$ is convex. This last fact implies that $\ell$ is $\eta$-mixable \citep{Chernov2010}.
\end{proof}

\subsection{Proof of Theorem \ref{constant regret}}
We start by the following characterization of $\Delta$-differentiability (this was defined on page 5 of the main body of the paper).
\begin{lemma}
\label{chardelta}
Let $\Phi\colon \mathbb{R}^k \rightarrow \mathbb{R}\cup \{+\infty\}$ be an entropy. Then $\Phi$ is $\Delta$-differentiable if and only if $\forall \cI\subseteq[k]$ such that $|\cI|>1$, $\tilde{\Phi}_{\cI} \coloneqq \Phi \circ \amalg_{k} \circ [\Pi^{\tilde{k}}_{\cI}]^{\T}$ is differentiable on $\Int \tilde{\Delta}_{|\cI|}$.
\end{lemma}
\begin{proof}
This is a direct consequence of Proposition B.4.2.1 in \cite{Hiriart-Urruty}, since 1) $\tilde{\Phi}_{\cI}$ is convex; and 2) \begin{align*}\tilde{\Phi}'_{\cI}(\tbm{u};\tbm{v}- \tbm{u})& = \tilde{\Phi}'([\Pi^{\tilde{k}}_{\cI}]^{\T} \tbm{u};[\Pi^{\tilde{k}}_{\cI}]^{\T}(\bm{v}- \bm{u})), \\& = \Phi'(\amalg_k[\Pi^{\tilde{k}}_{\cI}]^{\T} \tbm{u};\amalg_k[\Pi^{\tilde{k}}_{\cI}]^{\T}(\bm{v}- \bm{u})),\end{align*}
for all $\tbm{u}, \tbm{v} \in \Int \tilde{\Delta}_{|\cI|}$ and $\tilde{\Phi} \coloneqq \Phi  \circ \amalg_k$.
\end{proof}

\textbf{Theorem \ref{constant regret}}
\emph{
Let $\Phi: \mathbb{R} \rightarrow \mathbb{R} \cup\{+\infty\}$ be a $\Delta$-differentiable entropy. Let $\ell: \mathcal{A}\rightarrow [0,+\infty]^n$ be a loss (not  necessarily finite) such that $\br_{\ell}$ is twice differentiable on $]0,+\infty[^n$. If $\ell$ is $(\eta, \Phi)$-mixable then the GAA achieves a constant regret in the $\mathfrak{G}^n_{\ell}(\mathcal{A},k)$ game; for any sequence $(x^t, \bm{a}^t_{1:k})_{t=1}^T$,
	\begin{align*}
	\op{Loss}^{\ell}_{\textsc{GAA}}(T) - \min_{\theta \in [k]} \op{Loss}^{\ell}_{\theta}(T) \leq R^{\Phi}_{\ell}\coloneqq  \inf_{\bm{q}\in \Delta_k} \max_{\theta \in [k]}  D_{\Phi}(\bm{e}_{\theta}, \bm{q})/\eta^{\Phi}_{\ell},
	\end{align*}
where $\bm{e}_{\theta}$ is the $\theta$th basis element of $\mathbb{R}^k$. 
}
\begin{proof}
For all $\cI\subseteq[k]$ such that $|\cI|>1$, let $\tilde{\Phi} \coloneqq \Phi \circ \amalg_{k}$ and $\tilde{\Phi}_{\cI} \coloneqq \tilde{\Phi} \circ [\Pi^{\tilde{k}}_{\cI}]^{\T}$.
From Lemma \ref{15:} the infimum involved in the definition of the expert distribution $\bm{q}^t$ in Algorithm \ref{GAA2} is indeed attained. It remains to verify that this minimum is unique. This will become clear in what follows.

Let $\cI^0 = [k]$ and $\mathfrak{I}^t \coloneqq \{\theta \in [k]: \ell_{x^t}(\bm{a}^t_{\theta}) < +\infty\}, t \in [T]$. For $t \in [T]$, we define the non-increasing sequence of subsets $(\cI^t)$ of $[k]$ defined by $\cI^t \coloneqq  \mathfrak{I}^t \cap \cI^{t-1}$. We show by induction that $\bm{q}^t\in \Delta_{\cI^{t}}$ and
 \begin{align} \nabla \tilde{\Phi}_{\cI^{t}}(\Pi^{\tilde{k}}_{\cI^t} \tbm{q}^t) =\Pi^{\tilde{k}}_{\cI^t} \left( \nabla \tilde{\Phi} (\tbm{q}^0)- \sum_{s=1}^t  J_k^{\T}\ell_{x^s}(A^s)\right), \label{ind2}\end{align} 
 where $A^s \coloneqq [\bm{a}^s_{\theta}] \in \mathcal{A}^k$, $s\in\mathbb{N}$. Suppose that \eqref{ind2} holds true up to some $t\geq1$. We will now show that it holds for $t+1$. To simplify expressions, we denote $\tbm{x}_{\cI}\coloneqq \Pi^{\tilde{k}}_{\cI} \tbm{x}\in \mathbb{R}^{\cI}$ for $\tbm{x}\in \mathbb{R}^{\tilde{k}}$, and $\bm{z}^t \coloneqq \ell_{x^t}(A^t), t \in [T]$. 
From the definition of $\bm{q}^t$ in Algorithm \ref{GAA2}, we have 
\begin{align}
\bm{q}^{t+1} &\in \mathcal{M} \coloneqq  \Argmin_{\bm{\mu} \in \Delta_k} \inner{\bm{\mu}}{\bm{z}^{t+1}} + D_{\Phi}(\bm{\mu}, \bm{q}^{t}). \nonumber \\
\shortintertext{Using the definition of $\mathfrak{I}^{t+1}$, }
\mathcal{M} & = \Argmin_{\bm{\mu} \in \Delta_{\cI^{t+1}}} \inner{\bm{\mu}}{\bm{z}^{t+1}} + D_{\Phi}(\bm{\mu}, \bm{q}^{t}),  \nonumber \\
& = \Argmin_{\bm{\mu} \in \Delta_{\cI^{t+1}}} \inner{\bm{\mu}}{\bm{z}^{t+1}} + \tilde{\Phi}_{\cI^t}(\tbm{\mu}_{\cI^t}) - \tilde{\Phi}_{\cI^t}(\tbm{q}_{\cI}^{t}) - \tilde{\Phi}'_{\cI^{t}}(\tbm{q}_{\cI^t}^{t};\tbm{\mu}_{\cI^t}- \tbm{q}_{\cI^t}^{t}). \nonumber\\
\shortintertext{Now using the facts that $\bm{q}^{t}\in \Delta_{\cI^{t}}$, $\bm{\mu}\in \Delta_{\cI^{t+1}}\subseteq \Delta_{\cI^{t}}$, $\Phi$ is $\Delta$-differentiable, and Lemma \ref{chardelta}, we have}
& \mathcal{M}= \Argmin_{\bm{\mu} \in \Delta_{\cI^{t+1}}} \inner{\bm{\mu}}{\bm{z}^{t+1}} +  \tilde{\Phi}_{\cI^{t+1}}(\tbm{\mu}_{\cI^{t+1}}) - \tilde{\Phi}_{\cI^{t}}(\tbm{q}_{\cI^t}^{t}) - \inner{\tbm{\mu}_{\cI^{t}}-\tbm{q}_{\cI^{t}}^{t}}{\nabla \tilde{\Phi}_{\cI^t}(\tbm{q}_{\cI^{t}}^{t})}. \nonumber \\
\intertext{Using the facts that $\inner{\bm{\mu}}{\bm{z}^{t+1}} =z^{t+1}_k + \inner{\tbm{\mu}_{\cI^{t+1}}}{\Pi^{\tilde{k}}_{\cI^{t+1}} J^{\T}_k \bm{z}^{t+1}}$, for $\tbm{\mu}\in \tilde{\Delta}_{\cI^{t+1}}$, and $\inner{\tbm{\mu}_{\cI^{t}}}{\nabla \tilde{\Phi}_{\cI^t}(\tbm{q}_{\cI^{t}}^{t})}\mathcal{M}=\inner{\tbm{\mu}_{\cI^{t+1}}}{\Pi^{\cI^t}_{\cI^{t+1}} \nabla \tilde{\Phi}_{\cI^t}(\tbm{q}_{\cI^{t}}^{t})}$ (since $\bm{\mu}\in \Delta_{\cI^{t+1}}$)}
& \mathcal{M}= \Argmin_{\bm{\mu} \in \Delta_{\cI^{t+1}}} \inner{\tbm{\mu}_{\cI^{t+1}}}{-\Pi^{\cI^t}_{\cI^{t+1}} \nabla \tilde{\Phi}_{\cI^t}(\tbm{q}_{\cI^{t}}^{t})+\Pi^{\tilde{k}}_{\cI^{t+1}} J^{\T}_k \bm{z}^{t+1}} +  \tilde{\Phi}_{\cI^{t+1}}(\tbm{\mu}_{\cI^{t+1}})\nonumber \\&\hspace{6cm}+ \inner{\tbm{q}_{\cI^{t}}^{t}}{\nabla \tilde{\Phi}_{\cI^t}(\tbm{q}_{\cI^{t}}^{t})} - \tilde{\Phi}_{\cI^{t}}(\tbm{q}_{\cI^t}^{t}),  \nonumber \\
\shortintertext{and since the last two terms are independent of $\bm{\mu}$,}
&\mathcal{M}= \Argmin_{\bm{\mu} \in \Delta_{\cI^{t+1}}} \inner{\tbm{\mu}_{\cI^{t+1}}}{ -\Pi^{\cI^t}_{\cI^{t+1}} \nabla \tilde{\Phi}_{\cI^t}(\tbm{q}_{\cI^{t}}^{t})+ \Pi^{\tilde{k}}_{\cI^{t+1}}J^{\T}_k \bm{z}^{t+1}} +  \tilde{\Phi}_{\cI^{t+1}}(\tbm{\mu}_{\cI^{t+1}}).  \nonumber \\
\shortintertext{Now using Fenchel duality property in Proposition \ref{1:}-(iv),}
& \mathcal{M}=\{ \bm{\mu}\in \Delta_{\cI^{t+1}}:  \Pi^{\tilde{k}}_{\cI^{t+1}} \circ \Pi_k (\bm{\mu}) = \tbm{\mu}_{\cI^{t+1}}  \in \partial \tilde{\Phi}_{\cI^{t+1}}^*(\Pi^{\cI^t}_{\cI^{t+1}}\nabla \tilde{\Phi}_{\cI^t}(\tbm{q}_{\cI^{t}}^{t})-  \Pi^{\tilde{k}}_{\cI^{t+1}} J^{\T}_k \bm{z}^{t+1})\}. \nonumber\\
\shortintertext{Finally, due to Lemma \ref{12:} and Proposition \ref{14:},  $\tilde{\Phi}_{\cI^{t+1}}^*$ is differentiable on $\mathbb{R}^{|\cI^{t+1}|-1}$, and thus}
& \mathcal{M}= \{\amalg_k \circ [\Pi^{\tilde{k}}_{\cI^{t+1}}]^{\T} \circ \nabla \tilde{\Phi}_{\cI^{t+1}}^*(\Pi^{\cI^t}_{\cI^{t+1}}\nabla \tilde{\Phi}_{\cI^t}(\tbm{q}_{\cI^{t}}^{t})-  \Pi^{\tilde{k}}_{\cI^{t+1}} J^{\T}_k \bm{z}^{t+1})\}.\label{fine}
\end{align}
From \eqref{fine}, we obtain  
\begin{align}
\nabla \tilde{\Phi}_{\cI^{t+1}}(\Pi^{\tilde{k}}_{\cI^{t+1}}\tbm{q}^{t+1}) = \Pi^{\cI^t}_{\cI^{t+1}}\nabla \tilde{\Phi}_{\cI^t}(\tbm{q}_{\cI^{t}}^{t})-  \Pi^{\tilde{k}}_{\cI^{t+1}}J^{\T}_k \bm{z}^{t+1}. \label{ind1}
\end{align} 
Thus using the induction assumption and the fact that $\Pi^{\cI^t}_{\cI^{t+1}}\Pi^{\tilde{k}}_{\cI^{t}} = \Pi^{\tilde{k}}_{\cI^{t+1}}$ (since $\cI^{t+1} \subseteq \cI^t$), the result follows, i.e. \eqref{ind2} is true for all $t\in [T]$. Furthermore, $\bm{q}^{t+1}\in \Delta_{\cI^{t+1}}$, since $\Pi^{\tilde{k}}_{\cI^{t+1}}\tbm{q}^{t+1} \in \Dom \tilde{\Phi}_{\cI^{t+1}} \subseteq \tilde{\Delta}_{|\cI^{t+1}|}$.
Using the same arguments as above, one arrives at
\begin{align}
\Mix_{\Phi}(\bm{q}^t, \bm{z}^{t+1}) &=z^{t+1}_k + \inf_{\bm{\mu} \in \Delta_{\cI^{t+1}}} \inner{\tbm{\mu}_{\cI^{t+1}}}{-\Pi^{\cI^t}_{\cI^{t+1}}\nabla \tilde{\Phi}_{\cI^t}(\tbm{q}_{\cI^{t}}^{t})+ \Pi^{\tilde{k}}_{\cI^{t+1}}J^{\T}_k \bm{z}^{t+1}}  +  \tilde{\Phi}_{\cI^{t+1}}(\tbm{\mu}_{\cI^{t+1}}) \nonumber \\ &\hspace{3cm} + \inner{\tbm{q}_{\cI^{t}}^{t}}{\nabla \tilde{\Phi}_{\cI^t}(\tbm{q}_{\cI^{t}}^{t})} - \tilde{\Phi}_{\cI^{t}}(\tbm{q}_{\cI^t}^{t}).\nonumber 
\shortintertext{Using the Fenchel duality property Proposition \ref{1:}-(vi) and \eqref{fine},}
& = z^{t+1}_k + \tilde{\Phi}^*_{\cI^{t}}(\nabla \tilde{\Phi}_{\cI^t}(\tbm{q}^t_{\cI^t})) - \tilde{\Phi}^*_{\cI^{t+1}}(\Pi^{\cI^t}_{\cI^{t+1}} \nabla \tilde{\Phi}_{\cI^t}(\tbm{q}_{\cI^{t}}^{t})-  \Pi^{\tilde{k}}_{\cI^{t+1}} J^{\T}_k \bm{z}^{t+1}).  \label{tosum}
\end{align}
On the other hand, $\Phi$-mixability implies that there exists $\bm{a}^t_*\in \mathcal{A}^{t}$, such that for all $x^t \in [n]$, 
\begin{align}
\forall t \in [T], \ell_{x^t}(\bm{a}^t_*) &\leq \Mix_{\Phi}(\bm{q}^{t-1}, \bm{z}^t),  \nonumber\\
\shortintertext{Summing this inequality for $t=1,\dots, T$ yields,}
\sum_{t=1}^T \ell_{x^t}(\bm{a}^t_*)& \leq \sum_{t=1}^T \Mix_{\Phi}(\bm{q}^{t-1}, \bm{z}^t),\nonumber  \\
\shortintertext{and thus using \eqref{tosum} and \eqref{ind1} yields}
\sum_{t=1}^T \ell_{x^t}(\bm{a}^t_*)& \leq \sum_{t=1}^T \ell_{x^t}(\bm{a}^t_k) + \tilde{\Phi}^*(\nabla \tilde{\Phi}(\tbm{q}^0)) - \tilde{\Phi}^*_{\cI^{T}}(\Pi^{\cI^{T-1}}_{\cI^{T}} \nabla \tilde{\Phi}_{\cI^{T-1}}(\tbm{q}_{\cI^{T-1}}^{T-1})- \Pi^{\tilde{k}}_{\cI^{T}}J^{\T}_k \bm{z}^{T}). \nonumber
\shortintertext{Finally, using \eqref{ind2} together with the fact that $\Pi^{\cI^{T-1}}_{\cI^{T}}\Pi^{\tilde{k}}_{\cI^{T-1}} = \Pi^{\tilde{k}}_{\cI^{T}}$}
\sum_{t=1}^T \ell_{x^t}(\bm{a}^t_*)& \leq \sum_{t=1}^T \ell_{x^t}(\bm{a}^t_k) + \tilde{\Phi}^*(\nabla \tilde{\Phi}(\tbm{q}^0)) - \tilde{\Phi}^*_{\cI^{T}}\left(\Pi^{\tilde{k}}_{\cI^{T}} \left(\nabla \tilde{\Phi}(\tbm{q}^0) - \sum_{t=1}^T J_k^{\T}\ell_{x^t}(A^t)\right) \right).\nonumber
\end{align}
Using the definition of the Fenchel dual and Proposition \ref{1:}-(vi) again, the above inequality becomes
\begin{align}
\sum_{t=1}^T \ell_{x^t}(\bm{a}^t_*)& \leq \sum_{t=1}^T \ell_{x^t}(\bm{a}^t_k) + \inner{\tbm{q}^0}{\nabla \tilde{\Phi}(\tbm{q}^0))} - \tilde{\Phi}(\tbm{q}^0) \nonumber\\& \hspace{1cm} - \sup_{\bm{\pi}\in \Delta_{|\cI^{T}|}}\left[ \Inner{\tbm{\pi}}{\Pi^{\tilde{k}}_{\cI^{T}} \left(\nabla \tilde{\Phi}(\tbm{q}^0) - \sum_{t=1}^T J_k^{\T}\ell_{x^t}(A^t)\right)} - \tilde{\Phi}_{\cI^{T}}(\tbm{\pi})\right], \nonumber \\
 &= \sum_{t=1}^T \ell_{x^t}(\bm{a}^t_k) + \inner{\tbm{q}^0}{\nabla \tilde{\Phi}(\tbm{q}^0))} - \tilde{\Phi}(\tbm{q}^0)\nonumber \\& \hspace{1cm} + \inf_{\bm{\mu}\in \Delta_{\cI^{T}}}\left[ \Inner{\tbm{\mu}}{ \sum_{t=1}^{T} J_k^{\T}\ell_{x^t}(A^t)- \nabla \tilde{\Phi}(\tbm{q}^0)} + \tilde{\Phi}(\tbm{\mu})\right].\label{rhsineq}\end{align}
 Using the fact that $\forall \theta \in [k] \setminus \cI^{T}$, $\sum_{t=1}^T\ell_{x^t}(\bm{a}^t_{\theta}) =+\infty$ (by definition of $(\cI^t)$), the right hand side of \eqref{rhsineq} becomes
 \begin{align}
 \sum_{t=1}^T \ell_{x^t}(\bm{a}^t_k) + \inner{\tbm{q}^0}{\nabla \tilde{\Phi}(\tbm{q}^0))} - \tilde{\Phi}(\tbm{q}^0) + \inf_{\bm{\mu}\in \Delta_k}\left[ \Inner{\tbm{\mu}}{ \sum_{t=1}^{T} J_k^{\T}\ell_{x^t}(A^t)- \nabla \tilde{\Phi}(\tbm{q}^0)} + \tilde{\Phi}(\tbm{\mu})\right].\nonumber \end{align}
 Thus, we get 
 \begin{align}
 \forall \bm{\mu} \in \Delta_k, \;\sum_{t=1}^T \ell_{x^t}(\bm{a}^t_*)& \leq \sum_{t=1}^T \ell_{x^t}(\bm{a}^t_k) +\Inner{\tbm{\mu}}{ \sum_{t=1}^{T} J_k^{\T}\ell_{x^t}(A^t)} \nonumber \\ & \hspace{4cm} + \tilde{\Phi}(\tbm{\mu}) - \tilde{\Phi}(\tbm{q}^0) - \inner{\tbm{\mu}-\tbm{q}^0}{\nabla \tilde{\Phi}(\tbm{q}^0)}. \nonumber \\
  \shortintertext{Using the facts that $\sum_{t=1}^T \ell_{x^t}(\bm{a}^t_k) +\Inner{\tbm{\mu}}{ \sum_{t=1}^{T} J_k^{\T}\ell_{x^t}(A^t)} = \Inner{\bm{\mu}}{ \sum_{t=1}^{T} \ell_{x^t}(A^t)}$ and the definition of the divergence,}
   \forall \bm{\mu} \in \Delta_k, \; \sum_{t=1}^T \ell_{x^t}(\bm{a}^t_*) & \leq\Inner{\bm{\mu}}{ \sum_{t=1}^{T} \ell_{x^t}(A^t)}  + D_{\Phi}(\bm{\mu}, \bm{q}^0),\nonumber 
   \shortintertext{which for $\bm{\mu}=\bm{e}_{\theta}$ implies}
 \forall \theta\in [k],  \; \sum_{t=1}^T \ell_{x^t}(\bm{a}^t_*)& \leq \sum_{t=1}^T \ell_{x^t}(\bm{a}^t_{\theta}) + D_{\Phi}(\bm{e}_{\theta}, \bm{q}^0).\label{last3}
\end{align}
When instead of $\Phi$-mixability, we have $(\eta, \Phi)$-mixability, the last term in \eqref{last3} becomes $\frac{D_{\Phi}(\bm{e}_{\theta}, \bm{q}^0)}{\eta}$ and the desired result follows.

\end{proof}

\subsection{Proof of Theorem \ref{16:}}
\label{a16:}
We require the following result:
	\begin{proposition}
\label{2:}
For the Shannon entropy $\Se$, it holds that $\tilde{\Se}^* (\bm{v}) = \log(\inner{ \exp(\bm{v})}{\bm{1}_{\tilde{k}}}+ 1), \forall \bm{v}\in \mathbb{R}^{k-1}$, and  $ {\Se}^{\star}(\bm{z}) = \log \inner{ \exp(\bm{z})}{\bm{1}_{k}}, \forall \bm{z}\in \mathbb{R}^k$.
\end{proposition}
\begin{proof}
Given $\bm{v} \in \mathbb{R}^{k-1}$, we first derive the expression of the Fenchel dual $\tilde{\Se}^*(\bm{v}) \coloneqq \sup_{\tbm{q}\in \tilde{\Delta}_k} \inner{\tbm{q}}{\bm{v}} - \tilde{\Se}(\tbm{q})$. Setting the gradient of $\tbm{q} \mapsto \inner{\tbm{q}}{\bm{v}} - \tilde{\Se}(\tbm{q})$ to $\bm{0}_{\tilde{k}}$ gives $\bm{v} = \nabla \tilde{\Se}(\tbm{q})$. For $\bm{q} \in ]0, +\infty[^k$, we have $\nabla \Se (\bm{q}) = \log \bm{q} + \bm{1}_{k}$, and from \cref{s2.1} we know that $\nabla \tilde{\Se} (\tbm{q}) = J^{\mathsf{T}}_k \nabla \Se(\bm{q})$. Therefore,
		\[ \bm{v} = \nabla \tilde{\Se}(\tbm{q}) \implies \bm{v} = J^{\mathsf{T}}_k \nabla \Se(\bm{q})\implies \bm{v} = \log \frac{\tbm{q}}{q_k},  \]
		where the right most equality is equivalent to $\tbm{q}/q_k = \exp(\bm{v})$.  Since $\inner{\tbm{q}}{\bm{1}_{\tilde{k}}} = 1 - q_k$, we get $q_k = (\inner{ \exp(\bm{v})}{\bm{1}_{\tilde{k}}}+1)^{-1}$. Therefore, the supremum in the definition of $\tilde{\Se}^*(\bm{v})$ is attained at $\tbm{q}_* =  \exp(\bm{v}) (\inner{ \exp(\bm{v})}{\bm{1}_{\tilde{k}}}+1)^{-1}$. Hence $\tilde{\Se}^* (\bm{v}) = \inner{\tbm{q}_*}{\bm{v}} - \inner{\tbm{q}_*}{\log \tbm{q}_*} 
		= \log(\inner{ \exp(\bm{v})}{\bm{1}_{\tilde{k}}}+ 1)$. Finally, using \eqref{2:e} we get $\Se^{\star}(\bm{z}) = \log \inner{ \exp(\bm{z})}{\bm{1}_{k}}$, for $\bm{z}\in \mathbb{R}^k$.
\end{proof}

\textbf{Theorem \ref{16:}}\emph{
Let $\eta>0$.  A loss $\ell \colon \mathcal{A} \rightarrow [0,+\infty]^n$ is $\eta$-mixable if and only if $\ell$ is $(\eta,\Se)$-mixable.
}
\begin{proof}
\begin{claim}
For all $\bm{q}\in \Delta_k$, $A \coloneqq \bm{a}_{1:k} \in \mathbb{R}^k$, and $x\in[n]$
\begin{align}- \eta^{-1} \log \Inner{\exp(- \eta  \ell_{x}(A))}{\bm{q}} =  \M^{\eta}_{\Se}(\ell_x(A), \bm{q}). \label{50:e} \end{align}
\end{claim}

\emph{Let $\bm{q}\in \Ri \Delta_k$}. From Proposition \ref{2:}, the Shannon entropy is such that $\Se^{\star}$ is differentiable on $\mathbb{R}^k$, and thus it follows from Lemma \ref{15:} (\eqref{9:e}-\eqref{10:e}) that for any $\bm{d}\in [0,+\infty[^k$  
\begin{align}
\M_{\Se}(\bm{d}, \bm{q})  ={\Se}^{\star}(\nabla {\Se}^{}(\bm{q})) - {\Se}^{\star}(\nabla {\Se}^{}(\bm{q}) - \bm{d}). \label{122:e}
\end{align}

By definition of $\Se$, $\nabla \Se(\bm{q}) = \log \bm{q} + \bm{1}_k$, and due to Proposition \ref{2:}, ${\Se}^\star(\bm{z}) = \log \inner{\exp{\bm{z}}}{\bm{1}_{k}}, \bm{z} \in \mathbb{R}^k$. Therefore, \begin{align}\nabla \Se(\bm{q}) - \eta \bm{d} = \log (\exp(-\eta \bm{d}) \odot \bm{q} ) + \bm{1}_k. \label{extr}\end{align}
On the other hand, from \cite{Reid2015} we also have \begin{align} 
		 \M^{\eta}_{\Se}(\bm{d}, \bm{q}) & = \eta^{-1} \M_{\Se}(\eta \bm{d}, \bm{q}),\; \eta>0. \label{221:e}
\end{align}
 Combining \eqref{122:e}-\eqref{221:e}, yields 
 \begin{align} \;- \eta^{-1} \log \Inner{\exp(- \eta  \bm{d}}{\bm{q}}& =  \M^{\eta}_{\Se}( \bm{d}, \bm{q}).  \label{51:e}\end{align} 
 \emph{Suppose now that $\bm{q} \in \Ri \Delta_{\cI}$ for $\cI \subseteq [k]$ such that $|\cI|> 1$}. By repeating the argument above for $\Se_{\cI} \coloneqq \Se \circ \Pi_{\cI}^{\mathsf{T}}$, we get
  \begin{align}\forall \bm{d} \in [0,+\infty[^n, \;\M^{\eta}_{\Se_{\cI}}(\Pi_{\cI}^{}\bm{d}, \Pi_{\cI}^{}\bm{q}) &= -\eta^{-1} \log \inner{\exp({-\eta \Pi_{\cI}^{}\bm{d}})}{\Pi_{\cI}^{}\bm{q}}, \nonumber \\ &= -\eta^{-1} \log \inner{\exp({-\eta \bm{d}})}{\bm{q}}. \label{51:e}\end{align}

Fix $x\in[n]$ and let $\hat{\bm{d}}\coloneqq \ell_x(A) \in [0,+\infty]^k$. Let $(\hbm{d}_{m}) \subset [0,+\infty[^k$ be any sequence converging to $ \hat{\bm{d}}$. Lemma \ref{30:}, $\M^{\eta}_{\Se}(\hbm{d}_m, \bm{q}) \stackrel{m\to \infty}{\to}\M^{\eta}_{\Se}( \hat{\bm{d}},\bm{q})$. Using this with \eqref{51:e} gives
\begin{align}
- \eta^{-1} \log \Inner{\exp(- \eta  \ell_{x}(A))}{\bm{q}} &= \lim_{m \to \infty} - \eta^{-1} \log \inner{\exp(- \eta  \hat{\bm{d}}_m)}{\bm{q}},\nonumber \\
&= \lim_{m \to \infty} \M^{\eta}_{\Se}(\hbm{d}_m, \bm{q}), \nonumber  \\
& = \M^{\eta}_{\Se}( \hat{\bm{d}},\bm{q}) = \M^{\eta}_{\Se}( \ell_x(A),\bm{q}). \label{52:e}
\end{align}

\emph{It remains to check the case where $\bm{q}$ is a vertex}; Without loss of generality assume that $\bm{q}=\bm{e}_1$ and let $\bm{\mu}\in \Delta_k \setminus \{ \bm{e}_1\}$. Then there exists $\cI_* \subset [k]$, such that $(\bm{e}_1, \bm{\mu}) \in (\Rbd \Delta_{\cI_*}) \times (\Ri \Delta_{\cI_*})$ and by Lemma \ref{29:}, $\Se'(\bm{e}_1; \bm{\mu}- \bm{e}_1) =-\infty$. Therefore, $\forall \bm{q} \in \Delta_k \setminus \{\bm{e}_1\}, D_{\Se_{\eta}}(\bm{q}, \bm{e}_1) = +\infty$, which implies \begin{align}
\forall x\in [n], \M^{\eta}_{\Se} (\ell_x(A), \bm{e}_1)& = \inf_{\bm{q} \in \Delta_k} \inner{\bm{q}}{\ell_x(A)} +D_{\Se_{\eta}}(\bm{q}, \bm{e}_1),\nonumber\\ &=  \inner{\bm{e}_1}{\ell_x(A)} +D_{\Se_{\eta}}(\bm{e}_1, \bm{e}_1), \nonumber \\&
=  \inner{\bm{e}_1}{\ell_x(A)}, \nonumber \\ & = \ell_x(\bm{a}_1) = - \eta^{-1} \log \Inner{\exp(- \eta  \ell_{x}(A))}{\bm{e}_1}.\label{ext:e} \end{align}
Combining \eqref{ext:e} and \eqref{52:e} proves the claim in \eqref{50:e}. The desired equivalence follows trivially from the definitions of $\eta$-mixability and $(\eta, \Se)$-mixability.
 
\end{proof}

\subsection{Proof of Theorem \ref{17:}}
\label{a17:}
We need the following lemma to show Theorem \ref{17:}.
\begin{lemma}
\label{34:}
Let $\Phi$ be as in Theorem \ref{17:}. Then $\eta_{\ell} \Phi - \Se$ is convex on $\Delta_k$ only if $\Phi$ satisfies \eqref{8:e}. 
\end{lemma}
\begin{proof}
Let $\hat{\bm{q}} \in \Rbd \Delta_k$. Suppose that there exists $\bm{q} \in \Ri \Delta_k$ such that $\Phi'(\hat{\bm{q}}; \bm{q} - \hat{\bm{q}})>-\infty$. Since $\Phi$ is convex, it must have non-decreasing slopes; in particular, it holds that $\Phi'(\hat{\bm{q}}; \bm{q} - \hat{\bm{q}}) \leq \Phi(\bm{q}) - \Phi(\hat{\bm{q}})$. Therefore, since $\Phi$ is finite on $\Delta_k$ (by definition of an entropy), we have $\Phi'(\hat{\bm{q}}; \bm{q} - \hat{\bm{q}}) < +\infty$. Since by assumption $\eta_{\ell} \Phi - \Se $ is convex and finite on the simplex, we can use the same argument to show that $[\eta_{\ell}\Phi - \Se ]'(\hat{\bm{q}}; \bm{q} - \hat{\bm{q}})= \eta_{\ell} \Phi'(\hat{\bm{q}}; \bm{q} - \hat{\bm{q}}) - \Se'(\hat{\bm{q}}; \bm{q} - \hat{\bm{q}})< +\infty$.  This is a contradiction since $\Se'(\hat{\bm{q}}; \bm{q} - \hat{\bm{q}}) = -\infty$ (Lemma \ref{29:}). Therefore, it must hold that $\Phi'(\hat{\bm{q}}; \bm{q} - \hat{\bm{q}}) = -\infty$. 

Suppose now that $(\hbm{q}, \bm{q}) \in (\Rbd \Delta_{\cI}) \times (\Ri \Delta_{\cI})$ for $\cI \subseteq [k]$, with $|\cI|>1$. Let
$\Phi_{\cI} \coloneqq \Phi \circ \Pi_{\cI}^{\mathsf{T}}$ and $\Se_{\cI} \coloneqq \Se \circ \Pi_{\cI}^{\mathsf{T}}$. Since $\eta_{\ell} \Phi - \Se$ is convex on $\Delta_k$ and $\Pi_{\cI}^{}$ is a linear function, $\eta_{\ell} \Phi_{\cI} - \Se_{\cI}$ is convex on $\Delta_{|\cI|}$. Repeating the steps above for $\Phi$ and $\Se$ substituted by $\Phi_{\cI}$ and $\Se_{\cI}$, respectively, we get that $(\Phi_{\cI})'(\Pi_{\cI}^{} \hbm{q}; \Pi_{\cI}^{} \bm{q} - \Pi_{\cI}^{} \hbm{q}) =-\infty$. Since $(\Phi_{\cI})'(\Pi_{\cI}^{} \hbm{q}; \Pi_{\cI}^{} \bm{q} - \Pi_{\cI}^{} \hbm{q}) = \Phi'( \hbm{q}; \bm{q} - \hbm{q})$ the proof is completed.
\end{proof}

\textbf{Theorem \ref{17:}}\emph{
Let $\eta>0$, $\ell \colon \mathcal{A} \rightarrow [0, +\infty]^n$ a $\eta$-mixable loss, and $\Phi \colon \mathbb{R}^k \rightarrow \mathbb{R} \cup \{+ \infty\}$ an entropy. If $\eta \Phi - \Se$ is convex on $\Delta_k$, then $\ell$ is $\Phi$-mixable.
}

\begin{proof}
	Assume $\eta_{\ell} \Phi - \Se$ is convex on $\Delta_k$. For this to hold, it is necessary that $\eta_{\ell}>0$ since $-\Se$ is strictly concave. Let $\eta \coloneqq \eta_{\ell}$ and $\Se_{\eta} \coloneqq \eta^{-1}\Se$. Then $\tilde{\Se}_{\eta} = \eta^{-1} \tilde{\Se}$ and $\tilde{\Phi} - \tilde{\Se}_{\eta} = (\Phi - \Se_{\eta}) \circ \amalg_k$ is convex on $ \tilde{\Delta}_k$, since $\Phi - \Se_{\eta}$ is convex on $\Delta_k$ and $\amalg_k$ is affine.

 Let $x\in [n], A \coloneqq [\bm{a}_{\theta}]_{\theta \in [k]}$, and $\bm{q} \in \Delta_k$. Suppose that $\bm{q} \in \Ri \Delta_k$ and let $\bm{s}_{\bm{q}}^* \in \partial \tilde{\Phi} (\tbm{q})$ be as in Proposition \ref{15:}. Note that if $\ell_x(\bm{a}_{\theta}) =+\infty, \forall \theta \in [k]$, then the $\Phi$-mixability condition \eqref{6:e} is trivially satisfied. Suppose, without loss of generality, that $\ell_x(\bm{a}_k) <+\infty$. Let $(\bm{d}_m) \subset [0,+\infty[^k$ be any sequence such that $\bm{d}_m \stackrel{m\to \infty}{\to} \bm{d} \coloneqq \ell_x(A) \in [0,+\infty]^k$. From Lemmas \ref{29:} and \ref{30:}, $\M_{\Psi}(\bm{d}_m, \bm{q}) \stackrel{m\to \infty}{\to} \M_{\Psi}(\bm{d}, \bm{q})$ for $\Psi \in \{\Phi, \Se_{\eta}\}$.

	Let $\tilde{\Upsilon}_{\bm{q}}: \mathbb{R}^{k-1} \rightarrow \mathbb{R} \cup \{+ \infty\}$ be defined by 
	\begin{align*}
	\tilde{\Upsilon}_{\bm{q}}(\tbm{\mu}) &\coloneqq \tilde{\Se}_{\eta}(\tbm{\mu}) + \inner{\tbm{\mu}}{\bm{s}_{\bm{q}}^* -\nabla \tilde{\Se}_{\eta}(\tbm{q})} - \tilde{\Phi}^{\ast}( \bm{s}_{\bm{q}}^*) +\tilde{\Se}_{\eta}^*(\nabla \tilde{\Se}_{\eta}(\tbm{q})),
\intertext{and it's Fenchel dual follows from Proposition \ref{1:} (i+ii):}
	\tilde{\Upsilon}_{\bm{q}}^{\ast}(\bm{v}) & = \tilde{\Se}_{\eta}^{\ast}(\bm{v} -\bm{s}_{\bm{q}}^* + \nabla \tilde{\Se}_{\eta}(\tbm{q}) ) + \tilde{\Phi}^{\ast}(\bm{s}_{\bm{q}}^*) -\tilde{\Se}_{\eta}^*(\nabla \tilde{\Se}_{\eta}(\tbm{q})),
	\end{align*}
After substituting $\bm{v}$ by $\bm{s}_{\bm{q}}^*  - J^{\mathsf{T}}_k\bm{d}$ in the expression of $\tilde{\Upsilon}_{\bm{q}}^{\ast}$ and rearranging, we get 
	\begin{align}
	\tilde{\Se}_{\eta}^{\ast}(\nabla \tilde{\Se}_{\eta}(\tbm{q})) -\tilde{\Se}_{\eta}^{\ast}(\nabla \tilde{\Se}_{\eta}(\tbm{q}) - J^{\mathsf{T}}_k\bm{d}_m) 
	& =  \tilde{\Phi}^{\ast}(\bm{s}_{\bm{q}}^*) - \tilde{\Upsilon}_{\bm{q}}^{\ast}(\bm{s}_{\bm{q}}^* - J^{\mathsf{T}}_k \bm{d}_m). \label{53:e}
	\end{align}
	
	Since $\bm{s}^*_{\bm{q}} \in \partial \tilde{\Phi}(\tbm{q})$ and $\tilde{\Phi}$ is a closed convex function, combining Proposition \ref{1:}-(iv) and the fact that $\tilde{\Phi}^{**}=\tilde{\Phi}$ \citep[Cor. E.1.3.6]{\Hi} yields $\inner{\tbm{q}}{\bm{s}_{\bm{q}}^* } - \tilde{\Phi}^{\ast}(\bm{s}_{\bm{q}}^*)=\tilde{\Phi}(\tbm{q})$. Thus, after substituting $\tbm{\mu}$ by $\tbm{q}$ in the expression of $\tilde{\Upsilon}_{\bm{q}}$, we get \begin{align}\tilde{\Phi}(\tbm{q}) = \tilde{\Upsilon}_{\bm{q}}(\tbm{q}).\label{54:e}\end{align} On the other hand, $\tilde{\Phi} - \tilde{\Upsilon}_{\bm{q}}$ is convex on $\tilde{\Delta}_k$, since $\tilde{\Upsilon}_{\bm{q}}$ is equal to $\tilde{\Se}_{\eta}$ plus an affine function. Thus, $\partial [\tilde{\Phi} - \tilde{\Upsilon}_{\bm{q}}](\tbm{q}) + \partial  \tilde{\Upsilon}_{\bm{q}}(\tbm{q}) =  \partial \tilde{\Phi}(\tbm{q})$, since $\tilde{\Phi}$ and $\tilde{\Upsilon}_{\bm{q}}$ are both convex (\ibid, Thm. D.4.1.1). Since $\tilde{\Upsilon}_{\bm{q}}$ is differentiable at $\tbm{q}$, we have $\partial \tilde{\Upsilon}_{\bm{q}}(\tbm{q}) = \{\nabla \tilde{\Upsilon}_{\bm{q}}(\tbm{q}) \} = \{ \bm{s}_{\bm{q}}^* \}$. Furthermore, since $\bm{s}_{\bm{q}}^* \in  \partial \tilde{\Phi}(\tbm{q})$, then $\bm{0}_{\tilde{k}} \in \partial \tilde{\Phi}(\bm{q}) - \partial \tilde{\Upsilon}_{\bm{q}}(\tbm{q})=\partial [\tilde{\Phi} - \tilde{\Upsilon}_{\bm{q}}](\tbm{q})$. Hence, $\tilde{\Phi} - \tilde{\Upsilon}_{\bm{q}}$ attains a minimum at $\tbm{q}$ (\ibid, Thm. D.2.2.1). Due to this and \eqref{54:e}, $\tilde{\Phi} \geq \tilde{\Upsilon}_{\bm{q}}$, which implies that $\tilde{\Phi}^* \leq  \tilde{\Upsilon}_{\bm{q}}^*$ (Proposition \ref{1:}-(iii)). Using this in \eqref{53:e} gives for all $m\in \mathbb{N}$
	\begin{eqnarray}
	 \tilde{\Se}_{\eta}^{\ast}(\nabla \tilde{\Se}_{\eta}(\tbm{q})) -\tilde{\Se}_{\eta}^{\ast}(\nabla \tilde{\Se}_{\eta}(\tbm{q}) - J^{\mathsf{T}}_k \bm{d}_m) & \leq&  \tilde{\Phi}^{\ast}(\bm{s}_{\bm{q}}^*)-\tilde{\Phi}^{\ast}(\bm{s}_{\bm{q}}^* - J^{\mathsf{T}}_k \bm{d}_m), \nonumber \\
	\implies \M^{\eta}_{\Se} (\bm{d}_m, \bm{q}) & \leq & \M_{\Phi} (\bm{d}_m, \bm{q}), \nonumber 
	\end{eqnarray}
	where the implication is obtained by adding $[\bm{d}_m]_k$ on both sides of the first inequality and using Proposition \ref{15:}. 

Suppose now that $\bm{q}\in \Ri \Delta_{\cI}$, with $|\cI|>1$, and let $\Phi_{\cI} \coloneqq \Phi \circ \Pi_{\cI}^{\mathsf{T}}$ and $\Se_{\cI} \coloneqq \Se \circ \Pi_{\cI}^{\mathsf{T}}$. Note that since $\eta_{\ell} \Phi - \Se$ is convex on $\Delta_k$ and $\Pi_{\cI}^{}$ is a linear function, $\eta_{\ell} \Phi_{\cI} - \Se_{\cI}$ is convex on $\Delta_{|\cI|}$. Repeating the steps above for $\Phi$, $\Se$, $\bm{q}$, and $A$ substituted by $\Phi_{\cI}$, $\Se_{\cI}$, $\Pi_{\cI}^{} \bm{q}$, and $A \Pi_{\cI}^{\mathsf{T}}$, respectively, yields 
\begin{eqnarray}
\M^{\eta}_{\Se_{\cI}} (\Pi_{\cI}^{}\bm{d}_m, \Pi_{\cI}^{}\bm{q}) & \leq & \M_{\Phi_{\cI}} (\Pi_{\cI}^{}\bm{d}_m, \Pi_{\cI}^{}\bm{q}), \nonumber \\
\implies \M^{\eta}_{\Se} (\bm{d}_m, \bm{q}) & \leq & \M_{\Phi} (\bm{d}_m, \bm{q}), \nonumber \\
\implies \M^{\eta}_{\Se} (\ell_x(A), \bm{q}) & \leq & \M_{\Phi} (\ell_x(A), \bm{q}), \label{55:e}
\end{eqnarray}
where the first implication follows from Lemma \ref{13:}, since $\Se_{\eta}$ and $\Phi$ both satisfy \eqref{8:e} (see Lemmas \ref{29:} and \ref{34:}), and \eqref{55:e} is obtained by passage to the limit $m\to \infty$. Since $\eta = \eta_{\ell}>0$, $\ell$ is $\eta$-mixable, which implies that $\ell$ is $\Se_{\eta}$-mixable (Theorem \ref{16:}). Therefore, there exists $\bm{a}_*\in \mathcal{A}$, such that \begin{align}\label{ext2:e}\ell_x(\bm{a}_*) \leq \M^{\eta}_{\Se} (\ell_x(A), \bm{q})  \leq  \M_{\Phi} (\ell_x(A), \bm{q}). \end{align} 

To complete the proof (that is, to show that $\ell$ is $\Phi$-mixable), it remains to consider the case where $\bm{q}$ is a vertex of $\Delta_k$.  Without loss of generality assume that $\bm{q}=\bm{e}_1$ and let $\bm{\mu}\in \Delta_k \setminus \{ \bm{e}_1\}$. Thus, there exists $\cI_* \subseteq [k]$, with $|\cI_*|>1$, such that $(\bm{e}_1, \bm{\mu}) \in (\Rbd \Delta_{\cI_*}) \times (\Ri \Delta_{\cI_*})$, and Lemma \ref{34:} implies that $\Phi'(\bm{e}_1; \bm{\mu}- \bm{e}_1) =-\infty$. Therefore, $\forall \bm{q} \in \Delta_k \setminus \{\bm{e}_1\}, D_{\Phi}(\bm{q}, \bm{e}_1) = +\infty$, which implies \begin{align}
\forall x\in [n], \M_{\Phi} (\ell_x(A), \bm{e}_1)& = \inf_{\bm{q} \in \Delta_k} \inner{\bm{q}}{\ell_x(A)} +D_{\Phi}(\bm{q}, \bm{e}_1), \nonumber\\ &=  \inner{\bm{e}_1}{\ell_x(A)} +D_{\Phi}(\bm{e}_1, \bm{e}_1)
=  \inner{\bm{e}_1}{\ell_x(A)},\nonumber \\ & = \ell_x(\bm{a}_1). \label{ext3:e} \end{align}
The $\Phi$-mixability condition \eqref{6:e} is trivially satisfied in this case. Combining \eqref{ext2:e} and \eqref{ext3:e} shows that $\ell$ is $\Phi$-mixable.
\end{proof}

		\subsection{Proof of Theorem \ref{18:}}
		\label{a18:}
	The following Lemma gives necessary regularity conditions on the entropy $\Phi$ under the assumptions of Theorem \ref{18:}.
	\begin{lemma}
	\label{35:}
Let $\Phi$ and $\ell$ be as in Theorem \ref{18:}. Then the following holds

\begin{enumerate}[label=(\roman*)]
\item $\tilde{\Phi}$ is strictly concave on $\Int \tilde{\Delta}_k$.
\item $\tilde{\Phi}^*$ is be continuously differentiable on $\mathbb{R}^{k-1}$. 
\item $\tilde{\Phi}^*$ is twice differentiable on $\mathbb{R}^{k-1}$ and $\forall \tbm{q} \in \Int \tilde{\Delta}_k, \mathsf{H} \tilde{\Phi}^{\ast} (\nabla \tilde{\Phi}(\tbm{q})) = (\mathsf{H} \tilde{\Phi}(\tbm{q}))^{-1}$.
\item For the Shannon entropy, we have $(\mathsf{H} \tilde{\Se}(\tbm{q}))^{-1} =  \mathsf{H} \tilde{\Se}^*(\nabla \tilde{\Se}(\tbm{q})) = \Diag \tbm{q} - \tbm{q} \tbm{q}^{\mathsf{T}}.$ 
\end{enumerate}
	\end{lemma}
\begin{proof}
			Since $\ell$ is $\Phi$-mixable and $\br_{\ell}$ is twice differentiable on $]0,+\infty[^n$, $\tilde{\Phi}^*$ is continously differentiable on $\mathbb{R}^{n-1}$ (Proposition \ref{14:}). Therefore, $\tilde{\Phi}$ is strictly convex on $\Ri \Delta_k$ \citep[Thm. E.4.1.2]{\Hi}. 
			
The differentiability of $\tilde{\Phi}$ and $\tilde{\Phi}^*$ implies $\nabla \tilde{\Phi}^* (\nabla \tilde{\Phi}(\tbm{q})) = \tbm{q}$ (\ibid). Since $\tilde{\Phi}$ is twice differentiable on $\Int \tilde{\Delta}_k$ (by assumption), the latter equation implies that $\tilde{\Phi}^*$ is twice differentiable on $\nabla \tilde{\Phi}(\Int \tilde{\Delta}_k)$. Using the chain rule, we get $\mathsf{H} \tilde{\Phi}^{\ast}(\nabla \tilde{\Phi}(\bm{u})) \mathsf{H} \tilde{\Phi} (\bm{u})= I_{\tilde{k}}$. Multiplying both sides of the equation by $(\mathsf{H} \tilde{\Phi} (\bm{u}) )^{-1}$ from the right gives the expression in (iii). Note that $\mathsf{H} \tilde{\Phi}(\cdot)$ is in fact invertible on $\Int \tilde{\Delta}_k$ since $\tilde{\Phi}$ is strictly convex on $\Int \tilde{\Delta}_k$. It remains to show that $\nabla \tilde{\Phi}(\Int \tilde{\Delta}_k)= \mathbb{R}^{k-1}$. This set equality follows from 1) $[\tbm{q} \in \partial \tilde{\Phi}^*(\bm{s}) \iff \bm{s} \in \partial \tilde{\Phi}(\tbm{q})]$ (\ibid, Cor. E.1.4.4); 2) $\Dom \tilde{\Phi}^* = \mathbb{R}^{k-1}$; and 3) $\forall \tbm{q}\in \Bd \tilde{\Delta}_k, \partial \tilde{\Phi}(\tbm{q}) = \varnothing$ (Lemma \ref{13:}).

For the Shannon entropy, we have $\tilde{\Se}^*(\bm{v}) = \log (\inner{\exp({\bm{v}})}{\bm{1}_{\tbm{k}}}+ 1)$ (Proposition \ref{2:}) and $\nabla \tilde{\Se}(\tbm{q}) = \log \frac{\tbm{q}}{q_k}$, for $(\bm{v}, \tbm{q}) \in \mathbb{R}^{k-1}\times \tilde{\Delta}_k$. Thus $(\mathsf{H} \tilde{\Se}(\tbm{q}))^{-1} =  \mathsf{H} \tilde{\Se}^*(\nabla \tilde{\Se}(\tbm{q})) = \Diag \tbm{q} - \tbm{q} \tbm{q}^{\mathsf{T}}.$
\end{proof}

		To show Theorem \ref{18:}, we analyze a particular parameterized curve defined in the next lemma.  
		
		\begin{lemma}
			\label{36:}
			Let $\ell \colon \Delta_n\rightarrow [0,+\infty]^n$ be a proper loss whose Bayes risk $\br_{\ell}$ is twice differentiable on $]0,+\infty[^n$, and let $\Phi$ be an entropy such that $\tilde{\Phi}$ and $\tilde{\Phi}^*$ are twice differentiable on $\Int \tilde{\Delta}_k$ and $\mathbb{R}^{k-1}$, respectively. For $(\tbm{p},\tbm{q}, V) \in \Int\tilde{\Delta}_n \times \Int  \tilde{\Delta}_k \times \mathbb{R}^{\tilde{n} \times \tilde{k}}$, let $\beta:\mathbb{R} \rightarrow \mathbb{R}^{n}$ be the curve defined by
			\begin{align}
			\forall x \in [n], \quad \beta_x(t) = \tilde{\ell}_x(\tbm{p}) + \tilde{\Phi}^{\ast}(\nabla \tilde{\Phi}(\tbm{q})) - \tilde{\Phi}^{\ast}(\nabla \tilde{\Phi}(\tbm{q}) - J^{\mathsf{T}}_k\tilde{\ell}_x(\tilde{P}^t) ), \label{56:e}
			\end{align}
			where $\tilde{P}^t = [ \tbm{p} \bm{1}^{\mathsf{T}}_{\tilde{k}}  + t V,  \; \tbm{p} ] \in \mathbb{R}^{\tilde{n} \times k}$ and $t \in \{s\in \mathbb{R}:\forall j \in [\tilde{k}],\;  \tbm{p} + s V_{\bcdot,j} \in \Int \tilde{\Delta}_n \}$. Then  
			\begin{align}
			\beta(0)&= \tilde{\ell}(\tbm{p}),\nonumber \\
			\dot{\beta}(0)& =\mathsf{D} \tilde{\ell}(\tbm{p})V \tbm{q}, \nonumber \\
			\Der{ \Inner{\bm{p}}{\dot{\beta}(t)}}  = - \sum_{j=1}^{k-1} q_j   V^{\mathsf{T}}_{\bcdot,j} \mathsf{H} \tbr_{\ell}(\tbm{p}) &V_{\bcdot,j}  - \tr (\diag{\bm{p}}  \mathsf{D} \tilde{\ell}(\tbm{p}) V (\mathsf{H} \tilde{\Phi}(\tbm{q}))^{-1} (\mathsf{D} \tilde{\ell}(\tbm{p}) V)^{\mathsf{T}} ). \label{57:e}
			\end{align}
		\end{lemma}

	\begin{proof}
Since $\tilde{P}^t = [ \tbm{p} \bm{1}^{\mathsf{T}}_{\tilde{k}}  + t V, \;  \tbm{p} ] \in \mathbb{R}^{\tilde{n} \times k}$, $\tilde{P}^0=  \tbm{p} \bm{1}^{\mathsf{T}}_{k} $ and $\tilde{\ell}_x(\tilde{P}^0) = \tilde{\ell}_x(\tbm{p}) \bm{1}_{k}$. As a result, $J^{\mathsf{T}}_k \tilde{\ell}_x(\tilde{P}^0) = \bm{0}_{\tilde{k}}$, and thus $\beta_x(0) = \tilde{\ell}_x(\tbm{p}) + \tilde{\Phi}^*(\nabla \tilde{\Phi}(\tbm{q})) - \tilde{\Phi}^*(\nabla \tilde{\Phi}(\tbm{q}) - \bm{0}_{\tilde{k}})= \tilde{\ell}_x(\tbm{p})$. This shows that $\beta(0)=\tilde{\ell}(\tbm{p}). $
			Let $\gamma_x(t) \coloneqq \nabla \tilde{\Phi}(\tbm{q}) - J^{\mathsf{T}}_k \tilde{\ell}_x(\tilde{P}^t)$. For $j \in [k-1]$,
			\begin{align*}
			\frac{d}{dt}[\gamma_x(t)]_j &= \frac{d}{dt} \left([\nabla \tilde{\Phi}(\tbm{q})]_j - [J^{\mathsf{T}}_k \tilde{\ell}_x (\tilde{P}^{t})]_j\right), \\
			&= - \frac{d}{dt} \left( \tilde{\ell}_x (\tilde{P}^{t}_{\bcdot, j}) - \tilde{\ell}_x(\tilde{P}^{t}_{\bcdot, k})\right),\\
			& = - \frac{d}{dt} \left( \tilde{\ell}_x (\tbm{p} + t V_{\bcdot, j} ) - \tilde{\ell}_x (\tbm{p})  \right), \quad \quad \left(\mbox{since } \frac{d}{dt} \ell_x(\tilde{P}^t_{\bcdot, k})= \frac{d}{dt} \tilde{\ell}_x(\tbm{p})  =0\right) \\
			&   = -  \mathsf{D} \tilde{\ell}_x(\tilde{P}^{t}_{\bcdot,j}) V_{\bcdot,j}.
			\end{align*}
			
			From the definition of $\tilde{P}^t$, $\tilde{P}^0_{\bcdot, j}=\tbm{p}$, $\forall j\in [\tilde{k}]$, and therefore, $\dot{\gamma}_x(0)=- (\mathsf{D} \tilde{\ell}_x(\tbm{p}) V)^{\mathsf{T}}$. By differentiating $\beta_{x}$ in \eqref{56:e} and using the chain rule, $\dot{\beta}_x(t) = -(\dot{\gamma}_x(t))^{\mathsf{T}} \nabla \tilde{\Phi}^{\ast}(\gamma_x(t))$. By setting $t=0$ , $\dot{\beta}_x(0) = -(\dot{\gamma}_x(0))^{\mathsf{T}} \nabla \tilde{\Phi}^{\ast}(\nabla \tilde{\Phi}(\tbm{q}))  = \mathsf{D} \tilde{\ell}_x(\tbm{p}) V \tbm{q}$. Thus, $\dot{\beta}(0) = \mathsf{D} \tilde{\ell}(\tbm{p}) V \tbm{q}$. Furthermore,
			\begin{align*}
			\Der{\Inner{\bm{p}}{\dot{\beta}(t)}}  & =  \Der{\sum_{x =1}^{n} p_x \left(\sum_{j=1}^{k-1}	 \mathsf{D} \tilde{\ell}_x(\tilde{P}^t_{\bcdot, j}) V_{\bcdot, j} [\nabla \tilde{\Phi}^{\ast}(\gamma_x(t))]_j   \right)},   \\
			& =\sum_{j=1}^{k-1} \Der{ \left(\sum_{x =1}^{n} p_x	 \mathsf{D} \tilde{\ell}_x(\tilde{P}^t_{\bcdot, j}) V_{\bcdot, j} [\nabla \tilde{\Phi}^{\ast}(\gamma_x(t))]_j \right)}, \\	
			& = \sum_{j=1}^{k-1} \left( \Der{\Inner{\bm{p}}{\mathsf{D} \tilde{\ell}(\tilde{P}^t_{\bcdot, j}) V_{\bcdot, j} q_j}} + \sum_{x =1}^{n} p_x	 \mathsf{D} \tilde{\ell}_x(\tbm{p}) V_{\bcdot, j} \Der{ [\nabla \tilde{\Phi}^{\ast}(\gamma_x(t))]_j}  \right), \\
			&= - \sum_{j=1}^{k-1} q_j   V^{\mathsf{T}}_{\bcdot,j} \mathsf{H} \tbr_{\ell}(\tbm{p}) V_{\bcdot,j}  - \sum_{x =1}^{n} \sum_{\substack{i=1\\j=1}}^{k-1} p_x	 \mathsf{D} \tilde{\ell}_x(\tbm{p}) V_{\bcdot, j}  [\mathsf{H} \tilde{\Phi}^{\ast}(\nabla \tilde{\Phi}(\tbm{q}))]_{j, i} \mathsf{D} \tilde{\ell}_x(\tbm{p}) V_{\bcdot, i},  \\
			& =- \sum_{j=1}^{k-1} q_j   V^{\mathsf{T}}_{\bcdot,j} \mathsf{H} \tbr_{\ell}(\tbm{p}) V_{\bcdot,j}  - \tr (\diag{\bm{p}}  \mathsf{D} \tilde{\ell}(\tbm{p}) V \mathsf{H} \tilde{\Phi}^{\ast}(\nabla \tilde{\Phi}(\tbm{q})) (\mathsf{D} \tilde{\ell}(\tbm{p}) V)^{\mathsf{T}}), \\ 
			&= - \sum_{j=1}^{k-1} q_j   V^{\mathsf{T}}_{\bcdot,j} \mathsf{H} \tbr_{\ell}(\tbm{p}) V_{\bcdot,j}  - \tr (\diag{\bm{p}}  \mathsf{D} \tilde{\ell}(\tbm{p}) V (\mathsf{H} \tilde{\Phi}(\tbm{q}))^{-1} (\mathsf{D} \tilde{\ell}(\tbm{p}) V)^{\mathsf{T}}),
			\end{align*}
			where in the third equality we used Lemma \ref{25:}, in the fourth equality we used Lemma \ref{28:}, and in the sixth equality we used Lemma \ref{35:}-(iii).
			
		\end{proof}

		In next lemma, we state a necessary condition for $\Phi$-mixability in terms of the parameterized curve $\beta$ defined in Lemma \ref{36:}.
		
		\begin{lemma}
			\label{37:}
			Let $\ell$, $\Phi$, and $\beta$ be as in Lemma \ref{36:}. If $\exists (\tbm{p},\tbm{q}, V) \in \Int\tilde{\Delta}_n \times \Int  \tilde{\Delta}_k \times \mathbb{R}^{\tilde{n} \times \tilde{k}}$ such that the curve $\gamma(t) \coloneqq \tilde{\ell}(\tbm{p} + t V \tbm{q})$ satisfies $\left.\frac{d}{dt}  \inner{\bm{p}}{\dot{\beta}(t) - \dot{\gamma}(t)}\right|_{t=0} < 0$, then $\ell$ is not $\Phi-$mixable. In particular, $\exists P \in \Ri \Delta^k_n$, such that $[\M_{\Phi}(\ell_x(P), \bm{q})]^{\mathsf{T}}_{x\in [n]}$ lies outside $\sps$.
		\end{lemma}
			\begin{proof}
			First note that for any triplet $(\tbm{p},\tbm{q}, V) \in \Int\tilde{\Delta}_n \times \Int  \tilde{\Delta}_k \times \mathbb{R}^{\tilde{n} \times \tilde{k}}$, the map $t\mapsto \Inner{\bm{p}}{\dot{\beta}(t) - \dot{\gamma}(t)}$ is differentiable at $0$. This follows from Lemmas \ref{25:} and \ref{36:}.
			Let $r(t) \coloneqq \amalg_n(\tbm{p} +t V \tbm{q}) $ and $\delta(t) \coloneqq \inner{r(t)}{\beta(t) - \gamma(t)}$. Then
			\begin{align*}
			\dot{\delta}(t) & = \Inner{r(t)}{\dot{\beta}(t) - \dot{\gamma}(t)} + \Inner{V\tbm{q}}{\beta(t) - \gamma(t)}.
			\end{align*}
			
			Since $t\mapsto \inner{\bm{p}}{\dot{\beta}(t) - \dot{\gamma}(t)}$ is differentiable at 0, it follows from Lemma \ref{25:} that $t \mapsto \dot{\delta}(t)$ is also differentiable at $0$, and thus
			\begin{align}
			\ddot{\delta}(0) & =\left.\frac{d}{dt}  \Inner{r(t)}{\dot{\beta}(t) - \dot{\gamma}(t)}\right|_{t=0} + \Inner{J_n V \tbm{q}}{\dot{\beta}(0) - \dot{\gamma}(0)}, \nonumber \\
			& = \Inner{	 \dot r(0)}{\dot{\beta}(0) - \dot{\gamma}(0)}   +  	\left.\frac{d}{dt}  \Inner{\bm{p}}{\dot{\beta}(t) - \dot{\gamma}(t)}\right|_{t=0}, \label{58:e} \\
			&=	\Inner{J_n V \tbm{q}}{\dot{\beta}(0) - \dot{\gamma}(0)}   +  	\left.\frac{d}{dt} \Inner{\bm{p}}{\dot{\beta}(t) - \dot{\gamma}(t)}\right|_{t=0} , \nonumber \\
			&= \left.\frac{d}{dt} \Inner{\bm{p}}{\dot{\beta}(t) - \dot{\gamma}(t)}  \right|_{t=0}<0, \label{59:e}
			\end{align}
			where \eqref{58:e} and \eqref{59:e} hold because $ \dot{\beta}(0) =\mathsf{D} \tilde{\ell}(\tbm{p})V \tbm{q}= \dot{\gamma}(0)$ (see Lemma \ref{36:}).
 According to Taylor's theorem (see e.g. \citep[\S 151]{hardy2008}), there exists $\epsilon>0$ and $h: [-\epsilon, \epsilon] \rightarrow \mathbb{R}$ such that 
			\begin{align}
			\forall |t| \leq \epsilon, \;\delta(t) = \delta(0) + t \dot{\delta}(0) + \frac{t^2}{2} \ddot{\delta}(0) + h(t) t^2, \label{60:e}
			\end{align}
and $\lim_{t\to 0} h(t) = 0$. From Lemma \ref{36:}, $\beta(0) = \gamma(0)=0$ and $\dot{\beta}(0)=\dot{\gamma}(0)$. Therefore, $\delta(0)=\dot{\delta}(0)=0$ and \eqref{60:e} becomes $\delta(t) = \frac{t^2}{2} \ddot{\delta}(0) + h(t)t^2$. Due to \eqref{59:e} and the fact that $\lim_{t\to 0} h(t) = 0$, we can choose $\epsilon_*>0$ small enough such that $\delta(\epsilon_*) =\frac{\epsilon_*^2}{2} \ddot{\delta}(0) + h(\epsilon_*)\epsilon_*^2 <0$. This means that $\inner{\amalg_n(\tbm{p} + \epsilon_* V \tbm{q})}{\beta(\epsilon_*)} < \inner{\amalg_n(\tbm{p} + \epsilon_* V \tbm{q})}{\tilde{\ell}(\tbm{p} +\epsilon_* V \tbm{q})}=\inner{\amalg_n(\tbm{p} + \epsilon_* V \tbm{q})}{\ell(\amalg_n(\tbm{p} + \epsilon_* V \tbm{q})}$. Therefore, $\beta(\epsilon_*)$ must lie outside the superprediction set. Thus, the mixability condition \eqref{6:e} does not hold for $P^{\epsilon_*}  =\amalg_n [ \tbm{p} \bm{1}^{\mathsf{T}}_{\tilde{k}}  + \epsilon_* V,\;   \tbm{p} ]  \in \Ri \Delta^k_n$. This completes the proof.\end{proof}

\textbf{Theorem \ref{18:}}\emph{
		Let $\ell \colon \mathcal{A} \rightarrow [0,+\infty]^n$ be a loss such that $\br_{\ell}$ is twice differentiable on $]0,+\infty[^n$, and $\Phi \colon \mathbb{R}^k \rightarrow \mathbb{R} \cup \{+ \infty\}$ an entropy such that $\tilde{\Phi} \coloneqq \Phi \circ \amalg_k$ is twice differentiable on $\Int \tilde{\Delta}_k$. Then $\ell$ is $\Phi$-mixable only if $\underline{\eta_{\ell}} \Phi - \Se$ is convex on $\Delta_k$.
}
		\begin{proof}
			We will prove the contrapositive; suppose that $\underline{\eta_{\ell}} \Phi - \Se$ is not convex on $\Delta_k$ and we show that $\ell$ cannot be $\Phi$-mixable. Note first that from Lemma \ref{35:}-(iii), $\tilde{\Phi}^*$ is twice differentiable on $\mathbb{R}^{k-1}$. Thus Lemmas \ref{36:} and \ref{37:} apply. Let $\underline{\ell}$ be a proper support loss of $\ell$ and suppose that $\underline{\eta_{\ell}} \Phi - \Se$ is not convex on $\Delta_k$, This implies that $\underline{\eta_{\ell}} \tilde{\Phi} -\tilde{\Se}$ is not convex on $\Int \tilde{\Delta}_k$, and by Lemma \ref{22:} there exists $\tbm{q}_* \in \Int \tilde{\Delta}_k$, such that $1> \underline{\eta_{\ell}} \lambda_{\min}(\mathsf{H}\tilde{\Phi}(\tbm{q}_*) (\mathsf{H}\tilde{\Se}(\tbm{q}_*))^{-1})$. From this and the definition of $\underline{\eta_{\ell}}$, there exists $\tbm{p}_*  \in \Int \tilde{\Delta}_n$ such that
			\begin{align}
			1 >	\frac{	  \lambda_{\min}(\mathsf{H}\tilde{\Phi}(\tbm{q}_*) (\mathsf{H}\tilde{\Se}(\tbm{q}_*))^{-1})} {  \lambda_{\max}( [\mathsf{H} \tbr_{\log}(\tbm{p}_*)]^{-1} \mathsf{H} \tbr_{\ell}(\tbm{p}_*) )} = \frac{\lambda_{\min}(\mathsf{H}\tilde{\Phi}(\tbm{q}_*) (\diag{\tbm{q}_*} - \tbm{q}_* \tbm{q}_*^{\mathsf{T}} )) }{\lambda_{\max}( [\mathsf{H} \tbr_{\log}(\tbm{p}_*)]^{-1} \mathsf{H} \tbr_{\ell}(\tbm{p}_*) )},   \label{61:e}
			\end{align}
where the equality is due to Lemma \ref{35:}-(iv). For the rest of this proof let $(\tbm{p}, \tbm{q}) =(\tbm{p}^*, \tbm{q}^*)$. By assumption, $\tbr_{\ell}$ twice differentiable and concave on $\Int \tilde{\Delta}_n$, and thus $-\mathsf{H} \tbr_{\ell}(\tbm{p})$ is symmetric positive semi-definite. Therefore, their exists a symmetric positive semi-definite matrix $\Lambda_{\bm{p}}$ such that $\Lambda_{\bm{p}} \Lambda_{\bm{p}} = - \mathsf{H} \tbr_{\ell}(\tbm{p})$. From Lemma \ref{35:}-(i), $\tilde{\Phi}$ is strictly convex on $\Int \tilde{\Delta}_k$, and so there exists a symmetric positive definite matrix $K_{\bm{q}}$ such that $K_{\bm{q}} K_{\bm{q}} = \mathsf{H} \tilde{\Phi}  (\tbm{q})$. Let $\bm{w} \in \mathbb{R}^{n-1}$ be the unit norm eigenvector  of $ [\mathsf{H} \tbr_{\log}(\tbm{p})]^{-1} \mathsf{H} \tbr_{\ell}(\tbm{p}) $ associated with $\lambda^{\ell}_* \coloneqq   \lambda_{\max}( [\mathsf{H}\tbr_{\log}(\tbm{p})]^{-1} \mathsf{H} \tbr_{\ell}(\tbm{p}) )$.
Suppose that $c_{\ell}\coloneqq \bm{w}^{\mathsf{T}} \mathsf{H} \tbr_{\ell}(\tbm{p})\bm{w} =0$. Since $\bm{w}^{\mathsf{T}}\Lambda_{\bm{p}}\Lambda_{\bm{p}} \bm{w}= -c_{\ell}=0$, it follows from the positive semi-definiteness of $\Lambda_{\bm{p}}$ that $\Lambda_{\bm{p}}\bm{w}=\bm{0}_{\tilde{n}}$, and thus $\mathsf{H}\tbr_{\ell}(\tbm{p})\bm{w}= -\Lambda_{\bm{p}}\Lambda_{\bm{p}}\bm{w}=\bm{0}_{\tilde{n}}$. This implies that $\lambda^{\ell}_*=0$, which is not possible due to \eqref{61:e}. Therefore, $\mathsf{H}\tbr_{\ell}(\tbm{p})\bm{w} \neq \bm{0}_{\tilde{n}}$. Furthermore, the negative semi-definiteness of $\mathsf{H}\tbr_{\ell}(\tbm{p})$ implies that
\begin{align}
\label{pos:cl:e}
c_{\ell} = \bm{w}^{\mathsf{T}} \mathsf{H} \tbr_{\ell}(\tbm{p}) \bm{w}  < 0.
\end{align}
Let $\bm{v} \in \mathbb{R}^{k-1}$ be the unit norm eigenvector of $K_{\bm{q}} (\diag{\tbm{q}} - \tbm{q} \tbm{q}^{\mathsf{T}} ) K_{\bm{q}}$ associated with $\lambda^{\Phi}_{*}\coloneqq \lambda_{\min}(K_{\bm{q}} (\diag{\tbm{q}} - \tbm{q} \tbm{q}^{\mathsf{T}} ) K_{\bm{q}}) =  \lambda_{\min}(\mathsf{H} \tilde{\Phi}(\tbm{q}) (\diag{\tbm{q}} - \tbm{q} \tbm{q}^{\mathsf{T}} ))$, where the equality is due to Lemma \ref{31:}. Let $\hat{\bm{v}} \coloneqq K_{\bm{q}} \bm{v}$.  
			
		We will show that for $V= \bm{w} \hat{\bm{v}}^{\mathsf{T}}$, the parametrized curve $\beta$ defined in Lemma \ref{36:} satisfies $\left. \frac{d}{dt} \inner{\bm{p}}{\dot{\beta}(t)  - \dot{\gamma}(t)}\right|_{t=0}<0$, where $\gamma(t) = \tilde{\sell}(\tbm{p}+tV \tbm{q})$. According to Lemma \ref{37:} this would imply that there exists $P\in \Ri \Delta_n^k$, such that $[\M_{\Phi}(\sell_x(P), \bm{q})]^{\mathsf{T}}_{x\in [n]}$ lies outside $\sps$. From Theorem \ref{5:}, we know that there exists $A_* \in \mathcal{A}^k$, such that $\ell_x(A_*)=\sell_x(P), \forall x\in[n]$. Therefore, $[\M_{\Phi}(\ell_x(A_*), \bm{q})]^{\mathsf{T}}_{x\in [n]} = [\M_{\Phi}(\sell_x(P), \bm{q})]^{\mathsf{T}}_{x\in [n]} \notin \sps$, and thus $\ell$ is not $\Phi$-mixable.
			
			From Lemma \ref{36:} (Equation \ref{57:e}) and the fact that $V_{\bcdot,j} = \hat{v}_j \bm{w} $, for $j\in [\tilde{k}]$, we can write 
			\begin{align*}
			\left. \frac{d}{dt}	\Inner{\bm{p}}{\dot{\beta}(t)} \right|_{t=0}  & = -\sum_{j=1}^{k-1} q_j \hat{v}_j^2  \bm{w}^{\mathsf{T}} \mathsf{H} \tbr_{\ell}(\tbm{p}) \bm{w}  - \tr (\diag{\bm{p}}  \mathsf{D} \tilde{\sell}(\tbm{p}) V (\mathsf{H} \tilde{\Phi}(\tbm{q}))^{-1} (\mathsf{D} \tilde{\sell}(\tbm{p}) V)^{\mathsf{T}} ),  \\
			& =  -\Inner{\tbm{q}} {\hat{\bm{v}} \odot \hat{\bm{v}} }   \bm{w}^{\mathsf{T}} \mathsf{H} \tbr_{\ell}(\tbm{p}) \bm{w} - (\hat{\bm{v}}^{\mathsf{T}} (\mathsf{H} \tilde{\Phi}(\bm{q}))^{-1} \hat{\bm{v}}) \inner{\bm{p}}{[\mathsf{D} \tilde{\ell} (\tbm{p}) \bm{w}] \odot [(\mathsf{D} \tilde{\ell} (\tbm{p}) \bm{w}] },
			\end{align*}
			where the second equality is obtained by noting that $1)$ $(\hat{\bm{v}}^{\mathsf{T}} (\mathsf{H} \tilde{\Phi}(\bm{q}))^{-1} \hat{\bm{v}})$ is a scalar quantity and can be factorized out; and $2)$ $\tr (\Diag (\bm{p}) \mathsf{D} \tilde{\ell} (\tbm{p}) \bm{w} (\mathsf{D} \tilde{\ell} (\tbm{p}) \bm{w} )^{\mathsf{T}})  = \inner{\bm{p}}{(\mathsf{D} \tilde{\ell} (\tbm{p}) \bm{w} )\odot (\mathsf{D} \tilde{\ell} (\tbm{p}) \bm{w}) }$. 
			
			On the other hand, from Lemma \ref{28:}, $\left. \frac{d}{dt} \inner{\bm{p}}{\dot{\gamma}(t)} \right|_{t=0}=  -\inner{\tbm{q}}{\hat{\bm{v}}}^2 \bm{w}^{\mathsf{T}} \mathsf{H} \tbr_{\ell}(\tbm{q}) \bm{w}$. 
			Using \eqref{20:e} and the definition of $c_{\ell}$, we get 
			\begin{align}
			\left. \frac{d}{dt} \Inner{\bm{p}}{\dot{\beta}(t)  - \dot{\gamma}(t)}\right|_{t=0}& =  [-\Inner{\tbm{q}}{ \hat{\bm{v}} \odot \hat{\bm{v}}} +\Inner{\tbm{q}}{\hat{\bm{v}}}^2]c_{\ell} +\nonumber \\
& \hspace{2cm}  (\hbm{v}^{\mathsf{T}} (\mathsf{H} \tilde{\Phi}(\bm{q}))^{-1} \hat{\bm{v}}) (\bm{w}^{\mathsf{T}} (\mathsf{H} \tbr_{\ell}(\tbm{p}))  (\mathsf{H} \tbr_{\log}(\tbm{p}))^{-1} \mathsf{H} \tbr_{\ell}(\bm{p}) \bm{w}),\nonumber \\
& = -c_{\ell} [ \Inner{\tbm{q}}{ \hat{\bm{v}} \odot \hat{\bm{v}}} -\Inner{\tbm{q}}{\hat{\bm{v}}}^2 - \lambda^{\ell}_*  (\hbm{v}^{\mathsf{T}} (\mathsf{H} \tilde{\Phi}(\bm{q}))^{-1} \hat{\bm{v}})],\nonumber \\
& = -c_{\ell} [\hat{\bm{v}}^{\mathsf{T}}(\diag{\tbm{q}} - \tbm{q} \tbm{q}^{\mathsf{T}}) \hat{\bm{v}}  - \lambda^{\ell}_*  (\hbm{v}^{\mathsf{T}} (\mathsf{H} \tilde{\Phi}(\bm{q}))^{-1} \hat{\bm{v}})], \nonumber\\
& = -c_{\ell} [\hat{\bm{v}}^{\mathsf{T}}(\diag{\tbm{q}} - \tbm{q} \tbm{q}^{\mathsf{T}}) \hat{\bm{v}}  - \lambda^{\ell}_*  ( \bm{v}^{\mathsf{T}} K_{\bm{q}} (K_{\bm{q}}K_{\bm{q}})^{-1} K_{\bm{q}} \bm{v} )],\nonumber \\
& = -c_{\ell} [\bm{v}^{\mathsf{T}} K_{\bm{q}} (\diag{\tbm{q}} - \tbm{q} \tbm{q}^{\mathsf{T}})K_{\bm{q}}  \bm{v}  - \lambda^{\ell}_*  ],\label{norm:v:e}  \\
& = -c_{\ell} [\lambda_*^{\Phi}  - \lambda^{\ell}_*],\nonumber\\
			& =- c_{\ell}[ \lambda_{\min} (\mathsf{H} \tilde{\Phi}(\bm{q})(\diag{\tbm{q}} - \tbm{q} \tbm{q}^{\mathsf{T}}))  -\lambda_{\max}( \mathsf{H} \tbr_{\ell}(\tbm{p}) (\mathsf{H} \tbr_{\log}(\tbm{p}))^{-1})],\nonumber
			\end{align}
			where in \eqref{norm:v:e} we used the fact that $\bm{v}^{\mathsf{T}} \bm{v}=1$. The last equality combined with \eqref{61:e} and \eqref{pos:cl:e} shows that $\left. \frac{d}{dt}\inner{\bm{p}}{\dot{\beta}(t)  - \dot{\gamma}(t)}\right|_{t=0} <0$, which completes the proof.
		\end{proof}

\subsection{Proof of Lemma \ref{attained}}
\textbf{Lemma \ref{attained}}\emph{
Let $\ell\colon \mathcal{A}\rightarrow [0,+\infty]^n$ be a loss. If $\Dom \ell =\mathcal{A}$, then either $\mathfrak{H}_{\ell} =\varnothing$ or $\eta_{\ell} \in \mathfrak{H}_{\ell}$.
}
\begin{proof}
Suppose $\mathfrak{H}_{\ell}\neq \varnothing$. Let $\bm{q}\in \Delta_k$, $A\coloneqq \bm{a}_{1:k} \in \mathcal{A}^k$.
By definition of $\eta_{\ell}$ there exists $(\eta_m)\subset [0,+\infty[$ such that $\ell$ is $\eta_m$-mixable and $\eta_m \stackrel{m\to \infty}{\to} \eta_{\ell}$. Therefore, $\forall m \in \mathbb{N},$ $\exists \bm{a}_m \in \mathcal{A}$ such that 
\begin{align}
\forall x\in[n], \; \ell_x(\bm{a}_m) \leq - \eta_m^{-1} \log \inner{\bm{q}}{ \exp(-\eta_m(\ell_x(A))} < +\infty,
\label{attained3}
\end{align} 
where the right-most inequality follows from the fact $\Dom \ell = \mathcal{A}$. Therefore, the sequence $(\ell(\bm{a}_m)) \subset [0,+\infty[^n$ is bounded, and thus admits a convergent subsequence. If we let $\bm{s}$ be the limit of this subsequence, then from \eqref{attained3} it follows that \begin{align}
\forall x\in[n], \; \bm{s} \leq - \eta_{\ell}^{-1} \log \inner{\bm{q}}{ \exp(-\eta_{\ell}(\ell_x(A))},
\label{attained4}
\end{align} 
On the other hand, since $\ell$ is closed (by Assumption \ref{B:}), it follows that there exists $\bm{a}_* \in \mathcal{A}$ such that $\ell(\bm{a}_*) = \bm{s}$. Combining this with \eqref{attained4} implies that $\ell$ is $\eta_{\ell}$-mixable, and thus $\eta_{\ell}\in \mathfrak{H}_{\ell}$.
\end{proof}

%

\subsection{Proof of Theorem \ref{20:}}
\label{a20:}
\textbf{Theorem \ref{20:}}\emph{
		Let $\ell$ and $\Phi$ be as in Theorem \ref{19:}. Then 
		\begin{align*}
		\eta_{\ell}^{\Phi}  = \underline{\eta_{\ell}}	 \inf_{\substack{\tbm{q}\in \Int \tilde{\Delta}_k}}	  \lambda_{\min}(\mathsf{H}\tilde{\Phi}(\tbm{q}) (\mathsf{H}\tilde{\Se}(\tbm{q}))^{-1}),
		\end{align*}
}
	\begin{proof}
		From Theorem \ref{19:}, $\ell$ is $\Phi_{\eta}$-mixable if and only if $\underline{\eta_{\ell}} \Phi_{\eta} - \Se=\eta^{-1}\underline{\eta_{\ell}}\Phi   - \Se$ is convex on $\Delta_k$. When this is the case, Lemma \ref{22:} implies that \begin{align} \label{ter:e}
		1  \leq\eta^{-1} \underline{\eta_{\ell}}  ( \inf_{\substack{\tbm{q}\in \Int \tilde{\Delta}_k}} \lambda_{\min}[\mathsf{H}  \tilde{\Phi}(\tbm{q}) [\mathsf{H}\tilde{\Se}(\tbm{q})]^{-1}]),\end{align}
		where we used the facts that $\mathsf{H} (\eta^{-1} \underline{\eta_{\ell}} \tilde{\Phi}) =\eta^{-1} \underline{\eta_{\ell}}\mathsf{H} \tilde{\Phi}$, $\lambda_{\min}(\cdot)$ is linear, and $ \eta^{-1}\underline{\eta_{\ell}}$ is independent of $\tbm{q} \in \Int \tilde{\Delta}_k$. Inequality \ref{ter:e} shows that the largest $\eta$ such that $\ell$ is $\Phi_{\eta}$-mixable is given by $\eta_{\ell}^{\Phi}$ in \eqref{16:e}.
	\end{proof}

\subsection{Proof of Theorem \ref{21:}}
\label{a21:}
\textbf{Theorem \ref{21:}}\emph{
		Let $\Se, \Phi \colon \mathbb{R}^k\rightarrow \mathbb{R}\cup \{+\infty\}$, where $\Se$ is the Shannon entropy and $\Phi$ is an entropy such that $\tilde{\Phi}\coloneqq \Phi \circ \amalg_k$ is twice differentiable on $\Int \tilde{\Delta}_k$. A loss $\ell\colon \mathcal{A}\rightarrow [0,+\infty[^n$, with $\br_{\ell}$ twice differentiable on $]0,+\infty[^n$, is $\Phi$-mixable only if $R^{\Se}_{\ell}\leq R^{\Phi}_{\ell}$.
}
	\begin{proof}
Suppose $\ell$ is $\Phi$-mixable. Then from Theorem \ref{19:}, $\underline{\eta_{\ell}} \Phi - \Se$ is convex on $\Delta_k$, and thus $\underline{\eta_{\ell}} = \eta^{\Se}_{\ell}>0$ (Corollary \ref{20:}). Furthermore, $\underline{\eta_{\ell}} \tilde{\Phi} - \tilde{\Se} = [\underline{\eta_{\ell}} \Phi - \Se] \circ \amalg_k$ is convex on $\Int \tilde{\Delta}_k$, since $\amalg_k$ is an affine function. It follows from Lemma \ref{22:} and Corollary \ref{20:} that \[\eta^{\Phi}_{\ell} =  \underline{\eta_{\ell}}	 \inf_{\substack{\tbm{q}\in \Int \tilde{\Delta}_k}}	  \lambda_{\min}(\mathsf{H}\tilde{\Phi}(\tbm{q}) (\mathsf{H}\tilde{\Se}(\tbm{q}))^{-1})\geq 1 >0. \]

 Let $\bm{\mu} \in \Ri \Delta_k$ and $\theta_* \coloneqq \argmax_{\theta} D_{\Se}(\bm{e}_{\theta}, \bm{\mu})$. By definition of an entropy and the fact that the directional derivatives $\Phi'(\bm{\mu}; \cdot)$ and $\Se'(\bm{\mu}; \cdot)$ are finite on $\Delta_k$ \citep[Prop. D.1.1.2]{\Hi}, it holds that $D_{\Phi}(\bm{e}_{\theta_*}, \bm{\mu}), D_{\Se}(\bm{e}_{\theta_*}, \bm{\mu}) \in ]0, + \infty[$. Therefore, there exists $\alpha >0$ such that $\alpha^{-1} D_{\Phi}(\bm{e}_{\theta_*}, \bm{\mu}) = D_{\Se}(\bm{e}_{\theta_*}, \bm{\mu})$. If we let $\Psi \coloneqq \alpha^{-1} \Phi$, we get \begin{align}\label{eq:e}D_{\Psi}(\bm{e}_{\theta_*}, \bm{\mu}) = D_{\Se}(\bm{e}_{\theta_*}, \bm{\mu}).\end{align}
		
		Let $d_{\Psi} (\tbm{q})\coloneqq \tilde{\Psi}(\tbm{q}) - \tilde{\Psi}(\tbm{\mu}) - \inner{\tbm{q} - \tbm{\mu}}{\nabla \tilde{\Psi}(\tbm{\mu})} $. Observe that \begin{align*}d_{\Psi} (\tbm{q}) = \Psi(\bm{q}) - \Psi(\bm{\mu}) - \inner{\bm{q} - \bm{\mu}}{\nabla \Psi(\bm{\mu})} =D_{\Psi}(\bm{q}, \bm{\mu}).\end{align*} 
We define $d_{\Se}$ similarly. Suppose that $\eta^{\Psi}_{\ell} > \eta_{\ell}^{\Se}=\underline{\eta_{\ell}}$. Then, from Corollary \ref{20:}, $ \forall \tbm{q}\in \Int \tilde{\Delta}_k$, $ 	  \lambda_{\min}(\mathsf{H}\tilde{\Psi}(\tbm{q}) (\mathsf{H}\tilde{\Se}(\tbm{q}))^{-1}) >1.$ This implies that $\forall \tbm{q}\in \Int \tilde{\Delta}_k$, $\lambda_{\min}(\mathsf{H} d_{\Psi}(\tbm{q}) (\mathsf{H}d_{\Se}(\tbm{q}))^{-1}) >1,$ and from Lemma \ref{22:}, $d_{\Psi} - d_{\Se}$ must be strictly convex on $\Int \tilde{\Delta}_k$. We also have $\nabla d_{\Psi}(\tbm{\mu}) - \nabla d_{\Se}(\tbm{\mu}) = 0$ and $d_{\Psi}(\tbm{\mu}) - d_{\Se}(\tbm{\mu}) = 0$. Therefore, $d_{\Psi} - d_{\Se}$ attains a strict minimum at $\tbm{\mu}$ (\ibid, Thm. D.2.2.1); that is, $d_{\Psi}({\tbm{q}}) > d_{\Se}(\tbm{q})$, $\forall \tbm{q} \in \tilde{\Delta}_k \setminus \{\tbm{\mu}\}$. In particular, for $\tbm{q} = \Pi_k (\bm{e}_{\theta_*})$, we get $D_{\Psi}(\bm{e}_{\theta_*}, \bm{\mu}) = d_{\Psi}(\tbm{q}) > d_{\Se}(\tbm{q})= D_{\Se}(\bm{e}_{\theta_*}, \bm{\mu})$, which contradicts \eqref{eq:e}. Therefore, $\eta_{\ell}^{\Psi}\leq \eta_{\ell}^{\Se}$, and thus 
		\begin{align}
		R^{\Se}_{\ell}(\bm{\mu})=	\max\nolimits_{\theta}  D_{\Se}(\bm{e}_{\theta}, \bm{\mu})  / \eta_{\ell}^{\Se} &=  D_{\Se}(\bm{e}_{\theta_*}, \bm{\mu}) / \eta_{\ell}^{\Se},\nonumber \\& \leq D_{\Psi}(\bm{e}_{\theta_*}, \bm{\mu}) / \eta_{\ell}^{\Psi}, \label{64:e}  \\
		& \leq  \max \nolimits_{\theta} D_{\Psi}(\bm{e}_{\theta}, \bm{\mu}) / \eta_{\ell}^{\Psi}, \nonumber \\
&= R^{\Psi}_{\ell}(\bm{\mu}), \label{65:e}
		\end{align}
		where \eqref{64:e} is due to $D_{\Psi}(\bm{e}_{\theta_*}, \bm{\mu}) = D_{\Se}(\bm{e}_{\theta_*}, \bm{\mu})$ and $\eta_{\ell}^{\Psi}\leq \eta_{\ell}^{\Se}$. Equation \ref{65:e}, implies that $R^{\Se}_{\ell}(\bm{\mu})  \leq R^{\Phi}_{\ell}(\bm{\mu})$, since $R^{\Psi}_{\ell}(\bm{\mu}) =R^{\alpha\Phi}_{\ell}(\bm{\mu}) =  R^{\Phi}_{\ell}(\bm{\mu})$ \citep{Reid2015}. Therefore, \begin{align} \forall \bm{\mu} \in \Ri \Delta_k,\;R^{\Se}_{\ell}(\bm{\mu})\leq R^{\Phi}_{\ell}(\bm{\mu}).\label{reli:e} \end{align}

It remains to consider the case where $\bm{\mu}$ is in the relative boundary of $\Delta_k$. Let $\bm{\mu}\in \Rbd \Delta_k$. There exists $\cI_0 \subsetneq [k]$ such that $\bm{\mu} \in \Delta_{\cI_0}$. Let $\theta^* \in [k]\setminus \cI_0 $ and $\cI \coloneqq \cI_{0} \cup \{\theta^*\}$. It holds that $\bm{\mu} \in \Rbd \Delta_{\cI}$ and $\bm{\mu} + 2^{-1} (\bm{e}_{\theta^*} - \bm{\mu}) \in \Ri \Delta_{\cI}$. Since $\ell$ is $\Phi$-mixable, it follows from Proposition \ref{12:} and the 1-homogeneity of $\Phi'(\bm{\mu}; \cdot)$ \citep[Prop. D.1.1.2]{\Hi} that
				\begin{align}
				\Phi'(\bm{\mu}; \bm{e}_{\theta^*} - \bm{\mu}) &= 2 \Phi'(\bm{\mu}; [\bm{\mu} + 2^{-1} (\bm{e}_{\theta^*} - \bm{\mu} )] - \bm{\mu}) = -\infty.  \nonumber
\shortintertext{Hence, }
				R^{\Phi}_{\ell}(\bm{\mu}) &= \max\nolimits_{\theta \in [k]} D_{\Phi} (\bm{e}_{\theta}, \bm{\mu}),\nonumber  \\
				& \geq D_{\Phi} (\bm{e}_{\theta^*}, \bm{\bm{\mu}})  = \Phi(\bm{e}_{\theta^*}) - \Phi(\bm{\mu}) - \Phi'(\bm{\mu}; \bm{e}_{\theta^*} - \bm{\mu})= +\infty. \label{inf:e}
				\end{align}
Inequality \ref{inf:e} also applies to $\Se$, since $\ell$ is $(\underline{\eta_{\ell}}^{-1} \Se)$-mixable. From \eqref{inf:e} and \eqref{reli:e}, we conclude that $\forall \bm{\mu} \in \Delta_k, \; R_{\ell}^{\Se} (\bm{\mu}) \leq  R_{\ell}^{\Phi}(\bm{\mu})$.
	\end{proof}
		
		\subsection{Proof of Theorem \ref{rboundAGAA}}
		\label{rboundAGAApf}
		\textbf{Theorem \ref{rboundAGAA}} \emph{Let $\Phi: \mathbb{R}^k \rightarrow \mathbb{R} \cup\{+\infty\}$ be a $\Delta$-differentiable entropy. Let $\ell\colon \mathcal{A}\rightarrow [0,+\infty]^n$ be a loss such that $\br_{\ell}$ is twice differentiable on $]0,+\infty[^n$. Let $\bm{\beta}^t = - \eta \sum_{s=1}^{t-1} (\ell_{x^s}(A^s) + \bm{v}^s)$, where $\bm{v}^s\in \mathbb{R}^k$ and $A^s \coloneqq \bm{a}^s_{1:k} \in \mathcal{A}^k$. If $\ell$ is $(\eta, \Phi)$-mixable then for initial distribution $\bm{q}^0= \argmin_{\bm{q} \in \Delta_k} \max_{\theta \in [k]} D_{\Phi}(\bm{e}_{\theta}, \bm{q})$ and any sequence $(x^t, \bm{a}^t_{1:k})_{t=1}^T$, the AGAA achieves the regret
	\begin{align*}
\forall \theta \in[k], \quad \op{Loss}^{\ell}_{AGAA}(T) -  \op{Loss}^{\ell}_{\theta}(T) \leq  R^{\Phi}_{\ell} +  \sum_{t=1}^{T-1} ( v^t_{\theta} -\inner{\bm{v}^t}{\bm{q}^{t}}).
	\end{align*}
}
		\begin{proof}
		Recall that $\Phi_t(\bm{w}) \coloneqq \Phi(\bm{w}) - \Inner{\bm{w}}{\bm{\beta}^t - \bm{\theta}^t }$, where $\bm{\theta}^t =- \eta \sum_{s=1}^{t-1} \ell_{x^s}(A^s)$.
From Theorem \ref{19:} and since $\Phi_t$ is equal to $\Phi$ plus an affine function, it is clear that if $\ell$ is $(\eta, \Phi)$-mixable then $\ell$ is $(\eta, \Phi_t)$-mixable. Thus, for all $(A^t, \bm{q}^{t-1}) \in \mathcal{A}^k \times \Delta_k$, there exists $\bm{a}^t_* \in \mathcal{A}$ such that for any outcome $x^t \in [n]$
\begin{align}
\ell_{x^t}(\bm{a}^t_*) &\leq \eta^{-1} [\Phi_t^{\star}(\nabla \Phi_t(\bm{q}^{t-1})) - \Phi^{\star}_t(\nabla \Phi_t(\bm{q}^{t-1})- \eta \ell_{x^t}(A^t))]. \nonumber 
\shortintertext{Summing over $t$ from 1 to $T$, we get}
\sum^T_{t=1} \ell_{x^t}(\bm{a}^t_*) &\leq\eta^{-1} [\Phi_1^{\star}(\nabla \Phi_1(\bm{q}^{0})) - \Phi^{\star}_T(\nabla \Phi_T(\bm{q}^{T-1})- \eta \ell_{x^T}(A^T))]\label{back} \\ & \hspace{2cm} +\eta^{-1}\sum_{t=1}^{T-1}\left[ \Phi_{t+1}^{\star}(\nabla \Phi_{t+1}(\bm{q}^{t})) - \Phi^{\star}_t(\nabla \Phi_t(\bm{q}^{t-1})- \eta \ell_{x^t}(A^t)) \right]. \nonumber
\end{align}
ODue to the properties of the entropic dual \cite{Reid2015} and the definition of $\Phi_t$, the following holds for all $t\in[T]$ and $\bm{z}$ in $\mathbb{R}^k$,  
\begin{eqnarray}
\nabla\Phi_t(\bm{q}^{t-1}) &=& - \eta \sum_{s=1}^{t-1} \ell_{x^s}(A^s), \label{one}\\
\Phi^{\star}_t(\bm{z}) &=& {\Phi}^{\star}(\bm{z}+ \nabla \Phi(\bm{q}^{t-1}) + \eta \sum_{s=1}^{t-1} \ell_{x^s}(A^s)),\label{two}\\ 
\nabla {\Phi}(\bm{q}^t) &= &\nabla {\Phi}(\bm{q}^{t-1}) - \eta \ell_{x^t}(A^t) - \eta \bm{v}^t. \label{three}
\end{eqnarray}
Using \eqref{one}-\eqref{two}, we get for all $0\leq t<T$, $\Phi_{t+1}^{\star}(\nabla \Phi_{t+1}(\bm{q}^{t})) = \Phi^{\star}(\nabla \Phi(\bm{q}^t))$, and in particular $\Phi_{1}^{\star}(\nabla \Phi_{1}(\bm{q}^{0})) =  \Phi^{\star}(\nabla \Phi(\bm{q}^0))$. Similarly, using \eqref{one}-\eqref{three}, gives $\Phi^{\star}_t(\nabla \Phi_t(\bm{q}^{t-1})- \eta \ell_{x^t}(A^t)) = \Phi^{\star}(\nabla \Phi(\bm{q}^t) + \eta \bm{v}^t)$ for all $1\leq t \leq T$. Substituting back into \eqref{back} yields
\begin{align}
\sum\nolimits^T_{t=1} \ell_{x^t}(\bm{a}^t_*) &\leq \eta^{-1}[{\Phi}^{\star}(\nabla \Phi(\bm{q}^0)) - {\Phi}^{\star}(\nabla \Phi(\bm{q}^T) + \eta \bm{v}^T)]\nonumber\\ & \hspace{2cm} +\eta^{-1}\sum_{t=1}^{T-1}\left[ {\Phi}^{\star}(\nabla \Phi(\bm{q}^t)) - {\Phi}^{\star}(\nabla \Phi(\bm{q}^t) + \eta \bm{v}^t) \right], \label{sumst}
\end{align}
To conclude the proof, we note that  since $\Phi$ is convex it holds that \begin{align}\label{fpr}{\Phi}^{\star}(\nabla \Phi(\bm{q}^t)) - {\Phi}^{\star}(\nabla \Phi(\bm{q}^t) + \eta \bm{v}^t) \leq - \eta\inner{\bm{v}^t}{\nabla{\Phi}^{\star}(\nabla\Phi(\bm{q}^{t}))} = - \eta \inner{\bm{v}^t}{\bm{q}^{t}},\end{align} which allows us to bound the sum on the right hand side of \ref{sumst}. To bound the rest of the terms, we use the fact that $\nabla \Phi (\bm{q}^T) = \nabla \Phi(\bm{q}^0) - \eta \sum_{t=1}^T(\ell_{x^t}(A^t) + \bm{v}^t)$, and thus by letting $\Phi_{\eta} \coloneqq \eta^{-1}\Phi$, \begin{align}
\eta^{-1}[{\Phi}^{\star}(\nabla \Phi(\bm{q}^0)) - {\Phi}^{\star}(\nabla \Phi(\bm{q}^T) + \eta \bm{v}^T)] &={\Phi}_{\eta}^{\star}(\nabla \Phi_{\eta}(\bm{q}^0)) \nonumber \\ & \hspace{0.5cm}- {\Phi}_{\eta}^{\star}\left(\nabla \Phi_{\eta}(\bm{q}^0) -  \sum_{t=1}^T\ell_{x^t}(A^t) - \sum_{t=1}^{T-1} \bm{v}^t\right),\nonumber   \\
&= \inf_{q\in \Delta_k} \Inner{\bm{q}}{ \sum_{t=1}^T\ell_{x^t}(A^t) + \sum_{t=1}^{T-1} \bm{v}^t  } + \frac{D_{\Phi}(\bm{q}, \bm{q}^0)}{\eta},\nonumber\\
&\leq  \sum_{t=1}^T \ell_{x^t}(\bm{a}^t_{\theta})  + \sum_{t=1}^{T-1}v_{\theta}^t + \frac{D_{\Phi}(\bm{e}_{\theta}, \bm{q}^0)}{\eta}, \forall  \theta\in[k]. \nonumber
 \end{align}
Substituting this last inequality and \eqref{fpr} back into \eqref{sumst} yields the desired bound. 
\end{proof}
		
		\section{Defining the Bayes Risk Using the Superprediction Set}
\label{bayessup}
In this section, we argue that when a loss $\ell\colon \mathcal{A} \rightarrow [0,+\infty]^n$ is mixable, in the classical or generalized sense, it does not matter whether we define the Bayes risk $\br_{\ell}$ using the full superprediction set $\sps^{\infty}$ or its finite part $\sps$. Recall the definition of the Bayes risk;

\textbf{Definition \ref{3:}}
\emph{
		Let $\ell \colon \mathcal{A} \rightarrow [0, +\infty]^n$ be a loss such that $\Dom \ell \neq \varnothing$. The \emph{Bayes risk} $\br_{\ell}: \mathbb{R}^n \rightarrow \mathbb{R} \cup \{-\infty\}$ is defined by
		\begin{align}
		\forall \bm{u} \in \mathbb{R}^n, \quad 	\br_{\ell}(\bm{u})  \coloneqq \inf_{\bm{z} \in \sps } \Inner{\bm{u}}{\bm{z}}. \label{inf2} 
		\end{align}}
Note that the right hand side of \eqref{inf2} does not change if we substitute $\sps$ for its closure --- $\overline{\sps}$ --- with respect to $[0,+\infty]^n$. Thus, it suffices to show that $\sps^{\infty}\subseteq \overline{\sps}$ when the loss $\ell$ is mixable. 
We show this in Theorem \ref{goalbayes}, but first we give a characterization of the (finite part) of the superprediction set for a proper loss. 

\begin{proposition}
		\label{6661:}
		Let $\ell \colon \Delta_n \rightarrow [0,+\infty]^n$ be a proper loss. If $\br_{\ell}$ is differentiable on $]0,+\infty[^n$, then \begin{align}\label{equality1}\overline{\sps} \supseteq \mathcal{C}_{\ell} \coloneqq \{\bm{u}\in [0, +\infty]^n\colon \forall \bm{p}\in \Delta_n, \br_{\ell}(\bm{p}) \leq \inner{\bm{p}}{\bm{u}} \}.\end{align} 
\end{proposition}

\begin{proof}
Let $\bm{v} \in\mathcal{C}_{\ell} \cap [0,+\infty[^n$. 
Let $f\colon \Ri \Delta_n \times[n]\rightarrow \mathbb{R}$ be defined by $f(\bm{p}, x)\coloneqq \ell_x(\bm{p}) - v_x$. By the choice of $\bm{v}$, we have $\mathbb{E}_{x \sim \bm{p}} f(\bm{p},x) = \inner{\bm{p}}{\ell(\bm{p})} - \inner{\bm{p}}{\bm{v}} \leq 0 $ for all $\bm{p}\in \Delta_n$. Since $\br_{\ell}$ is differentiable on $]0,+\infty[^n$, by assumption, $\ell$ is continuous on $\Ri \Delta_{n}$, and thus $f$ is continuous in the first argument. Since $\bm{v}$ has finite components, the map $f$ satisfies all the conditions of Lemma \ref{23:}. Therefore, there exists $(\bm{p}_m) \subset \Ri \Delta_n$ such that 
			\begin{align}
			\forall m \in \mathbb{N}, \forall x\in[n],  \; \ell_x(\bm{p}_m) \leq v_x +\frac{1}{m}. \label{eqee}
			\end{align}
Without loss of generality, we can assume by extracting a subsequence if necessary that $\ell(\bm{p}_m)$ converges to $\bm{s} \in [0,+\infty]^n$. By definition, we have $\bm{s} \in \overline{\sps}$, and from \eqref{eqee} it follows that $\bm{s}\leq \bm{v}$ coordinate-wise. Thus, $\bm{v}$ is in $\overline{\sps}$. 

The above argument shows that $\mathcal{C}_{\ell} \cap [0,+\infty[^n \subseteq \overline{\sps}$, and since $\overline{\sps}$ is closed in $[0,+\infty]^n$ we have $\overline{\mathcal{C}} \subseteq \overline{\sps}$, where $\overline{\mathcal{C}}$ is the closure of $\mathcal{C}_{\ell} \cap [0,+\infty[^n$ in $[0,+\infty]^n$. Now it suffice to show that $\mathcal{C}_{\ell} \subseteq \overline{\mathcal{C}} $ to complete the proof.

Let $\bm{u} \in \mathcal{C}_{\ell}$ and $\cI \coloneqq \{ x\in [n]\colon u_x <+\infty \}$. Define $(\bm{u}_m) \subset [0,+\infty[^n$ by $u_{m,x}= u_x$ if $x \in \cI$; and $m$ otherwise. Let $\bm{p}\in \Delta_n$. It follows that 
\begin{align}
\inner{\bm{p}}{\bm{u}_m} &= \sum_{x'\in \cI} p_{x'} u_{m,x'} + \sum_{x\notin \cI} p_{x} u_{m,x}, \nonumber \\
 & =\sum_{x'\in \cI} p_{x'} u_{m,x'}+ \sum_{x\notin \cI} p_x u_{m,x},\label{lasttt}.
\end{align}
\begin{claim}
\label{c2}
 $\forall \epsilon >0, \exists m_{\epsilon}\geq 1,\forall \bm{p}\in \Delta_k, \br_{\ell}(\bm{p}) \leq  \inner{\bm{p}}{\bm{u}_{m_{\epsilon}}} - \epsilon$.
\end{claim}
Suppose that Claim \ref{c2} is false. This means that there exists $\delta>0$ such that \begin{align}\forall \bm{m}\geq 1, \exists \bm{p}_m \in \Delta_n, \inner{\bm{p}_m}{\bm{u}_m} -\delta <  \br_{\ell}(\bm{p}_m).\label{cont} \end{align} 
We may assume, by extracting a subsequence if necessary ($\Delta_n$ is compact), that $(\bm{p}_m)$ converges to $\bm{p}_*\in \Delta_n$. Taking the limit $m\to \infty$ in \eqref{cont} would lead to the contradiction `$\inner{\bm{p}_*}{\bm{u}} < \br_{\ell}(\bm{p}_*) $', since from \eqref{lasttt} we have $\lim_{m\to \infty} \inner{\bm{p}_m}{\bm{u}_m} = \inner{\bm{p}_*}{\bm{u}}$. Therefore, Claim \ref{c2} is true. For $\epsilon= \frac{1}{k}$ let $m_k\coloneqq m_{\epsilon}$ be as in Claim \eqref{c2}. The claim then implies that $\lim \inf_{k\to \infty} \inner{\bm{p}}{\bm{u}_{m_{k}}} \geq \br_{\ell}(\bm{p})$ uniformly for $\bm{p}\in \Delta_k$. By the claim we also have that $\bm{u}_{m_k} \in \mathcal{C}_{\ell} \cap [0,+\infty[^n$ for all $k \in \mathbb{N}$, and by construction of $\bm{v}_m$, we have $\lim_{k\to \infty} \bm{u}_{m_k} = \bm{u}$. This shows that $\mathcal{C}_{\ell} \subseteq \overline{\mathcal{C}}$, which completes the proof.

  
\end{proof}

\begin{theorem}
\label{goalbayes}
Let $\ell: \mathcal{A} \rightarrow [0,+\infty]^n$ be a loss. If $ \mathscr{S}_{\ell}^{\infty} \not\subseteq \overline{\mathscr{S}_{\ell}}$, then $\ell$ is not mixable.
\end{theorem}
\begin{proof}
Suppose that $\ell$ is mixable and let $\sell$ be a proper support loss of $\ell$. From Proposition \ref{14:}, $\br_{\ell}$ is differentiable on $]0,+\infty[^n$, and thus Theorem \ref{5:} implies that $\overline{\sps} = \overline{\mathscr{S}_{\underline{\ell}}}$. Therefore, Lemma \ref{6661:} implies that $\overline{\sps} \supseteq \{\bm{u}\in [0, +\infty]^n\colon \forall \bm{p}\in \Delta_n, \br_{\ell}(\bm{p}) \leq \inner{\bm{p}}{\bm{u}} \}$. Thus, if $\mathscr{S}_{\ell}^{\infty} \not\subseteq \overline{\mathscr{S}_{\ell}}$, there exists $\epsilon>0$, $\bm{p}_{\epsilon} \in \Delta_k$, and $\bm{s}\in \mathscr{S}_{\ell}^{\infty}\setminus \overline{\sps}$ such that \begin{align}\label{f1}\inner{\bm{p}_{\epsilon}}{\bm{s}} < \br_{\ell}(\bm{p}_{\epsilon}) - 2\epsilon.\end{align} 
Note that $\bm{p}_{\epsilon}$ cannot be in $\Ri \Delta_n$; otherwise, \eqref{f1} would imply that $\bm{s}$ has all finite components, and thus would be included in $\overline{\sps}$, which is a contradiction. Assume from now on that $\bm{p}_{\epsilon} \in \Rbd \Delta_n$. From the definition of the support loss, there exists a sequence $(\bm{p}_m) \subseteq \Ri \Delta_n$ such that $\bm{p}_m \stackrel{m \to \infty}{\to} \bm{p}_{\epsilon}$ and $\sell(\bm{p}_m)\stackrel{m \to \infty}{\to}\sell(\bm{p}_{\epsilon})$. Therefore, Theorem \ref{5:} implies that there exists $\bm{a}_{\epsilon}\in \mathcal{A}$ such that \begin{align}\inner{\bm{p}_{\epsilon}}{\ell(\bm{a}_{\epsilon})} < \inner{\bm{p}_{\epsilon}}{\sell(\bm{p}_{\epsilon})}+ \epsilon. \label{f2}\end{align}
To see this, note that since $(\bm{p}_m)\subset \Ri \Delta_n \subseteq \Dom \sell$, Theorem \ref{5:} guarantees the existence of a sequence $(\bm{a}_m)\subset \mathcal{A}$ such that $\ell(\bm{a}_m)=\sell(\bm{p}_m)$. On the other hand, for any $x \in[n]$ such that $\ell_x(\bm{p}_{\epsilon})=+\infty$, we have $p_{\epsilon,x}=0$ --- otherwise, $\br_{\ell}(\bm{p}_{\epsilon})$ would be infinite. It follows, by continuity of the inner product that $\inner{\bm{p}_{\epsilon}}{\ell(\bm{a}_{m})} \stackrel{m \to \infty}{\to} \inner{\bm{p}_{\epsilon}}{\sell(\bm{p}_{m})}$, and thus it suffices to pick $\bm{a}_{\epsilon}$ equal to $\bm{a}_m$ for $m$ large enough.

Now since $\ell$ is $\eta$-mixable, there exists $\eta>0$ and $\bm{a}_*\in \mathcal{A}$ such that 
\begin{align}
\ell(\bm{a}_*)& \leq - \eta^{-1} \log \left( \frac{1}{2}e^{- \eta \bm{s}} + \frac{1}{2} e^{-\eta\ell(\bm{a}_{\epsilon})} \right),\nonumber\\
\shortintertext{and due to the convexity of $-\log$,}
& \leq \frac{1}{2} \bm{s} + \frac{1}{2}\ell(\bm{a}_{\epsilon}). \nonumber \\  
\shortintertext{Using \eqref{f1} and \eqref{f2} yields}
\inner{\bm{p}_{\epsilon}}{\ell(\bm{a}_*)} &\leq \br_{\ell}(\bm{p}_{\epsilon}) - \epsilon /2.\label{f3}
\end{align} 
On the other hand, by definition of a proper support loss, $\inner{\bm{p}_{\epsilon}}{\sell(\bm{p}_{\epsilon})} \leq \inner{\bm{p}_{\epsilon}}{\ell(\bm{a}_*)}$. This combined with \eqref{f3}, lead to the contradiction $\inner{\bm{p}_{\epsilon}}{\sell(\bm{p}_{\epsilon})}< \br_{\ell}(\bm{p}_{\epsilon})$.
\end{proof}

	\section{The Update Step of the GAA and the Mirror Descent Algorithm}
	\label{s2.2}In this section, we demonstrate that the update steps of the GAA and the Mirror Descent Algorithm are essentially the same (at least for finite losses) according to the definition of the MDA given by Beck and Teboulle \citep{Beck2003};

Let $\ell\colon \mathcal{A} \rightarrow [0,+\infty[^n$ be a loss and $\Phi \colon \mathbb{R}^k \rightarrow \mathbb{R}\cup \{+\infty\}$ an entropy such that $\tilde{\Phi}$ is differentiable on $\Int \tilde{\Delta}_k$. Let $\bm{q}^{t}$ be the update distribution of the GAA at round $t$ and $\tbm{q}^{t}= \Pi_k(\bm{q}^{t})$. It follows from the definition of $\bm{q}^t$ (see Algorithm \ref{GAA2}) that
		\begin{align}
		\tbm{q}^{t} & = \argmin_{\tbm{q} \in \tilde{\Delta}_k} \Inner{\amalg_k(\tbm{q})}{\ell_{x^t}(A^t)} + \eta^{-1}  D_{\tilde{\Phi}} (\tbm{q}, \tbm{q}^{t-1}),\nonumber  \\
		& = \argmin_{\tbm{q} \in \tilde{\Delta}_k} \Inner{\tbm{q}}{J^{\mathsf{T}}_k\ell_{x^t}(A^t)} + \eta^{-1}  D_{\tilde{\Phi}} (\tbm{q}, \tbm{q}^{t-1}),\nonumber\\
		&=\argmin_{\tbm{q} \in  \tilde{\Delta}_k} \Inner{\tbm{q}}{\nabla l_t(\tbm{q}^{t-1})} + \eta^{-1}  D_{\tilde{\Phi}} (\tbm{q}, \tbm{q}^{t-1}),  \label{7:e}
		\end{align}
		where $l_t(\tbm{\mu}) \coloneqq \inner{\amalg_k(\tbm{\mu}) }{\ell_{x^t}(A^t)} =\inner{\bm{\mu}}{\ell_{x^t}(A^t)}$. Update \eqref{7:e} is, by definition \citep{Beck2003}, the MDA with the sequence of losses $l_t$ on $\Int \tilde{\Delta}_k$, `distance' function $D_{\tilde{\Phi}}(\cdot, \cdot)$, and learning rate $\eta$. Therefore, the MDA is exactly the update step of the GAA.

\section{The Generalized Aggregating Algorithm Using the Shannon Entropy $\Se$}
The purpose of this appendix is to show that the GAA reduces to the AA when the former uses the Shannon entropy. In this case, generalized and classical mixability are equivalent. In what follows, we make use of the following proposition which is proved in \ref{a16:}.

\textbf{Proposition \ref{2:}}
\emph{
For the Shannon entropy $\Se$, it holds that $\tilde{\Se}^* (\bm{v}) = \log(\inner{ \exp(\bm{v})}{\bm{1}_{\tilde{k}}}+ 1), \forall \bm{v}\in \mathbb{R}^{k-1}$, and  $ {\Se}^{\star}(\bm{z}) = \log \inner{ \exp(\bm{z})}{\bm{1}_{k}}, \forall \bm{z}\in \mathbb{R}^k$.
}

Let $\ell \colon\mathcal{A} \rightarrow [0,+\infty[^n$ be a loss and $\Phi$ be as in Proposition \ref{15:} and suppose that $\Phi$ and $\tilde{\Phi}^*$ are differentiable on $\Ri \Delta_k$ and $\mathbb{R}^{k-1}$, respectively.
 It was shown in \cite{Reid2015} that
\begin{align}
 \nabla \Phi^{\star} (\nabla \Phi(\bm{q}) - \ell_x(A)) = \argmin_{\bm{\mu}\in \Delta_k} \inner{\bm{\mu}}{\ell_x(A)} + D_{\Phi}(\bm{\mu}, \bm{q}), \label{11:e} \\
\M_{\Phi}(\ell_x(A), \bm{q})  =\Phi^{\star}(\nabla \Phi(\bm{q})) - \Phi^{\star}(\nabla \Phi(\bm{q}) - \ell_x(A)). \label{12:e}
\end{align}
Let $\bm{q}\in \Ri \Delta_k$. By definition of $\Se$, $\nabla \Se(\bm{q}) = \log \bm{q} + \bm{1}_k$, and due to Proposition \ref{2:}, ${\Se}^\star(\bm{z}) = \log \inner{\exp{\bm{z}}}{\bm{1}_{k}}, \bm{z} \in \mathbb{R}^k$. Therefore, $\nabla \Se(\bm{q}) - \eta \ell_{x}(A) = \log (\exp(-\eta \ell_{x}(A)) \odot \bm{q} ) + \bm{1}_k$ and $\nabla \Se^{\star}(\bm{z}) = \frac{\exp \bm{z}}{\inner{\exp{\bm{z}}}{\bm{1}_{k}}}$, $\forall (x, A) \in [n] \times (\Dom \ell)^k$. Thus, \begin{align} \label{13:e}
		\nabla {\Se}^{\star}(\nabla \Se (\bm{q}) - \eta \ell_{x}(A)) = \frac{\exp({-\eta \ell_{x}(A)}) \odot \bm{q}}{\Inner{\exp({-\eta \ell_{x}(A)})}{\bm{q}}}.\end{align}

Let $\Se_{\eta}\coloneqq \eta^{-1} \Se$. Then $\nabla \Se =\eta \nabla \Se_{\eta}$ and $\forall \bm{z} \in \mathbb{R}^k, \nabla \Se^{\star}_{\eta}(\bm{z}) = \nabla \Se^{\star}(\eta \bm{z})$ \citep{Reid2015}.\footnote{Reid et al. \cite{Reid2015} showed the equality $\nabla \Phi^{\star}_{\eta}(\bm{u}) = \nabla \Phi^{\star}(\eta \bm{u}), \forall \bm{u} \in \Dom \Phi^{\star}$, for any entropy differentiable on $\Delta_k$ - not just for the Shannon Entropy.} Then the left hand side of \eqref{13:e} can be written as $\nabla \Se^{\star}_{\eta}(\nabla \Se_{\eta} (\bm{q}) -  \ell_{x}(A))$. Using this fact, \eqref{11:e} and \eqref{13:e} show that the update distribution $\bm{q}^t$ of the GAA (Algorithm \ref{GAA2}) coincides with that of the AA  after substituting $\bm{q}, x$, and $A$ by $\bm{q}^{t-1}, x^t$, and $A^t\coloneqq [\bm{a}_{\theta}]_{\theta \in [k]}$, respectively.

Now using the fact that $\M^{\eta}_{\Se}(\ell_x(A), \bm{q})  = \eta^{-1} \M_{\Se}(\eta \ell_x(A), \bm{q}) $ \citep{Reid2015} and \eqref{12:e}, we get \begin{align}
\M^{\eta}_{\Se}(\ell_x(A), \bm{q})  & =  \eta^{-1} [ {\Se}^{\star}(\nabla \Se(\bm{q})) - {\Se}^{\star}(\nabla \Se(\bm{q}) -\eta \ell_x(A)) ], \nonumber \\ &
 = -\eta^{-1} \log \inner{\exp({-\eta \ell_{x}(A)})}{\bm{q}}. \label{stuff}
 \end{align}
 Equation \ref{stuff} shows that the $\eta$-mixability condition is equivalent to the $(\eta, \Se)$-mixability condition for a finite loss. This remains true for losses taking infinite values --- see the proof of Theorem \ref{16:} in Appendix \ref{a16:}.

		\section{Legendre $\Phi$, but no $\Phi$-mixable $\ell$}
		\label{ap:D}
In this appendix, we construct a \emph{Legendre type} entropy \citep{Rockafellar1997a} for which there are no $\Phi$-mixable losses satisfying a weak condition (see below).

Let $\ell: \mathcal{A}\rightarrow [0,+\infty]^n$ be a loss satisfying condition \ref{B:}. According to Alexandrov's Theorem, a concave function is twice differentiable almost everywhere (see e.g. \citep[Thm. 6.7]{Borwein2010}). Now we give a version of Theorem \ref{18:} which does not assume the twice differentiability of the Bayes risk. The proof is almost identical to that of Theorem \ref{18:} with only minor modifications.
\begin{theorem}
		\label{ext18:} Let $\Phi \colon \mathbb{R}^k \rightarrow \mathbb{R} \cup \{+ \infty\}$ be an entropy such that $\tilde{\Phi}$ is twice differentiable on $\Int \tilde{\Delta}_k$, and $\ell \colon \mathcal{A} \rightarrow [0,+\infty]^n$ a loss satisfying Condition \ref{B:} and such that $\exists (\tbm{p}, \bm{v})\in \mathcal{D}\times \mathbb{R}^{\tilde{n}}, \mathsf{H}\tbr_{\ell}(\tbm{p})\bm{v} \neq \bm{0}_{\tilde{n}}$, where $\mathcal{D} \subset \Int \tilde{\Delta}_n$ is a set of Lebesgue measure 1 where $\tbr_{\ell}$ is twice differentiable, and define 
\begin{align}
\underline{\eta_{\ell}}^* \coloneqq \inf_{\tbm{p} \in\mathcal{D} }(\lambda_{\max} ([\mathsf{H} \tbr_{\log}(\tbm{p})]^{-1} \mathsf{H} \tbr_{\ell} (\tbm{p})))^{-1}. \label{newmc:e}
\end{align}
Then $\ell$ is $\Phi$-mixable only if $\underline{\eta_{\ell}}^* \Phi - \Se$ is convex on $\Delta_k$. 
	\end{theorem}
The new condition on the Bayes risk is much weaker than requiring $\br_{\ell}$ to be twice differentiable on $]0,+\infty[^n$. In the next example, we will show that there exists a Legendre type entropy for which there are no $\Phi$-mixable losses satisfying the condition of Theorem \ref{ext18:}.

\begin{example}
		 Let $\Phi: \mathbb{R}^2 \rightarrow \mathbb{R} \cup \{+\infty\}$ be an entropy such that \begin{align*}
		\forall q\in ]0,1[,\; \Phi(q, 1-q) = \tilde{\Phi}(q)  = \int_{1/2}^q \log \left(\frac{\log (1-t)}{\log t}\right) dt.
		\end{align*} 
		$\tilde{\Phi}$ is differentiable and strictly convex on the open set $(0,1)$. Furthermore, it satisfies \eqref{8:e} which makes it a function of Legendre type \citep[Lem. 26.2]{Rockafellar1997a}. In fact, \eqref{8:e} is satisfied due to 
		\begin{align*}
		\left|\frac{d}{dq} \tilde{\Phi} (q)\right| &= \left|\log \left(\frac{\log (1-q)}{\log q}\right)\right| \stackrel{q\to b}{\to}   +\infty, \mbox{ where } b \in \{0,1\},\\
		\frac{d^2}{dq^2} \tilde{\Phi} (q) & = \frac{-1}{q \log q } + \frac{-1}{(1-q) \log (1-q)} > 0, \; \forall q \in ]0,1[.
		\end{align*}
		The Shannon entropy on $\Delta_2$ is defined by $\Se(q, 1-q) = \tilde{\Se}(q) = q \log q + (1-q) \log (1-q)$, for $q\in ]0,1[$. Thus, $\frac{d^2}{dq^2} \tilde{\Se}(q) = \frac{1}{q(1-q)}$.

Suppose now that there exists a $\Phi$-mixable loss $\ell \colon \mathcal{A} \rightarrow [0,+\infty]^n$ satisfying condition \ref{B:} and such that $\exists (\tbm{p}, \bm{v})\in \mathcal{D}\times \mathbb{R}^{\tilde{n}}, \mathsf{H}\tbr_{\ell}(\tbm{p})\bm{v} \neq \bm{0}_{\tilde{n}}$. Let $\underline{\eta_{\ell}}^*$ be as in \eqref{newmc:e}. By definition, we have $\underline{\eta_{\ell}}^*<+\infty$, and thus 
		\begin{align}
		\underline{\eta_{\ell}}^*\left[\frac{d^2}{dq^2} \tilde{\Phi} (q)\right] \left[\frac{d^2}{dq^2} \tilde{\Se}(q)\right]^{-1}= \underline{\eta_{\ell}}^* \left( \frac{q-1}{\log q} + \frac{-q}{\log (1-q)} \right) \stackrel{q \to b}{\to} 0, \label{count:e}
		\end{align}
		where $b \in \{0, 1\}$. From Lemma \ref{22:}, \eqref{count:e} implies that $\underline{\eta_{\ell}}^* \Phi - \Se$ is not convex on $\Delta_k$, which is a contradiction according to Theorem \ref{ext18:}.
\end{example}

		\section{Loss Surface and Superprediction Set }
In this appendix, we derive an expression for the curvature of the image of a proper loss function. We will need the following lemma. 
	\begin{lemma}
			\label{41:}
			Let $\sigma: [0,+\infty[^n \rightarrow \mathbb{R}$ be a 1-homogeneous, twice differentiable function on $]0, +\infty[^n$. Then $\sigma$ is concave on $]0, +\infty[^n$ if and only if $\tilde{\sigma} = \sigma \circ \amalg_n$ is concave on $\Int \tilde{\Delta}_n$.
		\end{lemma}
		
		\begin{proof}
			The forward implication is immediate; if $\sigma$ is concave on $]0, +\infty[^n$, then $\sigma \circ \amalg_k$ is concave on $\Int \tilde{\Delta}_k$, since $\amalg_k$ is an affine function. 
			
			Now assume that $\tilde{\sigma}$ is concave on $\Int \tilde{\Delta}_k$. Let $\lambda \in [0,1]$ and $(\bm{p}, \bm{q}) \in [0,+\infty[^n \times [0,+\infty[^n$. We need to show that 
			\begin{align}
			\lambda \sigma(\bm{p}) + (1- \lambda) \sigma(\bm{q}) \leq \sigma( \lambda \bm{p} + (1-\lambda) \bm{q}). \label{triv:e}
			\end{align}
			
			Note that if $\bm{p}=\bm{0}$ or $\bm{q}=\bm{0}$, \eqref{triv:e} is trivially with equality due to the 1-homogeneity of $\sigma$. Now assume that $\bm{p}$ and $\bm{q}$ are non-zero and let $c \coloneqq \lambda \norm{\bm{p}}_1  + (1- \lambda) \norm{\bm{q}}_1$. For convenience, we also denote ${\bm{p}_1} = \bm{p}/\norm{\bm{p}}_1$ and ${\bm{q}}_1 = \bm{q}/\norm{\bm{q}}_1$ which are both in $\Delta_n$. It follows that 
			\begin{align*}
			\lambda \sigma(\bm{p}) + (1- \lambda) \sigma(\bm{q}) &=cM \left( \lambda \frac{\norm{\bm{p}}_1}{c} \sigma({\bm{p}_1})+ (1-\lambda)\frac{\norm{\bm{q}}_1}c  \sigma({\bm{q}_1}) \right), \\
			&=c \left( \lambda \frac{\norm{\bm{p}}_1}{c} \tilde{\sigma}({\tbm{p}_1})+ (1-\lambda)\frac{\norm{\bm{q}}_1}c  \tilde{\sigma}({\tbm{q}_1}) \right),\\
& \leq c  \tilde{ \sigma}\left(\lambda \frac{\norm{\bm{p}}_1}{c} {\tbm{p}_1}+ (1-\lambda)\frac{\norm{\bm{q}}_1}c {\tbm{q}_1} \right), \\
			& = c   \sigma\left(\lambda \frac{\norm{\bm{p}}_1}{c} {\bm{p}_1}+ (1-\lambda)\frac{\norm{\bm{q}}_1}c {\bm{q}_1} \right), \\
			& = \sigma(\lambda \bm{p} + (1-\lambda) \bm{q}),
			\end{align*}
			where the first and last equalities are due the 1-homogeneity of $\sigma$ and the inequality is due to $\tilde{\sigma}$ being concave on the $\Int \tilde{\Delta}_n$.
		\end{proof}
		\subsection{Convexity of the Superprediction Set}

	In the literature, many theoretical results involving loss functions relied on the fact that the superprediction set of a proper loss is convex \citep{DBLP:journals/jmlr/WilliamsonVR16,dawid2007geometry}. An earlier proof of this result by \cite{DBLP:journals/jmlr/WilliamsonVR16} was incomplete\footnote{It was claimed that if $\sps$ is non-convex, there exists a point $\bm{s}_0$ on the loss surface $\mathcal{S}_{\ell}$ such that no hyperplane supports $\sps$ at $\bm{s}_0$. The non-convexity of a set by itself is not sufficient to make such a claim; the continuity of the loss $\ell$ is required.}. In the next theorem we restate this result.
	\begin{theorem}
		\label{6:}
		If $\ell \colon \Delta_n \rightarrow [0,+\infty[^n$ is a continuous proper loss, then  $\sps = \bigcap_{\bm{p} \in \Delta_n} \mathcal{H}_{-\bm{p}, -\br_{\ell}(\bm{p})}$. In particular, $\sps$ is convex.
	\end{theorem}
\begin{proof}
			[$\sps \subseteq \bigcap_{\bm{p} \in \Delta_n} \mathcal{H}_{-\bm{p}, -\br_{\ell}(\bm{p})}$]: Let $\bm{v}\in \sps$, $\bm{u} \in [0,+\infty[^n$, and $\bm{q} \in \Delta_n$ such that $\bm{v} =\ell(\bm{q})+\bm{u}$. Since $\ell$ is proper then $\forall \bm{p} \in \Delta_n,  \br_{\ell}(\bm{p})= \inner{\bm{p}}{\ell(\bm{p})} \leq \inner{\bm{p}}{\ell(\bm{q})} \leq \inner{\bm{p}}{\ell(\bm{q}) + \bm{u}} = \inner{\bm{p}}{\bm{v}}$. Therefore, $\bm{v} \in \bigcap_{\bm{p}\in \Delta_n} \mathcal{H}_{-\bm{p}, -\br_{\ell}(\bm{p})}$.
			
			[$\bigcap_{\bm{p} \in \Delta_n} \mathcal{H}_{-\bm{p}, -\br_{\ell}(\bm{p})} \subseteq \sps $]:	Let $\bm{v} \in \bigcap_{\bm{p} \in \Delta_n}  \mathcal{H}_{-\bm{p}, -\br_{\ell}(\bm{p})}$. Let $\Omega = [n]$, $\Delta(\Omega)=\Delta_n$, and $Q(\bm{p}, x)=\ell_x(\bm{p}) - v_x$ for all $(\bm{p}, x)\in \Delta_n \times [n] $. Since $\bm{v} \in \bigcap_{\bm{p} \in \Delta_n}  \mathcal{H}_{-\bm{p}, -\br_{\ell}(\bm{p})}$, $\mathbb{E}_{x \sim \bm{p}} Q(\bm{p},x) = \inner{\bm{p}}{\ell(\bm{p})} - \inner{\bm{p}}{\bm{v}} \leq 0 $ for all $\bm{p}\in \Delta_n$. Lemma \ref{24:}, implies that there exists $\bm{p}_*\in \Delta_n$ such that $Q(\bm{p}_*, x) = \ell_x(\bm{p}_*) - v_x \leq 0$, for all $x \in [n]$. This shows that $\bm{v}\in \sps$.
		\end{proof}

		\subsection{Curvature of the Loss Surface }
		\label{b1:}
		The \emph{normal curvature} of a $\tilde{n}$-manifold $\mathcal{S}$ \citep{thorpe2012elementary} at a point $\bm{r} \in \mathcal{S}$ in the direction of $\bm{w} \in T_{\bm{r}}\mathcal{S}$, where $T_{\bm{r}}\mathcal{S}$ is the \emph{tangent space} of $\mathcal{S}$ at $\bm{r}\in \mathcal{S}$, is defined by 
		\begin{align}
		\label{66:e}
		\kappa(\bm{r},\bm{w}) = \frac{\Inner{\bm{w}}{\mathsf{D}\mathsf{N}^{\mathcal{S}}(\bm{r}) \bm{w}}}{\Inner{\bm{w}}{\bm{w}}},
		\end{align}
		where $\mathsf{N}^{\mathcal{S}}(\bm{r})$ is the normal vector to the surface at $\bm{r}$.  The \emph{minimum principal curvature} of $\mathcal{S}$ at $\bm{r}$ is expressed as $\underline{\kappa}(\bm{r}) \coloneqq \inf \{\kappa(\bm{r},\bm{w}): \bm{w}\in T_{\bm{r}} \mathcal{S} \cap \mathcal{B}(\bm{r}, 1)\}$.
		
		In the next theorem, we establish a direct link between the curvature of a loss surface and the Hessian of the loss' Bayes risk.
		\begin{theorem}
			\label{38:}
			Let $\ell \colon \Ri \Delta_n \rightarrow [0,+\infty[^n$ be a loss whose Bayes risk is twice differentiable and strictly concave on $]0,+\infty[^n$. Let $\bm{p} \in \Ri \Delta_n$, $X_{\bm{p}} \coloneqq I_{\tilde{n}} - \tbm{p} \bm{1}^{\mathsf{T}}_{\tilde{n}}$, and $\bm{w} \in T_{\tilde{\ell}(\tbm{p})}\mathcal{S}_{\ell}$. Then \begin{enumerate} 
\item $\exists \bm{v}\in \mathbb{R}^{n-1}$ such that $\mathsf{D}\tilde{\ell}(\tbm{p})\bm{v}=\bm{w}$. \item $\mathcal{S}_{\ell}$ is a $\tilde{n}$-manifold. \item The normal curvature of $\mathcal{S}_{\ell}$ at $\ell(\bm{p}) =\tilde{\ell}(\tbm{p})$ in the direction $\bm{w}$ is given by 
			\begin{align}
			\label{67:e}
			\kappa_{\ell}(\ell(\bm{p}), \bm{w}) &= \norm{\begin{bmatrix} X_{\bm{p}} \\ - \tbm{p}^{\mathsf{T}}   \end{bmatrix} (- \mathsf{H} \tbr_{\ell}(\tbm{p}))^{\frac{1}{2}} \bm{u}}^{-1},
			\end{align}
			where $\bm{u} = (- \mathsf{H} \tbr_{\ell}(\tbm{p}))^{\frac{1}{2}}\bm{v}/\Norm{(- \mathsf{H} \tbr_{\ell}(\tbm{p}))^{\frac{1}{2}}\bm{v}}$. 
\end{enumerate} 
		\end{theorem}
		It becomes clear from \eqref{67:e} that smaller eigenvalues of $- \mathsf{H} \tbr_{\ell}(\tbm{p})$ will tend to make the loss surface more curved at $\ell(\bm{p})$, and vice versa. 

Before proving Theorem \ref{38:}, we first define parameterizations on manifolds.
		\begin{definition}[Local and Global Parameterization]
			Let $\mathcal{S} \subseteq \mathbb{R}^{n}$ be a $\tilde{n}$-manifold and $\mathcal{U}$ an open set in $\mathbb{R}^{\tilde{n}}$. The map $\varphi:\mathcal{U} \rightarrow \mathcal{S}$ is called a \emph{local parameterization} of $\mathcal{S}$ if $\mathsf{D} \varphi(\bm{u}): \mathbb{R}^{\tilde{n}} \rightarrow T_{\varphi(\bm{u})}\mathcal{S}$ is injective for all $\bm{u}\in \mathcal{U}$, where $T_{\varphi(\bm{u})}\mathcal{S}$ is the tangent space of $\mathcal{S}$ at ${\varphi(\bm{u})} \in \mathcal{S}$. $\varphi$ is called a \emph{global} parameterization of $\mathcal{S}$ if it is, additionally, onto. 
		\end{definition}
		
		Let $\varphi$ be a global parameterization of $\mathcal{S}$ and $\mathsf{N}^{\varphi} \coloneqq \mathsf{N}^{\mathcal{S}}\circ \varphi$. By a direct application of the chain rule, \eqref{66:e} can be written as
		\begin{align}
		\kappa({\varphi(\bm{u})},\bm{w}) = \frac{\Inner{\bm{w}}{\mathsf{D} \mathsf{N}^{\varphi} (\bm{u})  \bm{v}}}{\Inner{\bm{w}}{\bm{w}}}, \label{68:e}
		\end{align}
		where $\bm{v}$ is such that $\mathsf{D}   \varphi (\bm{u}) \bm{v} = \bm{w}$. The existence of such a $\bm{v}$ is guaranteed by the fact that $\mathsf{D} \varphi$ is injective and $\Dim \mathbb{R}^{\tilde{n}} = \Dim T_{\varphi(\bm{u})}\mathcal{S} = \tilde{n}$.
		
		\begin{proof}[\textbf{Theorem \ref{38:}}]
			First we show that $\mathcal{S}_{\ell}$ is a $\tilde{n}$-manifold. Consider the map $\tilde{\ell}:\Int \tilde{\Delta}_n \rightarrow \mathcal{S}_{\ell}$ and note that $\Int \tilde{\Delta}_n$ is trivially a $\tilde{n}$-manifold. Due to the strict concavity of the Bayes risk, $\tilde{\ell}$ is injective \citep{DBLP:journals/jmlr/ErvenRW12} and from Lemmas \ref{27:} and \ref{41:}, $\mathsf{D} \tilde{\ell}(\tbm{p}): \mathbb{R}^{\tilde{n}} \rightarrow T_{\tilde{\ell}(\tbm{p})} \mathcal{S}_{\ell}$ is also injective. Therefore, $\tilde{\ell}$ is an \emph{immersion} \citep{Robbin2011}. $\tilde{\ell}$ is also \emph{proper} in the sense that the preimage of every compact subset of $\mathcal{S}_{\ell}$ is compact. Therefore, $\tilde{\ell}$ is a proper injective immersion, and thus it is an embedding from the $\tilde{n}$-manifold $\Int \tilde{\Delta}_n$ to $\mathcal{S}_{\ell}$ (\ibid). Hence, $\mathcal{S}_{\ell}$ is a manifold.
			
			Now we prove \eqref{67:e}. The map $\tilde{\ell}$ is a global parameterization of $\mathcal{S}_{\ell}$.  In fact, from Lemma \ref{27:}, $\mathsf{D} \tilde{\ell}(\tbm{p})$ has rank $\tilde{n}$, for all $\tbm{p} \in \Int \tilde{\Delta}_n$, which implies that $\mathsf{D}  \tilde{\ell}(\tbm{p})$ is onto from $\mathbb{R}^{\tilde{n}}$ to $T_{\tilde{\ell}(\tbm{p})}\mathcal{S}_{\ell}$. Therefore, given $\bm{w}\in T_{\tilde{\ell}(\tbm{p})}\mathcal{S}_{\ell}$, there exists $\bm{v}\in \mathbb{R}^{\tilde{n}}$ such that $\bm{w} = \mathsf{D}  \tilde{\ell}(\tbm{p}) \bm{v}$. Furthermore, Lemma \ref{27:} implies that $\mathsf{N}^{\tilde{\ell}}(\tbm{p}) = \bm{p}$, since $\inner{\bm{p}}{\mathsf{D} \tilde{\ell}(\tbm{p})}=\bm{0}_{\tilde{n}}^{\mathsf{T}}$. Substituting $\mathsf{N}^{\tilde{\ell}}$ into \eqref{68:e} yields 
			
			\begin{align}
			\kappa_{\ell}(\tilde{\ell}(\tbm{p}), \bm{w}) &= \frac{\bm{v}^{\mathsf{T}} (\mathsf{D}  \tilde{\ell}(\tbm{p}))^{\mathsf{T}} \begin{bmatrix}I_{\tilde{n}}, \\ \bm{1}_{\tilde{n}}\end{bmatrix} \bm{v} }{\Inner{\mathsf{D}  \tilde{\ell}(\tbm{p}) \bm{v}}{\mathsf{D}  \tilde{\ell}(\tbm{p}) \bm{v}}}, \nonumber \\
			& = \frac{\bm{v}^{\mathsf{T}} \mathsf{H} \tbr_{\ell}(\tbm{p}) \begin{bmatrix}X^{\mathsf{T}}_{\bm{p}}, & - \tbm{p} \end{bmatrix}\begin{bmatrix}I_{\tilde{n}} \\ \bm{1}_{\tilde{n}}\end{bmatrix} \bm{v}}{\Inner{\mathsf{D} \tilde{\ell}(\tbm{p}) \bm{v}}{\mathsf{D} \tilde{\ell}(\tbm{p}) \bm{v}}}, \nonumber \\
			& = \frac{\bm{v}^{\mathsf{T}} \mathsf{H} \tbr_{\ell}(\tbm{p})  \bm{v}}{\bm{v}^{\mathsf{T}}  \mathsf{H} \tbr_{\ell}(\tbm{p})\begin{bmatrix}X^{\mathsf{T}}_{\bm{p}}, & - \tbm{p} \end{bmatrix} \begin{bmatrix}X_{\bm{p}}\\ - \tbm{p}^{\mathsf{T}} \end{bmatrix}  \mathsf{H} \tbr_{\ell}(\tbm{p})\bm{v}}. \label{curv:e}
			\end{align}
Setting $\bm{u} = (- \mathsf{H} \tbr_{\ell}(\tbm{p}))^{\frac{1}{2}}\bm{v}/\Norm{(- \mathsf{H} \tbr_{\ell}(\tbm{p}))^{\frac{1}{2}}\bm{v}}$ in \eqref{curv:e} gives the desired result.
		\end{proof}

		\section{Classical Mixability Revisited}
		\label{classicmix}
			In this appendix, we provide a more concise proof of the necessary and sufficient conditions for the convexity of the superprediction set \citep{DBLP:journals/jmlr/ErvenRW12}.
		\begin{theorem}
			\label{40:}
			Let $\ell \colon \Delta_n \rightarrow [0,+\infty[^n$ be a strictly proper loss whose Bayes risk is twice differentiable on $]0, +\infty[^n$. The following points are equivalent;
			\begin{enumerate}[label=(\roman*)]
				\item $\forall \tbm{p}\in \Int \tilde{\Delta}_n,\; \eta \mathsf{H}\tbr_{\ell}(\tbm{p}) \succeq \mathsf{H} \tbr_{\log}(\tbm{p})$.
				\item $e^{-\eta \sps} = \bigcap_{\bm{p} \in \Delta_n} \mathcal{H}_{\tau(\bm{p}),1} \cap [0,+\infty[^n$, where $\tau(\bm{p}) \coloneqq \bm{p} \odot e^{\eta \ell(\bm{p})}$.
				\item $e^{-\eta \sps}$ is convex.
			\end{enumerate}
		\end{theorem}

		\begin{proof}
			We already showed (i) $\implies$ (ii) $\implies$ (iii) in the proof of Theorem \ref{8:}. 

We now show (iii) $\implies$ (i).
Since $e^{-\eta \sps }$ is convex, any point $\bm{s} \in \Bd e^{-\eta \sps}$ is supported by a hyperplane \citep[Lem. A.4.2.1]{Hiriart-Urruty}. Since $\bm{u} \rightarrow e^{-\eta \bm{u}}$ is a homeomorphism, it maps boundaries to boundaries. From this and Lemma \ref{33:}, $\Bd e^{-\eta \sps} = e^{-\eta \mathcal{S}_{\ell}}$. Thus, for $\bm{p}\in \Ri \Delta_n$, there exists a unit-norm vector $\bm{u} \in \mathbb{R}^n$ such that for all $\bm{s} \in \sps$ it either holds that $\inner{\bm{u}}{e^{-\eta \ell(\bm{p})}} \leq \inner{\bm{u}}{e^{-\eta \bm{s}}}$; or $\inner{\bm{u}}{e^{-\eta \ell(\bm{p})}} \geq \inner{\bm{u}}{e^{-\eta \bm{s}}}$. It is easy to see that it is the latter case that holds, since we can choose $\bm{s} = \ell(\bm{r}) + c \bm{1}  \in \sps$, for $\bm{r} \in \Delta_n$, and make $\inner{\bm{u}}{e^{-\eta \bm{s}}}$ arbitrarily small by making $c \in \mathbb{R}$ large. Therefore, $\forall \bm{r} \in \Ri \Delta_n, \inner{\bm{u}}{e^{-\eta \tilde{\ell}(\tilde{\bm{p}})}} = \inner{\bm{u}}{e^{-\eta \ell(\bm{p})}} \geq \inner{\bm{u}}{e^{-\eta \ell(\bm{r})}} =\inner{\bm{u}}{e^{-\eta \tilde{\ell}(\tilde{\bm{r}})}}$ and $\tbm{p} $ is a critical point of the function $f(\tbm{r}) \coloneqq  \inner{\bm{u}}{e^{-\eta \tilde{\ell}(\tbm{r})}}$ on $\Int \tilde{\Delta}_n$. This implies that $\nabla f(\tbm{p}) = \bm{0}_{\tilde{n}}$; that is, $- \eta \inner{\bm{u}}{\Diag (e^{-\eta \tilde{\ell}(\tbm{p})}) \mathsf{D} \tilde{\ell} (\tbm{p}) } = - \eta \inner{\Diag (e^{-\eta \tilde{\ell}(\tbm{p})})\bm{u}}{  \mathsf{D} \tilde{\ell} (\tbm{p}) }=\bm{0}^{\mathsf{T}}_{\tilde{n}}$. From Lemma \ref{27:}, there exists $\lambda \in \mathbb{R}$ such that $\Diag (e^{-\eta \tilde{\ell}(\tbm{p})}) \bm{u} = \lambda \bm{p}$. Therefore, $\bm{u} = \lambda  \bm{p} \odot e^{\eta \tilde{\ell}(\tbm{p})}$, where $\lambda = \Norm{\bm{p} \odot e^{\eta \tilde{\ell}(\tbm{p})}}^{-1}$, since $\norm{\bm{u}}=1$. For $\bm{v} \in \mathbb{R}^{n-1}$, let $\tbm{\alpha}^t \coloneqq \tbm{p} + t \bm{v}$, where $t \in \{s: \tbm{p} + s \bm{v} \in \Int \tilde{\Delta}_n \}$. Since $f$ is twice differentiable and attains a maximum at $\tbm{p}$,  
			\begin{align}
			0 \geq \frac{1}{\lambda \eta } \left. \frac{d^2}{dt^2}  f \circ \tbm{\alpha}^t\right|_{t=0} & = \frac{1}{\lambda} \left. \frac{d}{dt}  \Inner{\bm{u}}{\Diag e^{- \eta \tilde{\ell}(\tbm{\alpha}^t)} \mathsf{D} \tilde{\ell}( \tbm{\alpha}^t) \bm{v} }\right|_{t=0} , \nonumber  \\
			& =  \left. \frac{d}{dt}  \Inner{\bm{p} \odot e^{\eta \tilde{\ell}(\tbm{p})}}{\Diag e^{- \eta \tilde{\ell}(\tbm{\alpha}^t)} \mathsf{D} \tilde{\ell}( \tbm{p})\bm{v} } \right|_{t=0} + \left. \frac{d}{dt}   \Inner{\bm{p}}{\mathsf{D} \tilde{\ell}( \tbm{\alpha}^t) \bm{v}}\right|_{t=0}, \nonumber \\
			& = \eta \bm{v}^{\mathsf{T}} \mathsf{H} \tbr_{\ell}(\tbm{p})  (\mathsf{H} \tbr_{\log}(\tbm{p}))^{-1} \mathsf{H} \tbr_{\ell}(\tbm{p}) \bm{v} - \bm{v}^{\mathsf{T}} \mathsf{H} \tbr_{\ell}(\tbm{p}) \bm{v}, \label{70:e}
			\end{align}
			where in the second equality we substituted $\bm{u}$ by $\lambda \bm{p} \odot e^{\eta \tilde{\ell}(\tbm{p})}$ and in \eqref{70:e} we used \eqref{20:e} and \eqref{21:e} from Lemma \ref{28:}. Note that by the assumptions on $\ell$ it follows that the Bayes risk $\tbr_{\ell}$ is strictly concave  \citep[Lemma~6]{DBLP:journals/jmlr/ErvenRW12} and $- \mathsf{H} \tbr_{\ell}(\tbm{p})$ is symmetric negative-definite. In particular, $\mathsf{H} \tbr_{\ell}(\tbm{p})$ is invertible. Setting $\hat{\bm{v}} \coloneqq  \mathsf{H} \tbr_{\ell}(\tbm{p})  \bm{v}$ in \eqref{70:e} yields   
			\begin{align*}
			0 \geq  \eta \hat{\bm{v}} (\mathsf{H} \tbr_{\log}(\tbm{p}))^{-1} \hat{\bm{v}} - \hat{\bm{v}} (\mathsf{H} \tbr_{\ell}(\tbm{p}))^{-1} \hat{\bm{v}}. 
			\end{align*}
			Since $\bm{v} \in \mathbb{R}^{n-1}$ was chosen arbitrarily, $ (\mathsf{H} \tbr_{\ell}(\tbm{p}))^{-1} \succeq \eta (\mathsf{H} \tbr_{\log}(\tbm{p}))^{-1}, \forall \tbm{p} \in \Int \tilde{\Delta}_n$. This is equivalent to the condition $\forall \tbm{p}\in \Int \tilde{\Delta}_n,\; \eta \mathsf{H}\tbr_{\ell}(\tbm{p}) \succeq \mathsf{H} \tbr_{\log}(\tbm{p})$.
		\end{proof}

			\section{An Experiment on Football Prediction Dataset}
			\label{Experiment}

\begin{figure}[h]
\centering
\begin{minipage}[c]{.49\textwidth}
\hspace{-0.1cm}
    \includegraphics[width=1.05\textwidth]{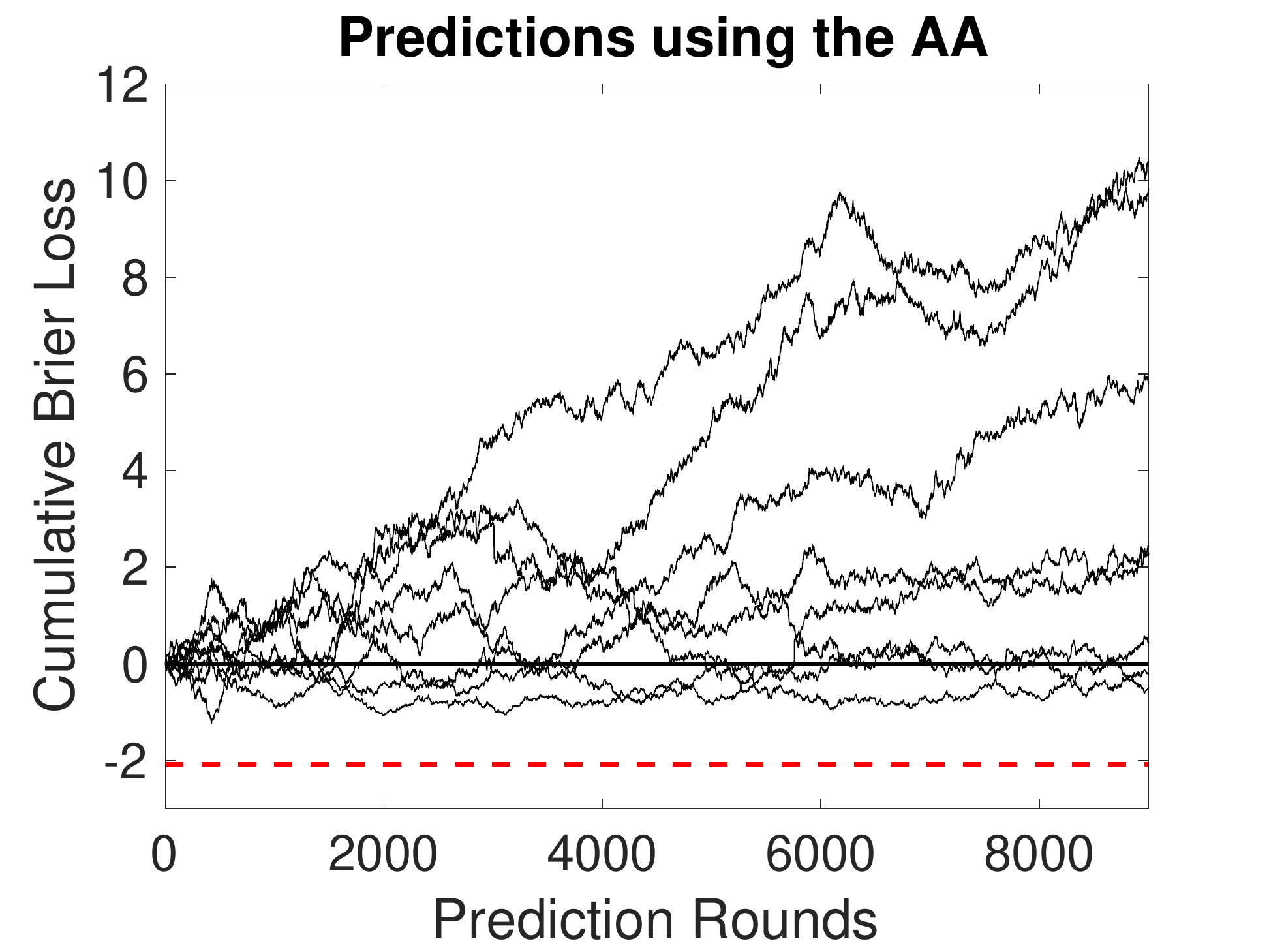}
\end{minipage}
\begin{minipage}[c]{.49\textwidth}
\hspace{-0.1cm}
    \includegraphics[width=1.05\textwidth]{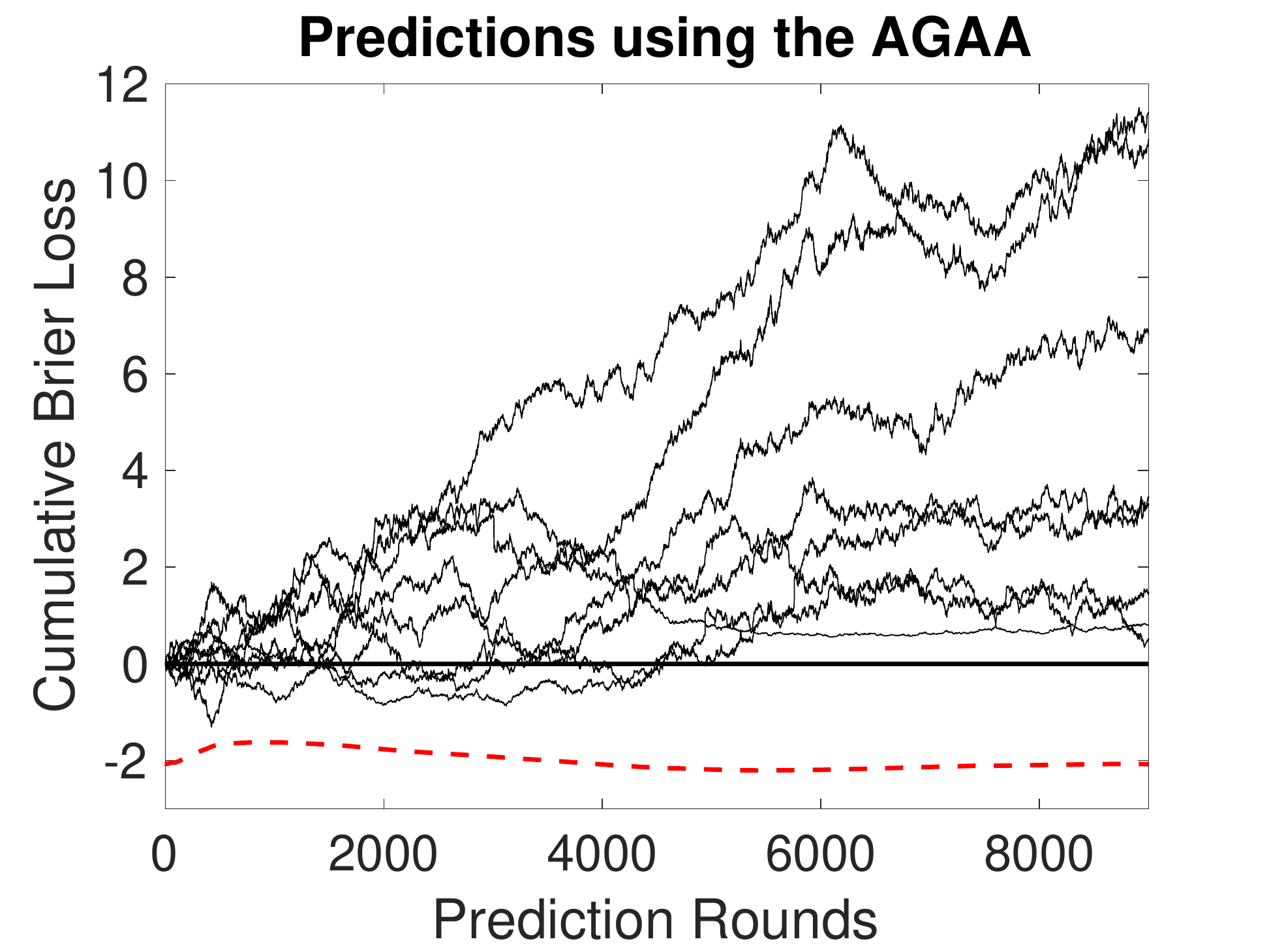}
\end{minipage}
    \caption{The figure corresponds to the 2005/2006, 2006/2007, 2007/2008, and 2008/2009 seasons. The solid lines represent, at each round $t$, the difference between the cumulative losses of the experts and that of the learner who uses either the AA (left) or the AGAA (right); that is, $\op{Loss}^{\ell_{\op{Brier}}}_{\theta}(t) - \op{Loss}^{\ell_{\op{Brier}}}_{\mathfrak{M}}(t)$, for $\mathfrak{M}\in \{\op{AA}, \op{AGAA}\}$. The red dashed lines represent the negative of the regret bound in \eqref{98:e} with respect to the best expert $\theta^*$; that is, $- R^{\Se}_{\ell_{\op{Brier}}} - \Delta R_{\theta^*}(t)$ at each round $t$.}
    \label{fig:0509}
\end{figure}

\begin{figure}[h]
\centering
\begin{minipage}[c]{.49\textwidth}
\hspace{-0.1cm}
    \includegraphics[width=1.05\textwidth]{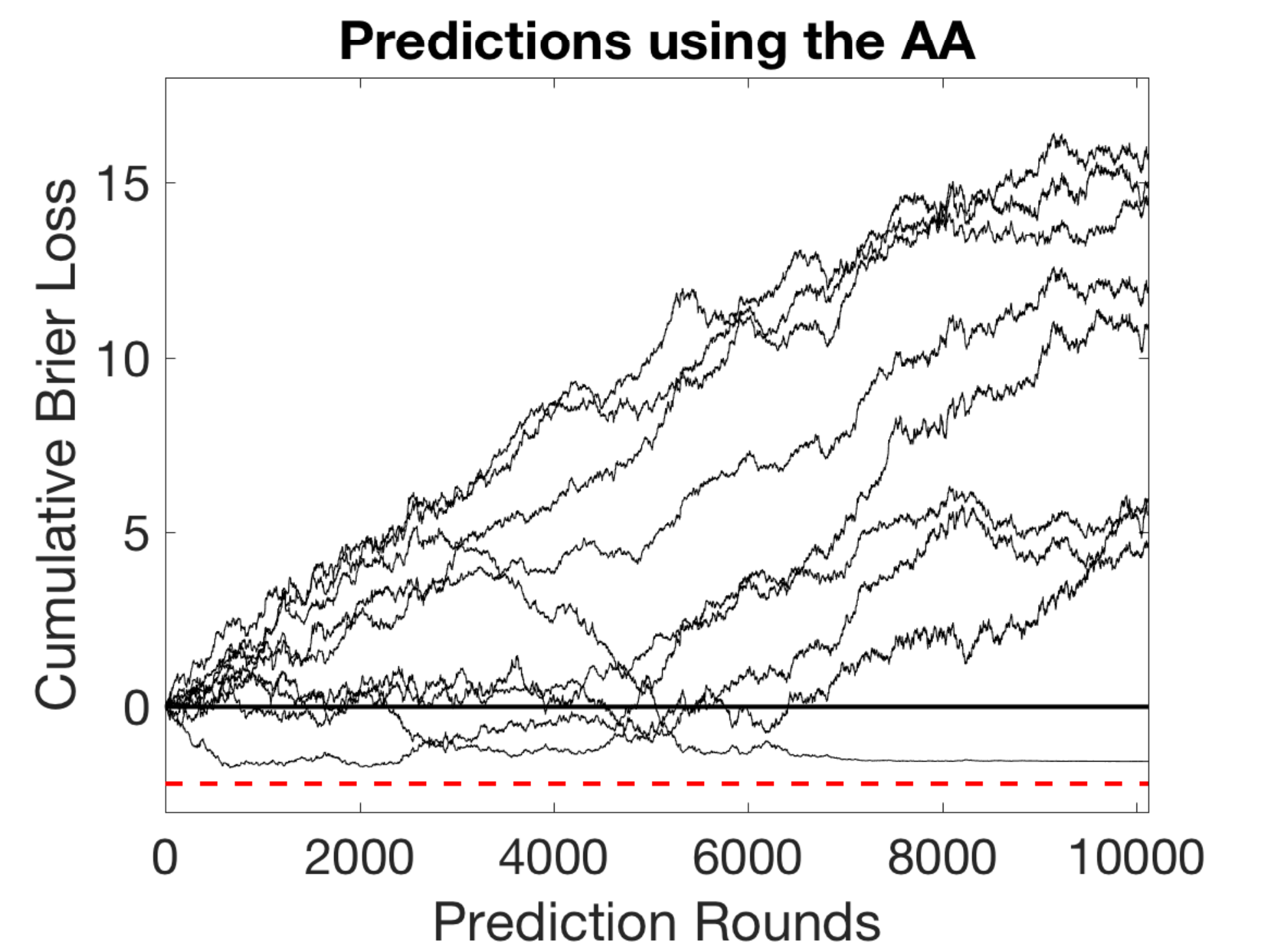}
\end{minipage}
\begin{minipage}[c]{.49\textwidth}
\hspace{-0.1cm}
    \includegraphics[width=1.05\textwidth]{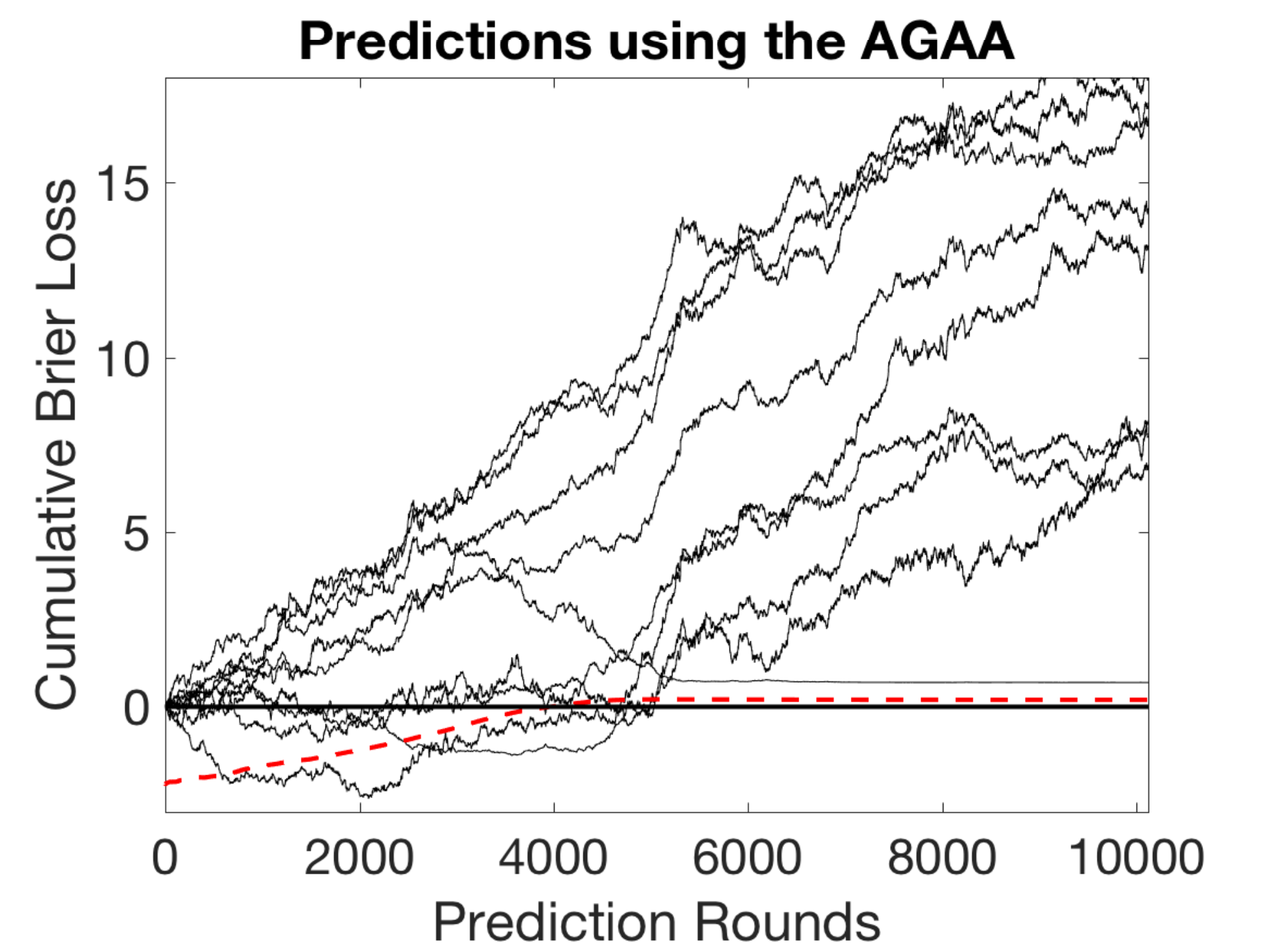}
\end{minipage}
    \caption{The figure corresponds to the 2009/2010, 2010/2011, 2011/2012, and 2012/2013 seasons The solid lines represent, at each round $t$, the difference between the cumulative losses of the experts and that of the learner who uses either the AA (left) or the AGAA (right); that is, $\op{Loss}^{\ell_{\op{Brier}}}_{\theta}(t) - \op{Loss}^{\ell_{\op{Brier}}}_{\mathfrak{M}}(t)$, for $\mathfrak{M}\in \{\op{AA}, \op{AGAA}\}$. The red dashed lines represent the negative of the regret bound in \eqref{98:e} with respect to the best expert $\theta^*$; that is, $- R^{\Se}_{\ell_{\op{Brier}}} - \Delta R_{\theta^*}(t)$ at each round $t$.}
    \label{fig:0913}
\end{figure}

\begin{figure}[h]
\centering
\begin{minipage}[c]{.49\textwidth}
\hspace{-0.1cm}
    \includegraphics[width=1.05\textwidth]{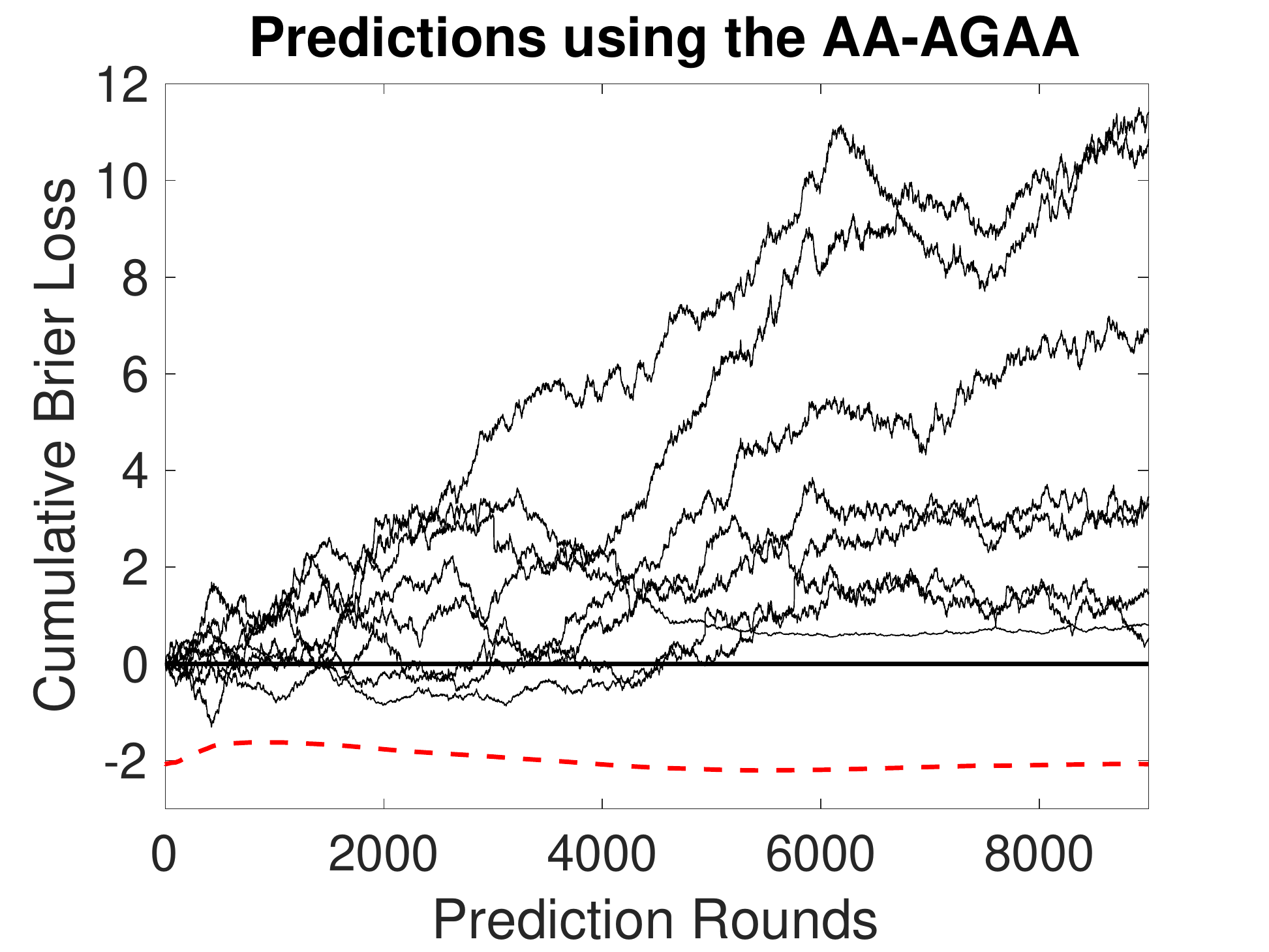}
\end{minipage}
\begin{minipage}[c]{.49\textwidth}
\hspace{-0.1cm}
    \includegraphics[width=1.05\textwidth]{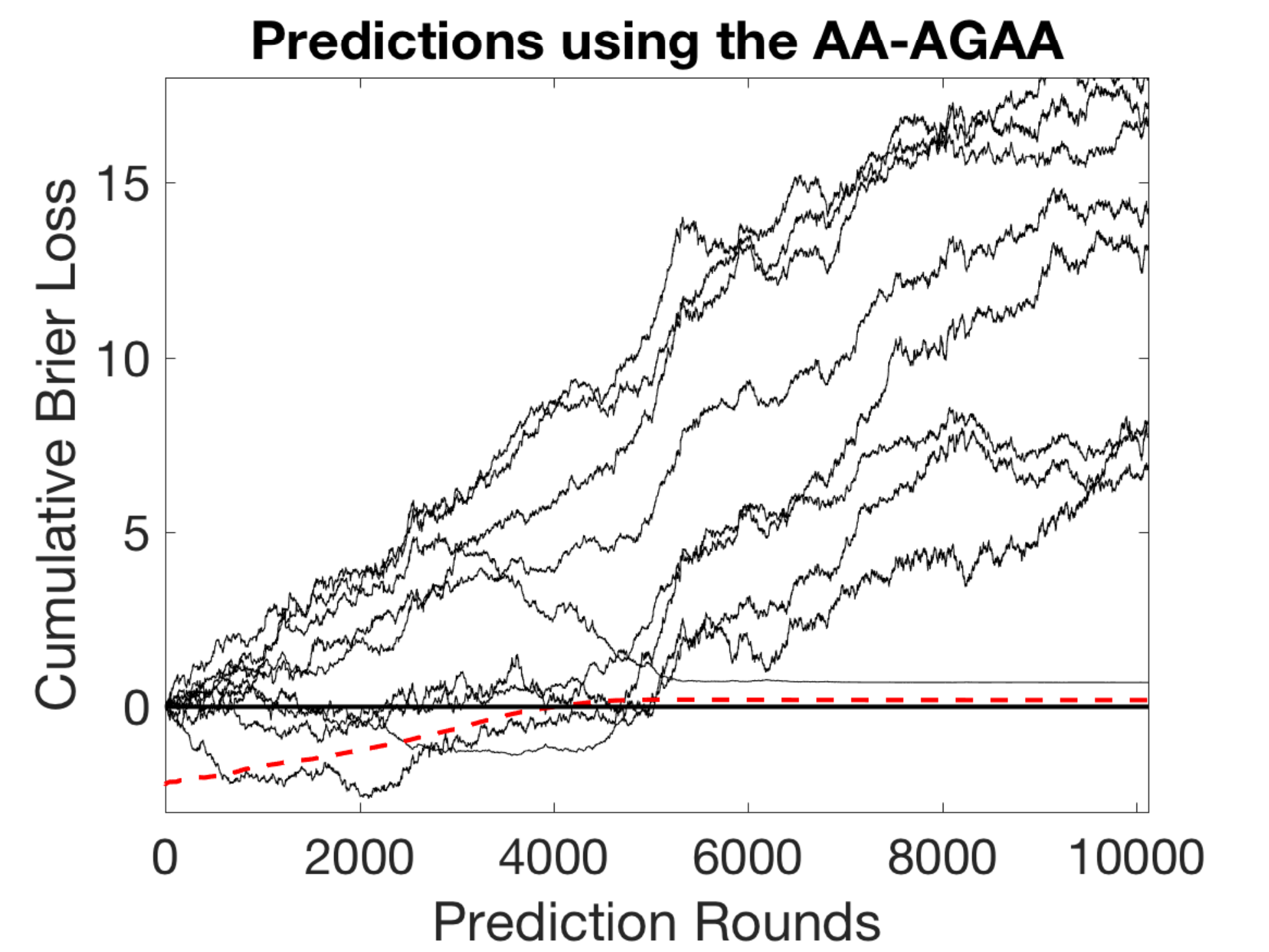}
\end{minipage}
    \caption{The figure on the left [resp. right] hand side corresponds to the football seasons from 2005 to 2009  [resp. 2009 to 2013]. The solid lines represent, at each round $t$, the difference between the cumulative losses of the experts and that of the learner using the AA-AGAA meta algorithm (refer to text); that is, $\op{Loss}^{\ell_{\op{Brier}}}_{\theta}(t) - \op{Loss}^{\ell_{\op{Brier}}}_{\op{AA-AGAA}}(t)$. The red dashed lines represent the negative of the regret bound in \eqref{98:e} with respect to the best expert $\theta^*$; that is, $- R^{\Se}_{\ell_{\op{Brier}}} - \Delta R_{\theta^*}(t)$ at each round $t$.}
    \label{fig:050913}
\end{figure}

\label{experiment}
\subsection{Testing the AGAA}
To test the AGAA empirically, we used prediction data\footnote{The data was collected from http://www.football-data.co.uk/.} from the British football leagues, including the Premier Leagues, Championships, Leagues 1-2, and Conferences. The first dataset contains predictions for the 2005/2006, 2006/2007, 2007/2008, and 2008/2009 seasons, matching the dataset used in \cite{Vovk2009}. The second dataset contains predictions for the 2009/2010, 2010/2011, 2011/2012, and 2012/2013 seasons. For this set, we considered predictions from 9 bookmakers; Bet365, Bet\&Win, Blue Square, Gamebookers, Interwetten, Ladbrokes, Stan James, VC Bet, and William Hill.

On each dataset, we compared the performance of the AGAA with that of the AA using the Brier score (the Brier loss is 1-mixable). For the AGAA, we chose $\bm{\beta}^t$ according to Theorem \ref{rboundAGAA} with $\bm{v}^t \coloneqq  -  \frac{1}{2t} \sum_{s=1}^{t} \ell_{x^s}(A^s)$ and we set $\Phi=\Se$, i.e. the Shannon entropy. The results in Figure \ref{fig:0509} [resp. Figure \ref{fig:0913}] correspond to the seasons from 2005 to 2009 [resp. 2009 to 2013].  For fair comparison with the results of Vovk \cite{Vovk2009}, we 1) used the same substitution function as \cite{Vovk2009}; 2) used the same method for converting odds to probabilities; and 3) sorted the data first by date then by league and then by name of the host team (For more detail see \cite{Vovk2009}). 

In all figures the solid lines represent, at each round $t$, the difference /between the cumulative losses of the experts and that of the learners; that is, $\op{Loss}^{\ell_{\op{Brier}}}_{\theta}(t) - \op{Loss}^{\ell_{\op{Brier}}}_{\mathfrak{M}}(t)$, for $\mathfrak{M}= \op{AA}, \op{AGAA}$. The red dashed lines represent the negative of the regret bound in \eqref{98:e} with respect to the best expert $\theta^*$; that is, $- R^{\Se}_{\ell_{\op{Brier}}} - \Delta R_{\theta^*}(t)= - R^{\Se}_{\ell_{\op{Brier}}}-  \sum_{s=1}^{t-1} ( v^s_{\theta} -\inner{\bm{v}^s}{\bm{q}^{s}})$ at each round $t$, where $(\bm{q}^s)$ are the distributions over experts.

From Figures \ref{fig:0509} and \ref{fig:0913} it can be seen that the learners using the AGAA perform better than the best expert (and better than the AA) at the end of the games.

		\subsection{Testing a AA-AGAA Meta-Learner}
Consider the algorithm (referred to as AA-AGAA) that takes the outputs of the AGAA and the AA as in the previous section and aggregates them using the AA to yield a \emph{meta prediction}. The worst case regret of this algorithm is guaranteed not to exceed that of the original AA and AGAA by more than $\eta^{-1} \log 2$ for an $\eta$-mixable loss. Figure \ref{fig:050913} shows the results for this algorithm for the same datasets as the previous section. The AA-AGAA still achieves a negative regret at the end of the game.

\end{appendix}

\end{document}